\documentclass{article}
\usepackage[T1]{fontenc}
\usepackage{iclr2027_conference,times}
\usepackage{microtype,graphicx,subcaption,booktabs}
\usepackage{amsmath,amssymb,mathtools,amsthm}
\usepackage{enumitem,algorithm,algorithmic}
\usepackage{hyperref,url}
\usepackage{xurl}
\usepackage{placeins}
\usepackage{flafter,etoolbox}
\usepackage[capitalize,noabbrev]{cleveref}
\iclrfinalcopy
\apptocmd{\maketitle}{\lhead{Preprint}}{}{\PackageError{subopt}{Cannot set the preprint header}{Check the ICLR style file.}}
\newenvironment{longalgorithm}{  \par\medskip\hrule\smallskip
  \captionsetup{type=algorithm}
}{  \smallskip\hrule\medskip
}

\DeclareMathOperator*{\argmax}{arg\,max}
\providecommand{\Beta}{\mathbf{B}}

\theoremstyle{plain}
\newtheorem{theorem}{Theorem}[section]

\newtheorem{lemma}[theorem]{Lemma}
\newtheorem{corollary}[theorem]{Corollary}
\theoremstyle{definition}
\newtheorem{definition}[theorem]{Definition}
\newtheorem{assumption}[theorem]{Assumption}
\newtheorem{claim}{Claim}[section]
\theoremstyle{remark}

\BeforeBeginEnvironment{theorem}{\begin{samepage}}
\AfterEndEnvironment{theorem}{\end{samepage}}
\allowdisplaybreaks[2]
\hypersetup{colorlinks=true,linkcolor=blue,citecolor=blue,urlcolor=blue,
 pdftitle={On Semi-Markov Suboptimality in Hierarchical Reinforcement Learning},
 pdfauthor={Bingyun Liu; Yuheng Jing}}

\title{On Semi-Markov Suboptimality in\\Hierarchical Reinforcement Learning}
\author{Bingyun Liu\thanks{Equal contribution.}\\
Institute of Automation\\
Chinese Academy of Sciences\\
\texttt{liubingyun2021@ia.ac.cn}
\And
Yuheng Jing\footnotemark[1]\kern0.4em\thanks{Corresponding author.}\\
Institute of Automation\\
Chinese Academy of Sciences\\
\texttt{jingyuheng2022@ia.ac.cn}}
\date{}
\begin{document}
\maketitle
\begin{abstract}
\IfFileExists{sections/abstract.tex}{Hierarchical reinforcement learning uses temporally extended subtasks for exploration, yet committing to their execution can restrict both deployment and policy learning. We identify and separate the resulting execution and policy suboptimality. Task and execution trees distinguish reward objectives from policy choices and decision interruption. A Unified Value Function for HRL and a four-stage Generalized Hierarchical Bellman Equation then support a common analysis of both losses. Under bounded rewards and uniform termination, we establish hierarchical policy and execution improvement results. With the remaining node policies fixed, task-subtree compatibility and node-policy optimality under the original execution mode establish when Markov execution is optimal. The resulting decomposition leads to independent execution choices for behavior, targets, and deployment. We instantiate this principle through execution improvement and one-stage or two-stage policy improvement at arbitrary hierarchy depth. Option-based and goal-conditioned experiments demonstrate complementary gains from changing execution and changing the learning target. Controlled stochastic environments show how these gains depend on stochastic transition strength and spatial structure. This framework makes execution design an explicit component of hierarchical policy optimization.
}{Hierarchical reinforcement learning separates task objectives, policy optimization, and execution.}
\end{abstract}
\section{Introduction}
\label{new:sec:introduction}

Temporal abstraction gives hierarchical reinforcement learning a structured way to explore: committing to a useful subtask produces transitions that would be difficult to discover through independent primitive actions \citep{AI99MDPsSMDPs,NIPS20EfficiencyHRL}. Exploiting the learned hierarchy poses a different question: when should the agent reconsider its choice? After a stochastic transition, another subtask can become more valuable even while the current subtask remains unfinished.

This distinction reveals two sources of loss within a fixed task hierarchy. \emph{Execution suboptimality} concerns when a learned policy reselects subtasks; \emph{policy suboptimality} concerns which execution process its training objective optimizes. Improving deployment execution and optimizing a policy for that execution are separate interventions. Here \emph{Semi-Markov Execution} (SME) runs each selected subtask until its task termination event, while \emph{Markov Execution} (ME) permits a new subtask decision at every environment frame. Their difference motivates a theory that treats task objectives and execution as independent design choices.

The main contributions of this paper are the following three:

\textbf{We first formalize the two sources of suboptimality within a unified value model.} Task and execution trees separate nested reward objectives and termination boundaries from policy choices and decision interruption. The Unified Value Function for HRL (UVFH) distinguishes the evaluated task from the node holding decision authority (Figure~\ref{new:fig:overview}). Its four-stage Generalized Hierarchical Bellman Equation (GHBE) and optimality counterpart, GHBOE, support node-policy and execution improvement with the remaining policies fixed. An optimal execution theorem under task-subtree compatibility and node-policy optimality then yields a decomposition into execution and policy suboptimality.

\begin{figure}[t]
\centering
\includegraphics[width=\textwidth]{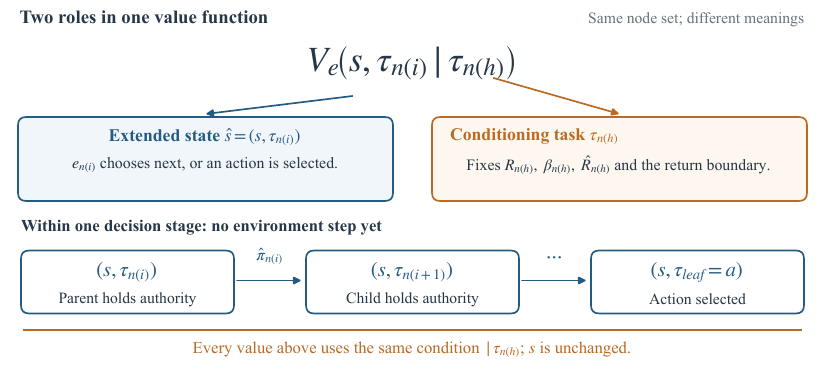}
\caption{\textbf{Decision state is not the task whose return is evaluated.} The extended-state coordinate records the current decision position; the condition specifies the reward objective and exit boundary. Decision authority moves down the tree within a frame while the environment state and conditioning task stay fixed. The two node labels can coincide, but their roles are distinct.}
\label{new:fig:overview}
\end{figure}

\textbf{Building on this decomposition, we decouple behavior, target, and deployment execution.} ESIF addresses execution suboptimality by changing how a learned policy is deployed; PSIF addresses policy suboptimality by changing the execution process optimized during learning. Both PSIF procedures apply at arbitrary depth: PSIF-2S changes behavior and target execution together through bottom-up, node-wise stages, while PSIF-1S preserves SME behavior and forms local ME targets. This design retains temporal structure for exploration while optimizing policies and execution for exploitation.

\textbf{We then test the resulting frameworks from discrete control to continuous-control extensions.} Tabular OC provides finite, discounted Markov control with the two-level compatible hierarchy of the theory, allowing us to evaluate ESIF and PSIF in a directly aligned setting. We extend this evaluation to continuous control: goal-conditioned HAC tests ESIF execution improvement, and option-based DAC tests PSIF-style policy improvement through a second-stage comparison. Finally, controlled local-slip and global-teleportation experiments vary both random-transition probability and type to analyze their distinct effects on execution and policy improvement.

\paragraph{Setting.}
We use the task-based MDP formalization $(\mathcal S,\mathcal A,T,R_\tau,\alpha_\tau,\beta_\tau)$: $T$ is the Markov transition kernel, $R_\tau$ the reward, $\alpha_\tau$ the initial-state distribution, and $\beta_\tau$ the termination function. A frame is one primitive environment interaction. A node policy $\hat\pi_\tau$ selects child tasks; successive child choices induce a primitive-action policy $\pi_\tau$. All comparisons keep the available task hierarchy explicit. Appendices~\ref{sec:2.1}, \ref{sec:3.1}, and~\ref{sec:A} provide extended preliminaries, execution examples, and the exact GridWorld construction.

\section{A unified theory of task objectives and execution}
\label{new:sec:theory}

The source of the difficulty is that a task has two different roles: it defines an optimization objective, and it induces a transition when called. Its \emph{termination function} $\beta$ determines the first role, whereas its \emph{decision interruption function} $\hat\beta$ determines the second. Binding these functions makes the optimization objective and the execution schedule change together. We separate them using task-tree and execution-tree formalization; the complete definitions and proofs appear in Appendix~\ref{sec:B}.

\paragraph{Task, policy, and execution trees.}
Let $N(i)$ denote the node indices at depth $i$, let $n(i)$ denote a node on a given path, and write $x|n(i)$ for the subtree rooted at node $x_{n(i)}$. A \emph{task tree} $\tau$ describes nested task objectives. Each task node is the quadruple
\[
\tau_{n(i)}=\left(\mathcal{T}^{\prime},\beta,R,\hat{R}\right)_{n(i)}.
\]
Here $\mathcal T'$ is its child-task set or space, $\beta$ is its termination function, $R$ is its per-frame reward, and $\hat R$ is its exit reward. If a task or one of its ancestors terminates, the highest ancestor with a termination event determines the exit boundary; the task and its descendants terminate together. Thus $\beta$ and $\hat R$ specify both how long rewards accumulate and which boundary reward applies. For options, an exit reward can be the parent's value; for goal-conditioned tasks, it can be zero. These two common frameworks therefore fit the same task-node definition.

An \emph{execution tree} $e$ specifies how the agent calls these tasks. Each execution node is the pair $(\hat\pi,\hat\beta)_{n(i)}$: $\hat\pi$ selects a child task, and $\hat\beta$ determines whether decision authority returns to an ancestor. The node policies form the \emph{policy tree} $\Pi$; the interruption functions form the \emph{interruption tree} $\Beta$. Only a node holding decision authority may choose a child. Decisions proceed down the tree until an action is selected. Changing $\Beta$ changes this schedule while leaving the task tree and policy tree fixed.

Under SME, interruption events are coupled to task termination events, with $\hat\beta=\beta$. Under ME, $\hat\beta\equiv1$: decisions restart at the root of the executing subtree every frame. This does not redefine the rewards or termination boundary of the task being evaluated. In particular, task termination and decision interruption must remain distinct even when the two mechanisms happen to use the same probabilities.

\paragraph{Unified Value Function for HRL (UVFH).}
The value must identify both \emph{whose reward is being evaluated} and \emph{who currently holds decision authority}. Define the extended state $\hat{s}\dot{=}\left(s,\tau_{n(i)}\right)$, where $e_{n(i)}$ holds decision authority, or an action has just been selected. For the evaluated task $\tau\dot{=}\tau_{n(h)}$, UVFH is defined by
\begin{equation}
\label{new:eq:uvfh}
\begin{aligned}
V_\Pi^\Beta(\hat{s}|\tau)=V_e(\hat{s}|\tau)
={}&\mathbb{E}_{s_0=s,T\sim\tau,e}\Big[\sum_{t=0}^{T-1}R_{n(h)}(s_t,a_t)\,+\,
\hat{R}_{n(h)}\Big(s_T,\tau_{n(j_{\min}-1)}\Big)\Big].
\end{aligned}
\end{equation}
Here $T$ is the frame at which the evaluated task terminates. The exit destination records the hierarchical boundary of that event. This definition unifies state values, action values, and values at intermediate decision nodes without identifying the evaluated task with the current decision node. We write $V_\Pi^{SME}$ and $V_\Pi^{ME}$ for the two special execution modes.

The two appearances of a task label in $V_e(s,\tau_{n(i)}\mid\tau_{n(h)})$ have different meanings. The \emph{decision state} $\tau_{n(i)}$ records where the next choice is made in the execution tree; it does not select the reward being accumulated. The conditioning task $\tau_{n(h)}$ supplies that reward and its exit boundary. As Figure~\ref{new:fig:overview} shows, several decision states can be visited at the same $s$ before one action is executed. Their values all evaluate the same task. Conversely, fixing the extended state and changing the conditioning task asks about a different reward objective. The indices may coincide; they are independent roles, not two different kinds of task.

\paragraph{Generalized Hierarchical Bellman Equation (GHBE).}
For an execution subtree $e|n(h')$ evaluating task $\tau_{n(h)}$, $h\geq h'$, one environment frame contains four stages. Decisions descend through $\Pi$; an action produces a transition and reward; task termination determines the exit reward; and decision interruption reallocates authority for the next frame. Let $\gamma_{n(k)}$ denote the task-transition probabilities and $\hat\gamma_{n(k)}$ the decision-transition probabilities, as defined in Appendix~\ref{sec:B.1}. If depth $l$ directly selected the preceding action, the four equations are
\begingroup
\small
\begin{align}
\label{new:eq:ghbe_decision}
V_{e|n\left(h^\prime\right)}\left(s, \tau_{n\left(i\right)}\middle|\tau_{n\left(h\right)}\right)
&=\mathbb{E}_{\tau_{n\left(i+1\right)}\sim{\hat{\pi}}_{n\left(i\right)}}\left[V_{e|n\left(h^\prime\right)}\left(s, \tau_{n\left(i+1\right)}\middle|\tau_{n(h)}\right)\right],\quad\forall n\left(i\right),\\
\label{new:eq:ghbe_interaction}
V_{e|n\left(h^\prime\right)}\left(s, \tau_{leaf}\middle|\tau_{n\left(h\right)}\right)
&=R_{n\left(h\right)}\left(s,a\right)+\mathbb{E}_{s^\prime\sim \mathrm{T}\left(s,a\right)}\left[{\bar{V}}_{e|n\left(h^\prime\right)} \left(s^\prime, \tau_{leaf}\middle|\tau_{n\left(h\right)}\right)\right],\\
\label{new:eq:ghbe_termination}
{\bar{V}}_{e|n\left(h^\prime\right)} \left(s^\prime, \tau_{leaf}\middle|\tau_{n\left(h\right)}\right)
&={\hat{V}}_{e|n\left(h^\prime\right)} \left(s^\prime, \tau_{leaf}\middle|\tau_{n\left(h\right)}\right)\gamma_{n\left(h\right)}\left(s^\prime\right)\nonumber\\
&\quad+\sum_{k=-1}^{h-1}{\gamma_{n\left(k\right)}\left(s^\prime\right){\hat{R}}_{n\left(h\right)}\left(s^\prime,\tau_{n\left(k\right)}\right)},\\
\label{new:eq:ghbe_preparation}
{\hat{V}}_{e|n\left(h^\prime\right)} \left(s^\prime, \tau_{leaf}\middle|\tau_{n\left(h\right)}\right)
&=\sum_{k=h^\prime}^{l}{{\hat{\gamma}}_{n\left(k\right)}\left(s^\prime\right)}V_{e|n\left(h^\prime\right)}\left(s^\prime, \tau_{n\left(k\right)}\middle|\tau_{n(h)}\right).
\end{align}
\endgroup
The intermediate values $\bar V$ and $\hat V$ separate task termination from execution preparation. To optimize one node policy while holding the other policies and interruption functions fixed, replace the decision-stage expectation \emph{only at that node} by a maximum. The other decision nodes retain their policy expectations. This gives the Generalized Hierarchical Bellman Optimality Equation (GHBOE), with $V^*,\bar V^*,\hat V^*$ denoting its values (Appendix~\ref{sec:B.2}). Figure~\ref{new:fig:ghboe} shows the four stages for the node-local objective $h'=h$, the case used in the suboptimality comparison below.

\begin{figure}[!htbp]
\centering
\includegraphics[width=\textwidth]{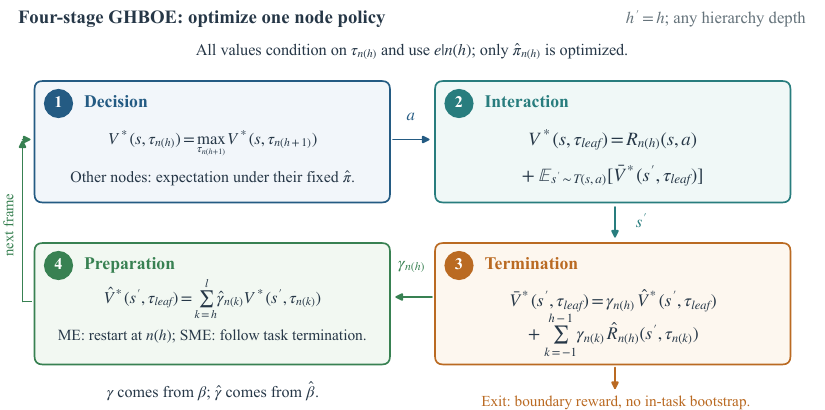}
\caption{\textbf{GHBOE separates optimization, task exit, and decision interruption.} Only $\hat\pi_{n(h)}$ is optimized; the other node policies are fixed. Termination adds the weighted survival and exit contributions, but only survival continues through preparation. The fixed condition $\tau_{n(h)}$ and execution-subtree subscript, shared by all four stages, are suppressed inside the boxes; transition probabilities are evaluated at $s'$.}
\label{new:fig:ghboe}
\end{figure}

The execution mode enters through preparation, not through a redefinition of the evaluated task's termination. Under ME, authority returns to the executing subtree's root; under SME, it follows task termination. The depth $l$ belongs to the execution node that selected the action, not to the action leaf. Only an environment interaction changes $s$ to $s'$. Thus the same task objective admits different Bellman recursions and different optimal node policies.

\paragraph{Evaluation and policy improvement.}
The evaluation lemma assumes uniform termination, $1-\gamma_{n(h)}(s)\geq\epsilon>0$, and bounded per-frame and exit rewards, with absolute values below $M$. The GHBE operator is then a contraction with factor $1-\epsilon$; it has a unique fixed point bounded by $2M/\epsilon$, and repeated evaluation converges to that fixed point (Lemma~\ref{def:B.17} and Corollary~\ref{def:B.18}). The value-comparison results below use these evaluation conditions. The policy-improvement statement is as follows.

\begin{theorem}[Hierarchical Policy Improvement]
\label{new:thm:policy_improvement}
Given execution trees $e$ and $e'$ that differ only in the node policies $\hat\pi_{n(h)}$ and $\hat\pi'_{n(h)}$ at $s$, if
\[
\mathbb{E}_{\tau_{n(h+1)}\sim\hat\pi_{n(h)}(s)}V_{e|n(h')}(s,\tau_{n(h+1)})
\leq
\mathbb{E}_{\tau_{n(h+1)}\sim\hat\pi'_{n(h)}(s)}V_{e|n(h')}(s,\tau_{n(h+1)}),
\]
then, for all $s$, $V_{e|n(h')}(s,\tau_{n(h)})\leq V_{e'|n(h')}(s,\tau_{n(h)})$, where $h\geq h'$.
\end{theorem}
The proof compares the two GHBE operators under the old value and propagates their ordering to their fixed points. The SME and ME specializations give policy-improvement results for their respective optimization objectives; the corresponding policy-iteration arguments follow Theorem~\ref{def:B.19} and Corollaries~\ref{def:B.20}--\ref{def:B.21} in the appendix.

\paragraph{Execution improvement.}
Keeping $\Pi$ fixed, the preparation-stage value also supplies an improvement criterion for the interruption tree.
\begin{theorem}[Hierarchical Execution Improvement]
\label{new:thm:execution_improvement}
Given execution trees $e$ and $e'$ differing only in $\hat\beta_{n(h)}$ and $\hat\beta'_{n(h)}$ at $s'$, compute $\hat\gamma,\hat\gamma'$ from their interruption functions. If, for every $n(l)$,
\[
\begin{aligned}
&\sum_{k=h'}^l\hat\gamma_{n(k)}(s')V_{e|n(h')}(s',\tau_{n(k)}|\tau_{n(h')})\\
&\hspace{1cm}\leq\sum_{k=h'}^l\hat\gamma'_{n(k)}(s')V_{e|n(h')}(s',\tau_{n(k)}|\tau_{n(h')}),
\end{aligned}
\]
then, for all $s$, $V_{e|n(h)}(s,\tau_{n(h)})\leq V_{e'|n(h)}(s,\tau_{n(h)})$. Here $l\geq h\geq h'$, and $l$ is the depth of the task that directly called the previous action.
\end{theorem}

\begin{assumption}[Task Subtree Compatibility]
\label{new:ass:compatibility}
The task subtree $\tau|n(h)$ has only two levels, or every descendant task $\tau_{n(h')}$, $h'>h$, has child-task set $\mathcal T'_{\tau_{n(h')}}\subseteq\mathcal T'_{\tau_{n(h)}}$.
\end{assumption}

\begin{theorem}[Optimal Execution]
\label{new:thm:optimal_execution}
Under Assumption~\ref{new:ass:compatibility}, if node policy $\hat\pi_{n(h)}$ is optimal for the execution mode induced by an arbitrary interruption subtree $\Beta|n(h)$, then
\[
V_{\Pi|n(h)}^{\Beta|n(h)}\leq V_{\Pi|n(h)}^{ME}.
\]
That is, ME is an optimal execution mode.
\end{theorem}
The theorem compares execution modes for a node policy optimal under its original mode, with the remaining node policies fixed. Compatibility matters because descendants can optimize rewards different from their ancestors' rewards. The multilevel counterexample (Figure~\ref{fig:Counterexample} in the appendix) shows that when compatibility fails, reselecting from above can reduce the evaluated task's value.

\paragraph{The two gaps.}
Fix the task tree and every node policy except $\hat\pi_{n(h)}$. Let $\Pi^-$ use the node policy obtained from HBOE-SMDP, and let $\Pi^*$ use the one obtained from HBOE-MDP. Denote the four returns by
\[
\begin{array}{c|cc}
&\mathrm{SME}&\mathrm{ME}\\\hline
\Pi^-&J_{\mathcal C}&J_{\mathcal B}\\
\Pi^*&J_{\mathcal D}&J_{\mathcal A}
\end{array}
\]
The rows distinguish node-policy solutions; the columns distinguish execution modes.
\begin{theorem}[Generalized SMDP Suboptimality]
\label{new:thm:generalized_suboptimality}
In an HRL problem with a Markov environment, with the task tree and all other node policies fixed, the HBOE-SMDP and HBOE-MDP solutions induce the SMDP-optimal and MDP-optimal node policies, respectively. Their returns satisfy $J_{\mathcal A}\geq J_{\mathcal B}$ and $J_{\mathcal C}\geq J_{\mathcal D}$. Under Assumption~\ref{new:ass:compatibility}, $J_{\mathcal B}\geq J_{\mathcal C}$ also holds.
\end{theorem}
Two-level HRL satisfies the stated compatibility alternative, giving the narrow ordering $J_{\mathcal A}\geq J_{\mathcal B}\geq J_{\mathcal C}\geq J_{\mathcal D}$. The resulting decomposition is
\begin{equation}
\label{new:eq:gap_decomposition}
\Delta J^{SMDP}=\Delta J^{E}+\Delta J^{P}=(J_{\mathcal{A}} - J_{\mathcal{B}}) + (J_{\mathcal{B}} - J_{\mathcal{C}}).
\end{equation}
Policy suboptimality is the loss $J_{\mathcal A}-J_{\mathcal B}$ from optimizing for SME and then deploying with ME; execution suboptimality is the loss $J_{\mathcal B}-J_{\mathcal C}$ from retaining SME at deployment. The decomposition separates execution design from policy optimization within a fixed task hierarchy.

\paragraph{A strict separation.}
The exact GridWorld construction in Appendix~\ref{sec:A} makes both losses visible. With a reward of $-1$ per frame, its four undiscounted returns are $J_{\mathcal A}=-8$, $J_{\mathcal B}=-11$, $J_{\mathcal C}=-12$, and $J_{\mathcal D}=-14$. Reselecting subtasks with the SME-trained node policy saves one frame in expectation; optimizing the node policy for ME saves three more. Executing that ME-optimal node policy with SME is worse than the original SME solution. Direct finite-trajectory calculations establish this undiscounted example; the general evaluation argument above uses uniform termination. The example separates the gains from changing execution and retraining for that execution mode.

\section{Decoupling behavior, targets, and deployment}
\label{new:sec:frameworks}

\paragraph{The design principle: three independently specified execution modes.}
The framework separates three roles: behavior determines which trajectories are collected; the target determines which execution process is optimized; deployment determines how the learned policy tree acts. The task tree, including termination $\beta$, fixes the reward objective and exit boundaries; interruption functions $\hat\beta$ specify execution separately for each role. This separation combines temporally extended exploration under SME behavior with frame-by-frame decisions in ME targets and deployment. Applying ME to behavior can interrupt a learning subtask before its reward boundary and leave subsequent actions to another task (Figure~\ref{fig:ME_offline}). The design therefore addresses the two sources of suboptimality separately: deployment execution addresses the execution gap, and the training target addresses the policy gap.

\begin{figure}[!htbp]
\centering
\includegraphics[width=\textwidth]{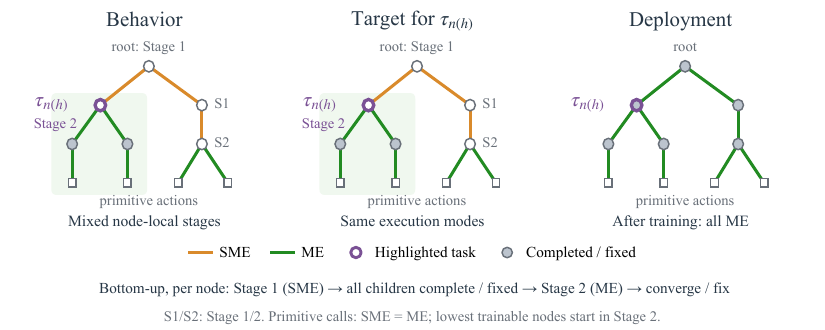}
\caption{\textbf{PSIF-2S matches behavior and target execution as training progresses up the tree.} In this mixed-stage snapshot, the highlighted node is in Stage~2 and its children are fixed, while its parent remains in Stage~1. Each caller's stage sets its outgoing calls. Deployment uses ME after training is complete.}
\label{new:fig:psif2s_modes}
\par\medskip
\includegraphics[width=\textwidth]{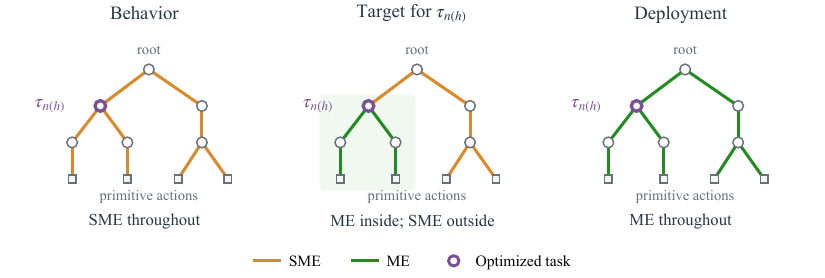}
\caption{\textbf{One policy tree, three execution specifications.} A multilevel PSIF-1S example: behavior uses SME; the target for $\tau_{n(h)}$ uses ME inside its subtree and SME outside, including the incoming call; deployment uses ME throughout. The optimized task can lie at any depth.}
\label{new:fig:multilevel_modes}
\end{figure}
Figures~\ref{new:fig:psif2s_modes} and~\ref{new:fig:multilevel_modes} compare the same hierarchy: PSIF-2S changes behavior and target execution together, node by node; PSIF-1S separates them throughout training. Appendix~\ref{new:app:thesis_expansion} gives the full execution-mode summary (Table~\ref{new:tab:design}) and training-flow comparison (Figure~\ref{new:fig:psif_flows}).

\paragraph{ESIF: change deployment execution.}
The Execution Suboptimality Improvement Framework keeps SME training and deploys the learned policy tree with ME. Its idealized comparison is $J_{\mathcal C}\to J_{\mathcal B}$: the node policies retain their SME optimization objectives. The Optimal Execution Theorem establishes the ME guarantee under its optimal-node-policy and compatibility conditions.

\paragraph{PSIF-2S: a bottom-up, node-wise training schedule.}
The Two-Stage Policy Suboptimality Improvement Framework assigns two training stages to \emph{every task node}, at any depth. In Stage 1, a node calls its children with SME and optimizes the corresponding SME objective. When all its children have completed Stage 2 and their policies are fixed, it enters its own Stage 2: it calls the children with ME and optimizes its MDP node policy on-policy. Only after this node policy has approximately converged is it fixed. The lowest task nodes, which call primitive actions directly, start in Stage 2 because their SME and ME objectives coincide. Higher nodes can therefore remain in Stage 1 while lower nodes train in Stage 2 or have already finished. The schedule propagates upwards until every node has completed Stage 2; the resulting tree is deployed with ME. In the OC instance, completing option learning permits the root to begin ME training with fixed option policies and termination functions.

\paragraph{PSIF-1S: local ME targets under global SME behavior.}
The One-Stage Policy Suboptimality Improvement Framework calls every task with SME during data collection. When optimizing $\hat\pi_{n(h)}$, its target execution uses ME \emph{inside that task's subtree}, while the execution outside remains SME. GHBE supplies the single-frame bootstrap for this node-specific target; GHBOE specifies the MDP node-policy objective. This construction applies recursively at arbitrary depth. It retains on-policy execution outside the optimized subtree and confines the execution mismatch to its interior. All nodes may continue learning without a manual stage transition, and the whole policy tree uses ME at deployment.

\paragraph{OC instance: separate values and explicit option-exit boundaries.}
The two-level OC instantiation illustrates these general rules. PSIF-1S maintains separate values for the root's ME objective and the options' objectives:
\begin{equation}
\label{new:eq:oc_target}
V^{hi}(s_t,\omega_t)\gets R(s_t,a_t)+\gamma\max_{\omega}V^{hi}(s_{t+1},\omega).
\end{equation}
For an option, the action lasts one frame, but option termination still determines its exit reward. The option-exit boundary uses the upper ME value:
\begin{equation}
\label{new:eq:oc_boundary}
U(s)=\max_\omega V^{hi}(s,\omega).
\end{equation}
SME trajectories supply the samples for both lower critics:
\begin{equation}
\label{new:eq:oc_lower_main}
\begin{aligned}
V^{lo}(s_t,\omega_t),\,V^{lo}(s_t,\omega_t,a_t)\gets{}&R(s_t,a_t)\\
&+\gamma\Big[(1-\beta_{\omega_t}(s_{t+1}))V^{lo}(s_{t+1},\omega_t)
+\beta_{\omega_t}(s_{t+1})U(s_{t+1})\Big].
\end{aligned}
\end{equation}
Here $\gamma$ is the discount factor. $V^{hi}$ induces upper selection; $V^{lo}$ supplies option-policy and termination updates. Appendix~\ref{new:app:oc_updates} gives their gradients and the Fourrooms termination-update variant. The resulting design preserves extended option trajectories while training toward ME deployment.

\section{Experimental analysis}
\label{new:sec:experiments}
The experiments isolate deployment execution, policy learning, and their interaction with transition randomness. OC/Fourrooms compares ESIF, PSIF-1S, and PSIF-2S within shared hierarchies; HAC tests execution improvement in goal-conditioned control; DAC tests a PSIF-2S-style second-stage ME switch from shared first-stage checkpoints. Figures~\ref{fig:Experiments} and~\ref{fig:HAC} provide the full OC and HAC curves and single-level controls, with environment details in Appendix~\ref{sec:D}. The learning and checkpoint evaluations below connect each intervention to its observed effect.
\IfFileExists{sections/learning_dynamics.tex}{\subsection{Learning behavior under decoupled execution}
\label{sec:learning-dynamics}
Figure~\ref{fig:main-learning} shows how decoupling shapes learning.
1Stage(PSIF) improves before the episode-500 switch while retaining SME
behavior; 2Stage(PSIF) reduces the recorded training-step count after this
scheduled switch and finishes below Baseline(ESIF). Direct ME behavior from initialization
(NaiveME) remains less effective in both displayed settings.
Together, these trajectories illustrate the role of temporally extended
exploration while the learned objective or subsequent execution changes.
The checkpoint comparisons below separate execution and policy effects
after this finite training budget; all twelve behavior-policy training
settings are provided in the appendix.

\begin{figure}[!htbp]
\centering
\includegraphics[width=\linewidth]{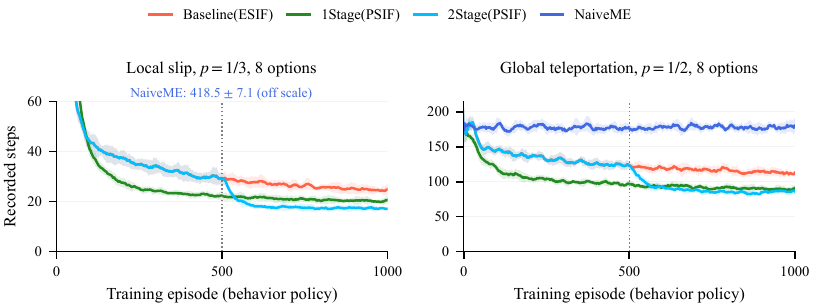}
\caption{Learning behavior with eight options: local slip ($p=1/3$) and
global teleportation ($p=1/2$). Lines use a trailing 20-episode mean;
bands show one SD across saved 35-seed batch curves (ten batches;
nine for teleportation NaiveME). The dotted line marks
the two-stage switch. Off-scale NaiveME reports its episode-901--1000
mean $\pm$ SD across batch-window means. Initial transients may exceed
the focused vertical range.}
\label{fig:main-learning}
\end{figure}
}{}
\IfFileExists{sections/new_experiments.tex}{\subsection{Stochasticity separates execution and policy effects}
\label{sec:stochastic-experiments}
The controlled transition sweep tests how stochasticity affects the gains
from execution and policy improvement. We extend Fourrooms to two transition
families with probabilities
$p\in\{1/6,1/3,1/2\}$ and either four or eight options. In \emph{local action
slip}, a feasible requested move is replaced, with probability $p$, by a
uniformly sampled traversable neighboring cell. In \emph{global
teleportation}, the corresponding event moves to a uniformly sampled free
cell; otherwise the original local-slip kernel is used. This yields twelve settings.

\begin{figure}[!htbp]
\centering
\includegraphics[width=\linewidth]{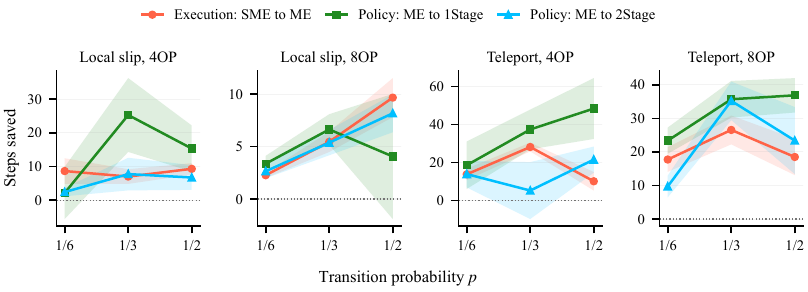}
\caption{\textbf{Execution and policy savings depend on the transition kernel.}
Positive values mean fewer steps: baseline SME minus ME (execution), or
Baseline(ESIF) minus PSIF under ME (policy). Means use checkpoints 910--1000;
bands show one SEM across ten policy indices, each evaluated for 35 episodes
per checkpoint. These are empirical step differences, not discounted optimal-value gaps.}
\label{fig:stochastic-gap}
\end{figure}

Figure~\ref{fig:stochastic-gap} separates changing the execution mode of a
fixed learned policy from learning policies for ME. In the teleportation
environment, the PSIF-1S savings relative to Baseline(ESIF) increase from
$18.6$ to $37.4$ to $48.5$ steps with four options, and from $23.4$ to $35.7$
to $36.9$ with eight options as $p$ increases. The local-slip and
execution-only curves are not uniformly monotone. The observed benefit
of ME-targeted learning thus depends on both the strength and spatial
structure of the random transitions. The full learning curves and evaluation
protocol appear in Appendix~\ref{app:new-experiments}.

\begin{table}[!htbp]
\centering
\caption{Effects across algorithm families; arrows mark the preferred direction. OC: ME steps ($p=1/3$, 910--1000 ep), mean $\pm$ SD over ten policy-window means. HAC: success rate (\%), mean $\pm$ SD across 10 seeds, averaged over the final five plotted checkpoints. DAC: final $0.1$M-step return, mean $\pm$ SD over 50 seeds; full second-stage comparison in Appendix~\ref{app:dac-continuation}. Effects are relative changes (\%).}
\label{tab:cross-family}
\small
\setlength{\tabcolsep}{3pt}
\begin{tabular}{@{}llrrr@{}}
\toprule
Algorithm / environment & Change & Reference & Modified & Effect \\
\midrule
OC / Slip (4 opts.) $\downarrow$ & PSIF-1S & $41.8\pm35.1$ & $16.5\pm1.8$ & $-60.6\%$ \\
OC / Slip (8 opts.) $\downarrow$ & PSIF-1S & $23.4\pm3.2$ & $16.8\pm2.1$ & $-28.4\%$ \\
OC / Teleport (4 opts.) $\downarrow$ & PSIF-1S & $110.2\pm29.9$ & $72.7\pm17.3$ & $-34.0\%$ \\
OC / Teleport (8 opts.) $\downarrow$ & PSIF-1S & $104.8\pm18.5$ & $69.1\pm11.6$ & $-34.1\%$ \\
\midrule
HAC / random AntReacher $\uparrow$ & ESIF & $49.1\pm5.9$ & $72.1\pm2.9$ & $+46.7\%$ \\
HAC / UR5 $\uparrow$ & ESIF & $73.1\pm4.4$ & $76.3\pm4.1$ & $+4.3\%$ \\
HAC / AntReacher $\uparrow$ & ESIF & $91.7\pm1.4$ & $94.0\pm1.1$ & $+2.5\%$ \\
HAC / AntFourrooms $\uparrow$ & ESIF & $69.1\pm5.7$ & $71.3\pm6.0$ & $+3.2\%$ \\
\midrule
DAC / HalfCheetah (2 opts.) $\uparrow$ & Stage-2 ME & $2257\pm506$ & $2269\pm607$ & $+0.54\%$ \\
DAC / HalfCheetah (8 opts.) $\uparrow$ & Stage-2 ME & $2111\pm698$ & $2268\pm636$ & $+7.41\%$ \\
\bottomrule
\end{tabular}
\end{table}

\begin{samepage}
Table~\ref{tab:cross-family} connects these interventions across algorithm
families. At fixed $p=1/3$, PSIF-1S reduces OC's ME evaluation steps for
both transition families and option counts. ESIF raises HAC's displayed
success rate in each environment. DAC~\citep{zhang2019dac} provides a
PSIF-2S-style second-stage comparison in HalfCheetah: a shared first-stage
checkpoint branches into learned-termination and ME continuations. The two-
and eight-option settings have positive mean return changes of $+0.54\%$
and $+7.41\%$, respectively. The complete option-count comparison appears
in Appendix~\ref{app:dac-continuation}.
These comparisons connect execution improvement and ME-targeted training
to option-based, goal-conditioned, and continuous-control learning.
\par
\end{samepage}
}{}
\section{Summary and future work}
\label{new:sec:discussion}

We have separated execution suboptimality from policy suboptimality within a fixed task hierarchy. UVFH and the four-stage Bellman recursion provide a common language for the task objective, policy choice, and decision timing. The resulting improvement criteria and loss decomposition lead to ESIF and two multilevel PSIF procedures, with behavior, target, and deployment execution specified independently. The discrete-control comparisons, continuous-control extensions, and controlled transition sweeps connect this design principle to its empirical effects.

Appendix~\ref{new:app:related_context} relates this execution-centered view to existing HRL approaches.

\paragraph{Future work.}
We identify three directions that follow from this separation: learning interruption rules alongside node policies, adapting training to parent-induced changes in state--task visitation, and improving off-policy value estimation for the local ME targets of PSIF-1S. The improvement criteria in Section~\ref{new:sec:theory} and the learning conditions in Appendix~\ref{new:app:thesis_expansion} provide a basis for these extensions. Together, they open a route to execution-aware learning in deeper hierarchies and continuous-control implementations while preserving temporal abstraction for exploration.

\FloatBarrier
\label{new:main_end}
\clearpage
\section*{Reproducibility statement}
The appendix provides the full formal definitions, theorem statements and proofs, together with the framework algorithms and their instantiations. It also details the experimental environments, hyperparameter settings, training and evaluation protocols, and statistical aggregation procedures needed to reproduce the reported results.

\section*{Ethics statement}
This work is a methodological study of reinforcement learning algorithms. It does not involve human subjects, personal data, or real-world deployment. We believe that there are no ethical concerns beyond the general question of how empirical claims in the field are substantiated, which this paper addresses directly.

\section*{AI use statement}
We used generative AI tools for readability editing of the manuscript. We have not used generative AI tools for proposing or refining the hypotheses, designing or critiquing the methodology or experiments, implementing the methods, interpreting the results, formulating the mathematical claims, or writing the proofs; generating synthetic datasets and qualitative data analysis are not applicable to this work. We have reviewed all AI-assisted work: every number in the paper was read back from the released derived data files by an author, the figure-generation code was checked against those files, and every theorem statement and proof was written and verified by the authors. We take responsibility for the final content of this work, including text, claims, and artifacts produced with the aid of generative AI.

\bibliography{ref}
\bibliographystyle{iclr2027_conference}
\clearpage
\appendix
\let\submissionsection\section
\renewcommand{\section}{\FloatBarrier\submissionsection}
\section{Guide to the supplementary material}
\label{new:app:guide}
The supplementary material contains the stochastic-environment and DAC experiments, an interpretation of the unified formalization, and the complete technical development. Appendices~\ref{sec:A} and~\ref{sec:B} give the GridWorld calculations and full formalization, including the multilevel counterexample. Appendix~\ref{sec:C} gives all three framework algorithms and their option-based and goal-based instantiations. Appendix~\ref{sec:D} documents the environments, hyperparameters, continuous-control comparisons, and visualizations. The extended preliminaries and analyses develop the arguments underlying the main-text presentation.

\IfFileExists{sections/new_experiments_appendix.tex}{\section{Additional stochastic and continuous-control experiments}
\label{app:new-experiments}

\subsection{OC transition families and measurement protocol}
The stochastic-transition Option-Critic (OC) experiments in Fourrooms keep the room geometry, task, option counts,
and learning hyperparameters fixed while varying the transition kernel.
For a requested move into a traversable cell, the local-slip kernel executes
the requested move with probability $1-p$ and samples uniformly among all
traversable neighboring cells with probability $p$; the latter set includes
the requested destination. A requested move into a wall leaves the agent
in place and does not trigger a random event. For a feasible requested move,
the global-teleportation kernel instead samples a free cell uniformly
(including the goal) with probability $p$; otherwise it applies the
local-slip kernel with fixed slip-event probability $1/3$. Thus $p$ denotes
an event probability conditional on a feasible requested move. The
unconditional probability of an unintended direction also depends on the
traversable neighbors and the requested action.
We evaluate $p\in\{1/6,1/3,1/2\}$ with four and eight options.
Each training run contains 2000 episodes, capped at 1000 environment steps
per episode. The displayed learning curves cover episodes 1--1000, and the
two-stage schedule switches at episode 500.

The training CSVs contain ten worker-batch averages, each formed
from 35 training seeds. The displayed uncertainty is dispersion across
these saved batch averages; the 350 individual seed trajectories were not
separately retained in the CSVs. The NaiveME global-teleportation
condition at $p=1/2$ contains nine saved batches, representing 315 seeds.
The training logger records the zero-based terminal step index and saves
integer-formatted batch averages. The learning curves use that recorded
scale, and checkpoint tests count actual steps; these two measurements
have distinct scales and sampling units.

For the checkpoint evaluation sweep, ten policy indices are evaluated every
ten training episodes from episode 10 through 1990. Each saved policy is
tested for 35 episodes, using the same seeded evaluation scheme for SME and
ME and an $\epsilon$-greedy option selector with $\epsilon=0.1$.
The saved checkpoint-test entry is the rounded mean of those 35 episodes.
The summary uses episodes 910--1000: first average the ten checkpoints for
each policy index, then report the mean and sample standard deviation across
the ten policy indices. A comparison of two conditions uses their matching
policy indices; the standard-error bands in
Figure~\ref{fig:stochastic-gap} are computed from those ten differences.
This paired protocol measures execution and policy effects for the learned
checkpoints at a fixed training budget. The theoretical optimal-value gap
compares Bellman-optimal node policies with the other policies fixed.

\subsection{All twelve OC behavior-policy learning curves}
Figures~\ref{fig:new-training-original} and~\ref{fig:new-training-random} report the OC learning curves in the local-slip and global-teleportation Fourrooms environments, respectively. Table~\ref{tab:all-checkpoint-results} gives the corresponding controlled checkpoint comparisons for OC: the same baseline policies are evaluated under SME and ME, and the 1Stage and 2Stage policies are evaluated under ME.
The algorithm names follow the experimental implementation:
Baseline(ESIF), 1Stage(PSIF), 2Stage(PSIF), and NaiveME. The ESIF designation
of the baseline concerns its deployment intervention. Baseline and 1Stage
collect with SME behavior, while the staged method changes behavior at the
prescribed switch. These curves show behavior during learning; the
controlled checkpoint comparisons separate the gains from
execution and policy improvement.

\begin{figure}[p]
\centering
\includegraphics[width=\linewidth]{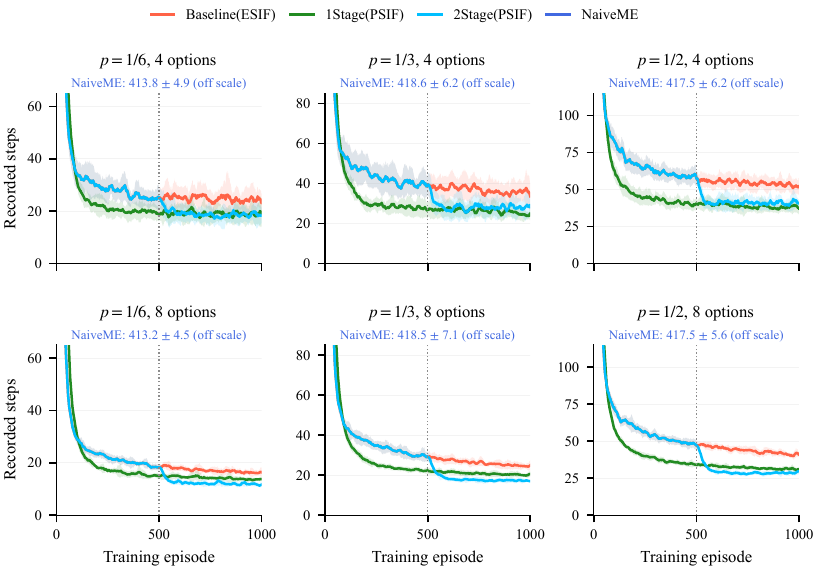}
\caption{OC on Fourrooms: all six local-slip learning curves, with four options in the upper
row and eight in the lower row. Curves apply a trailing 20-episode mean to
each saved worker-batch trajectory; bands are one sample standard deviation
across the smoothed batch trajectories. The dotted line marks the
two-stage switch at episode 500. Focused vertical scales can clip the
initial transient. Off-scale NaiveME labels report the unsmoothed
episode-901--1000 mean $\pm$ SD across worker-batch window means.
Legends and off-scale labels are outside the plotting areas.}
\label{fig:new-training-original}
\end{figure}

\begin{figure}[p]
\centering
\includegraphics[width=\linewidth]{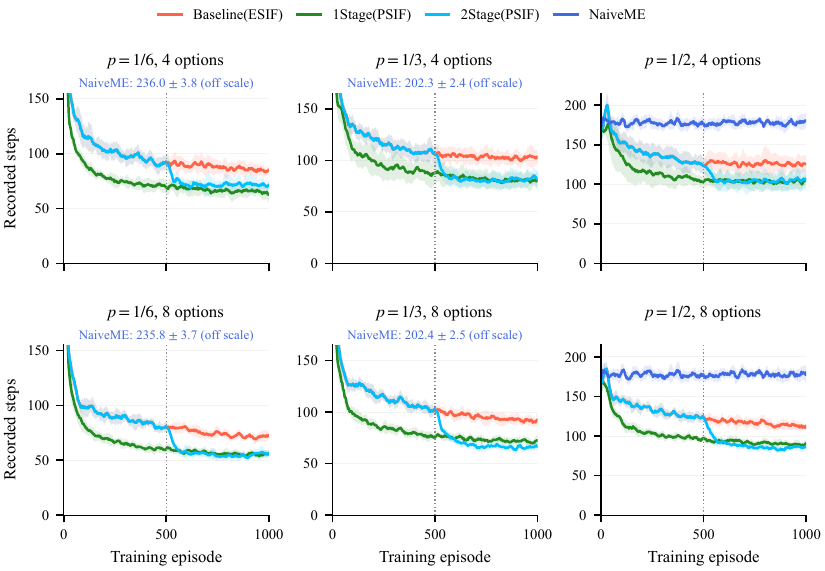}
\caption{OC on Fourrooms: all six global-teleportation learning curves, with the same
conventions as Figure~\ref{fig:new-training-original}. At $p=1/2$, NaiveME
is included directly in the plotting range; its nine saved worker batches
are retained without imputation.}
\label{fig:new-training-random}
\end{figure}

\begin{table}[t]
\centering
\caption{Complete OC checkpoint evaluation on Fourrooms at episodes 910--1000: mean $\pm$ sample SD of the ten policy-window means (steps, lower is better). Within each checkpoint, each saved policy is evaluated for 35 episodes.}
\label{tab:all-checkpoint-results}
\small
\setlength{\tabcolsep}{3pt}
\begin{tabular}{@{}lccrrrr@{}}
\toprule
Environment & Opts. & $p$ & Baseline SME & Baseline ME & 1Stage ME & 2Stage ME \\
\midrule
Slip & 4 & $1/6$ & $31.4\pm24.6$ & $22.7\pm14.7$ & $20.5\pm16.6$ & $20.3\pm13.5$ \\
Slip & 4 & $1/3$ & $48.8\pm38.8$ & $41.8\pm35.1$ & $16.5\pm1.8$ & $34.0\pm27.2$ \\
Slip & 4 & $1/2$ & $57.0\pm16.8$ & $47.7\pm14.1$ & $32.3\pm21.5$ & $40.9\pm13.7$ \\
Slip & 8 & $1/6$ & $17.5\pm3.0$ & $15.2\pm2.5$ & $11.9\pm0.5$ & $12.6\pm0.7$ \\
Slip & 8 & $1/3$ & $28.9\pm4.4$ & $23.4\pm3.2$ & $16.8\pm2.1$ & $18.0\pm2.2$ \\
Slip & 8 & $1/2$ & $47.8\pm12.2$ & $38.1\pm7.9$ & $34.1\pm17.9$ & $30.0\pm5.7$ \\
Teleport & 4 & $1/6$ & $88.8\pm23.3$ & $75.0\pm20.2$ & $56.3\pm29.5$ & $61.0\pm19.4$ \\
Teleport & 4 & $1/3$ & $138.4\pm34.6$ & $110.2\pm29.9$ & $72.7\pm17.3$ & $104.9\pm73.4$ \\
Teleport & 4 & $1/2$ & $168.4\pm61.5$ & $158.2\pm60.3$ & $109.8\pm21.9$ & $136.7\pm47.0$ \\
Teleport & 8 & $1/6$ & $86.2\pm17.0$ & $68.5\pm11.7$ & $45.1\pm7.5$ & $58.8\pm14.7$ \\
Teleport & 8 & $1/3$ & $131.3\pm27.7$ & $104.8\pm18.5$ & $69.1\pm11.6$ & $69.6\pm10.6$ \\
Teleport & 8 & $1/2$ & $144.6\pm22.4$ & $126.1\pm30.0$ & $89.2\pm22.0$ & $102.7\pm12.9$ \\
\bottomrule
\end{tabular}
\end{table}

\FloatBarrier
\subsection{Second-stage DAC comparison on HalfCheetah}
\label{app:dac-continuation}
The DAC study compares second-stage continuations from matched first-stage
checkpoints at $3\times10^6$ environment steps through $4\times10^6$ steps,
using 50 paired runs per option count. The reference branch retains DAC's
learned termination; the ME branch reselects an option each frame and fixes
its termination head to one. This PSIF-2S-style actor--critic continuation
adapts intra-option policies, while the idealized node-wise schedule freezes
completed children. The continuation restores network and normalization
parameters and reinitializes other training state.

Figure~\ref{fig:new-dac} and Table~\ref{tab:dac-second-stage-all} report the
two- and eight-option comparisons.
The displayed quantity is training return: curves use 20,000-step bins and
the table summarizes the final 100,000 steps.

\begin{figure}[!htbp]
\centering
\includegraphics[width=\linewidth]{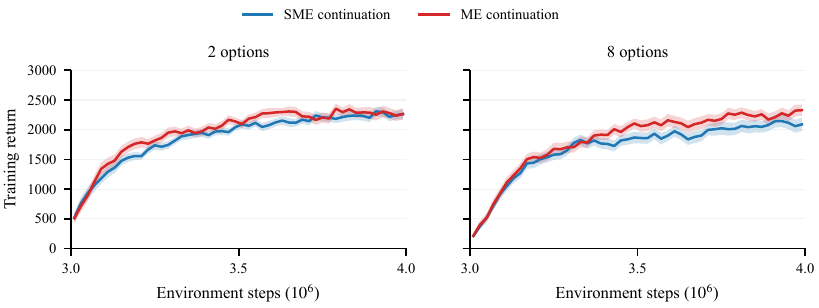}
\caption{\textbf{Second-stage DAC comparison with two and eight options.}
Learned-termination and ME continuations start from the shared $3$M-step
checkpoint. Curves show the observed $3$M--$4$M segment in 20,000-step bins;
bands are one standard error across 50 runs. Their summary statistics
are given in Table~\ref{tab:dac-second-stage-all}.}
\label{fig:new-dac}
\end{figure}

\begin{table}[!htbp]
\centering
\caption{DAC second-stage comparison on HalfCheetah for the valid two- and eight-option settings. Returns use the final $0.1$M steps: mean $\pm$ sample SD over 50 paired seeds. Effects are relative changes (\%).}
\label{tab:dac-second-stage-all}
\small
\setlength{\tabcolsep}{4pt}
\begin{tabular}{@{}crrr@{}}
\toprule
Options & SME return $\uparrow$ & ME return $\uparrow$ & Effect \\
\midrule
2 & $2257\pm506$ & $2269\pm607$ & $+0.54\%$ \\
8 & $2111\pm698$ & $2268\pm636$ & $+7.41\%$ \\
\bottomrule
\end{tabular}
\end{table}

}{}
\section{Additional interpretation of the unified formalization}
\label{new:app:thesis_expansion}

The unified formalization separates task objectives from decision authority, providing the basis for independent behavior, target, and deployment execution. The complete definitions and proofs appear in Appendix~\ref{sec:B}.

\begin{figure}[!htbp]
\centering
\includegraphics[width=\textwidth]{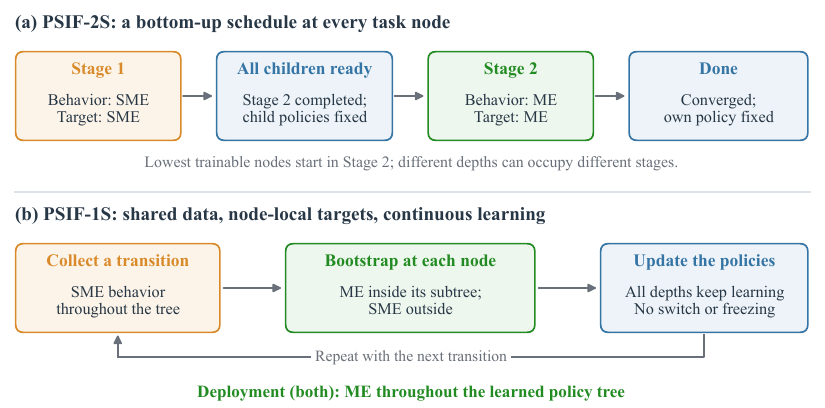}
\caption{\textbf{PSIF-2S changes the training schedule; PSIF-1S separates behavior from the target.} In PSIF-2S, behavior and target modes describe the current node's calls to its children, not a global switch. The schedule propagates bottom-up. PSIF-1S keeps SME behavior and updates node-local ME targets throughout training. Both apply at arbitrary depth and deploy with ME.}
\label{new:fig:psif_flows}
\end{figure}
Figure~\ref{new:fig:psif_flows} supplements the main-text tree illustrations in Figures~\ref{new:fig:multilevel_modes} and~\ref{new:fig:psif2s_modes} with the training schedule. Panel (a) shows the node-wise stage transitions of PSIF-2S; panel (b) shows the repeated SME collection and node-local target updates of PSIF-1S. Table~\ref{new:tab:design} records the execution modes and idealized endpoints for all frameworks.

\begin{table}[!htbp]
\centering
\caption{\textbf{Execution rules for hierarchies of arbitrary depth.} In PSIF-2S, the stage is local to each task node. In PSIF-1S, the target is specific to the node being optimized. Endpoints denote the conditional idealized comparisons, not finite-training guarantees.}
\label{new:tab:design}
\small
\setlength{\tabcolsep}{4pt}
\begin{tabular}{@{}lp{2.7cm}p{4.1cm}cc@{}}
\toprule
Framework & Behavior & Target & Deployment & Endpoint\\
\midrule
Baseline & SME throughout & SME throughout & SME & $J_{\mathcal C}$\\
ESIF & SME throughout & SME throughout & ME & $J_{\mathcal B}$\\
PSIF-2S & Node-local SME $\to$ ME & Same node-local modes as behavior & ME & $J_{\mathcal A}$\\
PSIF-1S & SME throughout & ME within the updated subtree; SME outside it & ME & $J_{\mathcal A}$\\
\bottomrule
\end{tabular}
\end{table}

\paragraph{Why two trees, and why four entries in a task node?}
The task tree answers four questions for each task: which subtasks can it call, when does its objective terminate, what reward is accumulated at each frame, and what reward is received on exit? The last question cannot generally be absorbed into a single shared environmental reward. A task can exit because of its own termination event or because an ancestor terminates, and these exits can require different boundary values. The tuple $(\mathcal T',\beta,R,\hat R)$ records these distinctions explicitly. Task termination propagates down the hierarchy, with the highest terminating ancestor determining the boundary. The external task $\tau_{out}=\tau_{n(-1)}$ calls the root task and receives its exit, whereas action leaves produce one frame of transition and have $\beta_{leaf}\equiv1$.

The execution tree answers a different question: which node policy may act now? An execution node pairs a node policy with a decision interruption function. The policy tree $\Pi$ contains the former; the interruption tree $\Beta$ contains the latter. This separation permits changing execution without changing the task's reward accumulation rule. It also distinguishes learning a task decomposition from learning a scheduler. For example, changing an option's intrinsic termination rule can change the subproblem it must learn, whereas changing only its interruption rule changes when another node can take control.

\paragraph{Why the extended state contains decision authority.}
In flat RL, a state value evaluates the process before selecting an action, and an action value evaluates it after an action is selected. In a hierarchy there are intermediate decision points: a parent can have selected a child task while that child has not yet selected its own child. The extended state $(s,\tau_{n(i)})$ identifies all such positions. The conditioning task $\tau_{n(h)}$ separately determines which reward and exit boundary are evaluated. Hence the two task indices serve different purposes. For a countable task tree with $N$ nodes, the formalization describes $|S|\times N$ extended-state values for each conditioning task, and $|S|\times N^2$ values across all conditioning tasks. This describes the value objects; it does not require estimating every one with a separate function approximator.

\paragraph{One environment frame, four value stages.}
The decision stage follows the policy tree downwards; the interaction stage applies an action and obtains the evaluated task's reward. The termination stage then either closes that task's reward accumulation with an exit reward or passes its remaining value to the preparation stage. The preparation stage redistributes decision authority through the interruption tree. These four stages form a closed recursion. They distinguish transitions of task objectives, represented by $\gamma$, from transitions of decision authority, represented by $\hat\gamma$. In particular, knowing the probability that a task terminates does not by itself specify how frequently a target policy reselects subtasks.

\paragraph{Why SME behavior can support ME targets.}
With SME behavior, if an ancestor terminates and decision authority leaves a descendant's subtree, that descendant's task also terminates. Its sampled trajectory can therefore extend to its own hierarchical boundary. With ME behavior, authority can leave a descendant while the descendant's task objective has not terminated. Subsequent actions can then be selected by another task. Using them to update the interrupted task without accounting for the change introduces an off-policy discrepancy. Separating execution roles preserves temporal abstraction for behavior while allowing the target and deployment to exploit ME decisions.

PSIF-2S resolves this scheduling issue through two training stages for each task node. Stage 1 uses SME to call children and optimizes the corresponding SME objective. Completion of every child's Stage 2, with those child policies fixed, permits the parent to enter its own Stage 2. The parent then calls children with ME and optimizes its MDP node policy; it is fixed only after completing this optimization. Lowest task nodes start in Stage 2 because their children are primitive actions and their SME and ME objectives coincide. Different nodes may therefore be in Stage 1, Stage 2, or finished at the same time. It is the caller's own stage that governs its execution of its children, and the bottom-up progression ends when all task nodes have finished Stage 2. These are node-local training stages, with behavior and target execution matched for the node being optimized.

PSIF-1S instead keeps SME behavior throughout the hierarchy and forms a separate target for each updated node: descendants within that node's subtree are called with ME, while the part outside remains SME. GHBE supplies its single-frame value recursion, and GHBOE specifies its MDP node-policy objective. The behavior can thus preserve the temporal structure useful for exploration without imposing the same execution mode on the learning target or final deployment. This construction concerns an arbitrary node and its descendants, not a fixed pair of upper and lower layers. The OC instantiation below is a two-level example of these general rules.

\paragraph{Scope of the value guarantees and learning updates.}
The value-evaluation proof uses bounded rewards and a uniform positive probability of task termination. The optimal-execution result further requires a node policy optimal for its original interruption tree and the stated task-subtree compatibility condition. The multilevel counterexample identifies why compatibility matters: independent descendant objectives can make a new scheduling opportunity worse for an ancestor. The idealized framework endpoints compare the specified node-policy optima, without asserting finite-training convergence; partially trained agents are evaluated empirically. In PSIF-2S, SME visitation can leave rarely called subtasks undertrained, and subsequent parent improvement can shift the state--task distribution seen by fixed children. PSIF-1S can also encounter state--option distribution shift and function-approximation error. Multi-frame bootstrapping requires accounting for off-policy data, while single-frame bootstrapping can accumulate approximation error over long horizons. These learning effects are separate from the exact value comparisons.

\subsection{Complete PSIF-1S updates for Option-Critic}
\label{new:app:oc_updates}

This two-level instantiation uses the environment reward $R$ and discount factor $\gamma$. The agent collects trajectories with SME. Its upper value estimates the ME target, while two lower values estimate the option objective:
\begin{equation}
\label{new:eq:oc_hi_complete}
V^{hi}(s_t,\omega_t)\gets R(s_t,a_t)+\gamma\max_\omega V^{hi}(s_{t+1},\omega).
\end{equation}
The assignment symbol denotes fitting the left-hand value estimates to the right-hand target on sampled inputs. For the option's one-frame action transitions, SME and ME coincide. At option termination, the exit reward uses the upper target policy's ME value $U(s)=\max_\omega V^{hi}(s,\omega)$. The lower updates used in this paper are
\begin{equation}
\label{new:eq:oc_offtrace}
\begin{aligned}
V^{lo}(s_t,\omega_t),\,V^{lo}(s_t,\omega_t,a_t)
\gets{}& R(s_t,a_t)+\gamma\Bigl[\bigl(1-\beta_{\omega_t}(s_{t+1})\bigr)V^{lo}(s_{t+1},\omega_t)\\
&+\beta_{\omega_t}(s_{t+1})\max_\omega V^{hi}(s_{t+1},\omega)\Bigr].
\end{aligned}
\end{equation}
This boundary trains options to support the upper ME target while retaining SME interaction. Both critics bootstrap from the observed next state after one sampled action; the target execution specifies its continuation value.

\paragraph{Policy and termination updates.}
The upper node policy is induced by $V^{hi}$. The option-policy and termination gradients use the lower values, with the termination baseline specified separately from the option-exit target:
\begin{equation}
\label{new:eq:oc_termination}
\begin{aligned}
\Delta\vartheta &\sim-\alpha_\vartheta\frac{\partial\beta_\omega(s_t)}{\partial\vartheta}
\bigl(V^{lo}(s_t,\omega)-V^{lo}(s_t)\bigr),\\
V^{lo}(s)&=\max_\omega V^{lo}(s,\omega).
\end{aligned}
\end{equation}
\begin{equation}
\label{new:eq:oc_policy}
\Delta\theta\sim\alpha_\theta\frac{\partial\pi_\omega(a_t\mid s_t)}{\partial\theta}
\bigl(V^{lo}(s_t,\omega,a_t)-V^{lo}(s_t,\omega)\bigr).
\end{equation}
The termination baseline $\max V^{lo}$ and the option-exit reward $U(s)=\max V^{hi}$ are distinct quantities with distinct roles.

\paragraph{Fourrooms termination-update variant.}
The runs labeled \texttt{1Stage(PSIF)} combine the upper-ME critic boundary in Equation~\eqref{new:eq:oc_offtrace} with the termination advantage $V^{hi}(s,\omega)-\max_{\omega'}V^{hi}(s,\omega')+\texttt{delib}$ in \texttt{A\_Omega}, where \texttt{delib} is the deliberation term. Their termination update therefore uses upper values, while the framework update in Equation~\eqref{new:eq:oc_termination} uses lower values. The reported curves and statistics evaluate this upper-value termination variant.

\subsection{Context within hierarchical reinforcement learning}
\label{new:app:related_context}
Hierarchical reinforcement learning uses task decomposition to organize temporally extended behavior; surveys describe both its algorithmic diversity and the resulting design choices~\citep{DEDS03HRL,ACMCS21HRLSurvey}.

Options permit interruption \citep{AI99MDPsSMDPs}. Our analysis separates its execution benefit from the loss induced by an SME training objective and formalizes both for nested tasks. Representation-induced restrictions \citep{ICLR19RLHRL} and the sample efficiency of temporal abstraction \citep{NIPS20EfficiencyHRL,robert2023sample} concern complementary questions. UVFA \citep{ICML15UVFA} conditions values on goals, while UVFH also records decision authority and an independently chosen execution tree. Recent analyses of option duration \citep{drappo2023option} and alternative option-induced MDPs \citep{li2023hit} reinforce the distinction between a hierarchy's structure and its execution. Appendix~\ref{related} discusses skill discovery, option learning, multitask learning, and goal-space learning.

\section{Extended background}
\begin{figure*}[t]
\centering
\includegraphics[width=0.9\textwidth]{
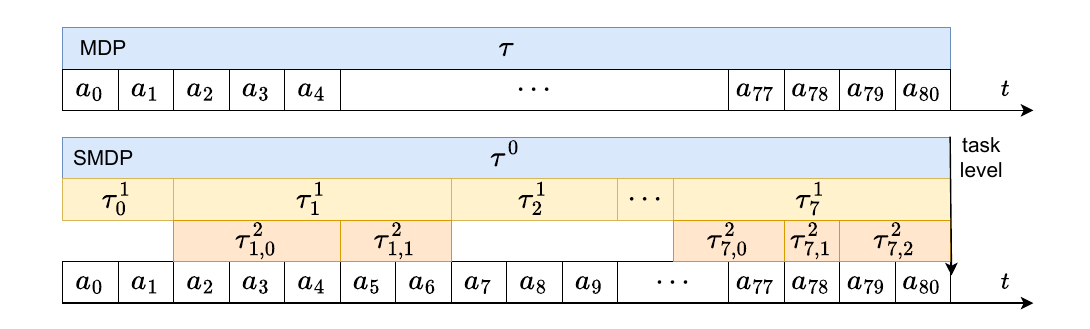
}
\caption{Differences between MDP and SMDP: calling actions vs. calling subtasks. The superscript of $\tau$ denotes its hierarchical level, while the subscript denotes its called index $\hat t$. The subscript of actions $a$ indicates the timestep index $t$.
}
\label{fig:MDP_and_SMDP}
\end{figure*}
\subsection{Markov Decision Processes}
\label{sec:2.1}
A \textcolor{black}{\textbf{task}} $\tau$ is described as a \textbf{Markov Decision Process} (MDP) defined by the six-tuple $(\mathcal S,\mathcal A, T, R_\tau, \alpha_\tau, \beta_\tau)$. This MDP comprises a state space $\mathcal S$, a action space $\mathcal A$, a transition kernel $T: \mathcal S\times\mathcal A\to\Delta(\mathcal S)$ (where $\Delta(\mathcal S)$ denotes probability distributions over $\mathcal S$), a reward function $R_\tau:\mathcal S\times\mathcal A\to\mathbb R$, a task  initial-state distribution $\alpha_\tau\in\Delta(\mathcal S)$, and a termination function $\beta_\tau:\mathcal S\to[0,1]$ which means the probability of exiting task $\tau$ upon visiting a state $s$. \textcolor{black}{$1-\beta_\tau$ plays a role analogous to the discount factor $\gamma_\tau$.}
We define one unit of timestep in which the agent takes an action and interacts with the environment as a \emph{frames}, and denote the current frame index as $t$. When the agent interacts with the environment by calling the policy
$\pi_{\tau}: \mathcal{S} \to \Delta \mathcal{A}$,
its value function is defined as
$
V_{\pi}(s|\tau) =
\mathbb{E}_{\pi_{\tau},\, s_0 = s}
\left[
\sum_{t=0}^{\infty}
\left(
R_{\tau}(s_t, a_t)
\prod_{k=0}^{t-1} (1 - \beta_{\tau}(s_k))
\right)
\right],
$
where $\pi$ is the set of policies for all tasks,\footnote{We will extend and redefine the concepts of certain traditional symbols. Notations are in~\cref{notations}}
$\forall \tau, \pi_{\tau} \in \pi$.
The optimization objective of $\pi_{\tau}$ is to maximize $V_{\pi}(\cdot | \tau)$,
and different $\tau$ will have different optimization objectives due to different $\beta_{\tau}$.
The expected return of task $\tau$ is
$
J(\pi | \tau) = \mathbb{E}_{s \sim \alpha_{\tau}} V_{\pi}(s | \tau)
$.
\subsection{Semi-Markov Decision Processes}
\label{sec:2.2}
Tasks affect the optimization objective of their policies and induce dynamic transitions in the environment when executed by the agent, similar to the effects of actions. Therefore, as illustrated in~\cref{fig:MDP_and_SMDP}, in HRL it is possible to call other tasks multiple times to accomplish $\tau$ instead of directly calling actions. We refer to this stochastic process with variable-duration transitions as a \textbf{Semi-Markov Decision Process} (SMDP), and the tasks that can be called under $\tau$ are called the \textcolor{black}{\textbf{subtasks}} of $\tau$.
In the formalization, we replace $\mathcal{A}$ in the MDP tuple with the subtask set $\mathcal{T}'_{\tau}$ of $\tau$, and define the SMDP of task $\tau$ as $(\mathcal{S},\mathcal{T}'_{\tau},\bar T_{\tau},\bar R_{\tau},\alpha_{\tau},\beta_{\tau})$. Here, $\bar T_{\tau}: \mathcal{S} \times \mathcal{T}'_{\tau} \to \Delta \mathcal{S}$ is the subtask transition function, and $\bar R_{\tau}: \mathcal{S} \times \mathcal{T}'_{\tau} \to \mathbb{R}$ is the subtask reward function.
Since tasks at different hierarchical levels operate on different time scales, we use $\hat t$ to record the number of subtasks that have been called by a given task. We refer to the policy by which $\tau$ calls its subtasks as the node policy $\hat\pi_{\tau}: \mathcal{S} \to \Delta \mathcal{T}'_{\tau}$, and we denote the mapping for $\tau$ that starts from $\hat\pi_{\tau}$ and keeps selecting subtasks until an action is eventually chosen, while keeping the original notation, as the policy $\pi_{\tau}: \mathcal{S} \to \Delta \mathcal{A}$.
The value function that the node policy $\hat{\pi}_{\tau}$ needs to maximize can be simply formalized as
$
V_{\pi}(s | \tau) =
\mathbb{E}_{\tau' \sim \hat{\pi}_{\tau},\, s_0 = s}
\left[
\sum_{\hat{t}=0}^{\infty}
\left(
R_{\tau}(s_{\hat{t}}, \tau'_{\hat{t}})
\prod_{k=0}^{\hat{t}-1} (1 - \beta_{\tau}(s_k))
\right)
\right]
$.
\footnote{This calculation ignores the case where a subtask is terminated due to the termination of its higher-level task during its execution.
}
\subsection{Problem setup}
In this paper, we divide the algorithm into two phases: the training phase and the deployment phase.
\textbf{Training phase:} The agent interacts with the environment to learn appropriate policies.
\textbf{Deployment phase:} The agent needs to call the learned policies in a reasonable manner.
We examine the final performance exhibited by the agent when interacting with the environment during the deployment phase while executing policies in a certain manner.\footnote{This setup is more reasonable in RL, for example, by using $\epsilon$-greedy during training and greedy during deployment, or softmax during training and argmax during deployment.}The main motivation for this setup is that during training, the agent needs to explore the environment, and policy execution involves an exploration mechanism. However, during deployment, more emphasis should be placed on the optimal performance that the learned policies can achieve in terms of exploitation. We find that this exploration-exploitation tradeoff takes on a new form in HRL.
\section{Detailed problem formulation}
In this section, we analyze the impact of Semi-Markov Decision Processes on the agent during training and deployment through a simple case study. We point out two distinctly different execution modes when calling subtasks in HRL.
\begin{definition}[\textbf{Markov Execution, ME}]
Task $\tau$ makes a decision and calls a new subtask $\tau'$ at every frame, meaning that each called $\tau'$ generates only a single-frame transition. We refer to this execution mode, in which each transition lasts only one frame, as task $\tau$ calling $\hat\pi_{\tau}$ in a ME manner. When all tasks of the agent call their node policies in the ME manner, we refer to the agent as calling $\pi$ in the ME manner.
\textcolor{black}{The value function under ME is formalized as
$
V_{\pi}^{\text{ME}}(s | \tau) =
\mathbb{E}_{\tau' \overset{\text{ME}}{\sim} \hat{\pi}_{\tau},\, s_0 = s}
\left[
\sum_{\hat{t}=0}^{\infty}
\left(
R_{\tau}(s_{\hat{t}}, \tau'_{\hat{t}})
\prod_{k=0}^{\hat{t}-1} (1 - \beta_{\tau}(s_k))
\right)
\right]
$.}
\end{definition}
\begin{definition}[\textbf{Semi-Markov Execution, SME}]
After task $\tau$ decides and calls a subtask $\tau'$, $\tau'$ is executed in full until it terminates due to termination function $\beta_{\tau'}$. We refer to this execution mode, in which each transition lasts for the entire execution of the subtask, as task $\tau$ calling $\hat\pi_{\tau}$ in a SME manner. When all tasks of the agent call their node policies in the SME manner, we refer to the agent as calling $\pi$ in the SME manner. \textcolor{black}{The value function under SME is formalized as
$
V_{\pi}^{\text{SME}}(s | \tau) =
\mathbb{E}_{\tau' \overset{\text{SME}}{\sim} \hat{\pi}_{\tau},\, s_0 = s}
\left[
\sum_{\hat{t}=0}^{\infty}
\left(
R_{\tau}(s_{\hat{t}}, \tau'_{\hat{t}})
\prod_{k=0}^{\hat{t}-1} (1 - \beta_{\tau}(s_k))
\right)
\right]
$.}
\end{definition}
As shown in~\cref{fig:ME_and_SME}, ME ignores the original termination functions of subtasks, whereas SME strictly follows the termination functions of subtasks.
SME used in HRL can lead to suboptimality. We decompose this suboptimality into \emph{execution suboptimality} and \emph{policy suboptimality}, as illustrated by the \emph{GridWorld} shown in~\cref{fig:Gridworld-1}:
\begin{figure*}[t]
\centering
\includegraphics[width=0.9\textwidth]{
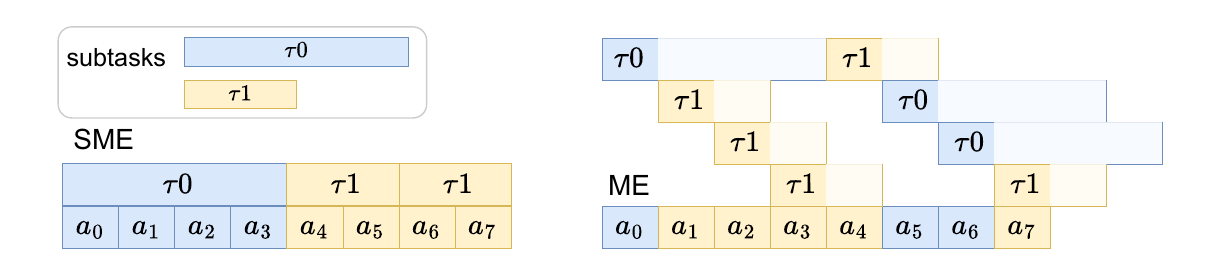
}
\caption{Interaction differences between SME and ME: calling subtasks in full vs. calling subtasks in one timestep. Task $\tau$ has a 4-timestep subtask $\tau0$ and a 2-timestep subtask $\tau1$. Both $\tau0$ and $\tau1$ directly call actions.
The subtask $\tau0$ and the actions it decides are represented in blue, while $\tau1$ and its actions are represented in yellow.
}
\label{fig:ME_and_SME}
\end{figure*}
\begin{figure*}[tbp]
    \centering
        \begin{subfigure}[b]{0.18\textwidth}
            \includegraphics[width=\textwidth]{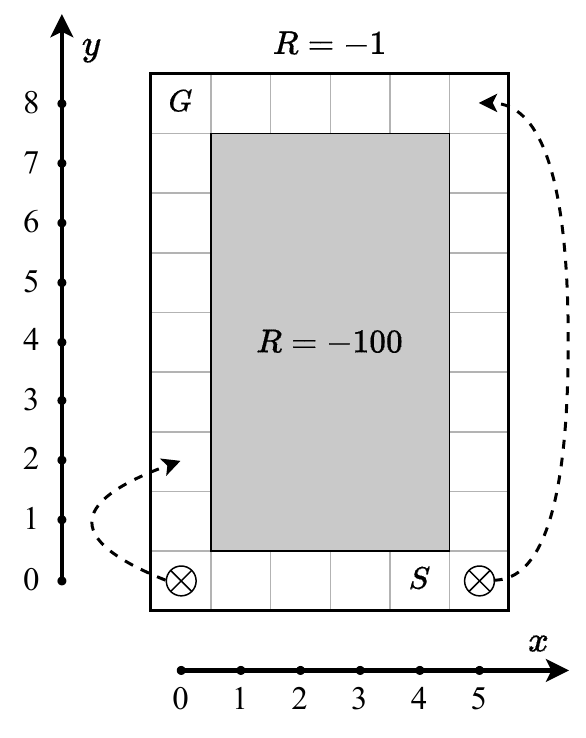}
            \caption{The GridWorld.}
            \label{fig:Gridworld-1}
        \end{subfigure}
        \begin{subfigure}[b]{0.35\textwidth}
            \includegraphics[width=\textwidth]{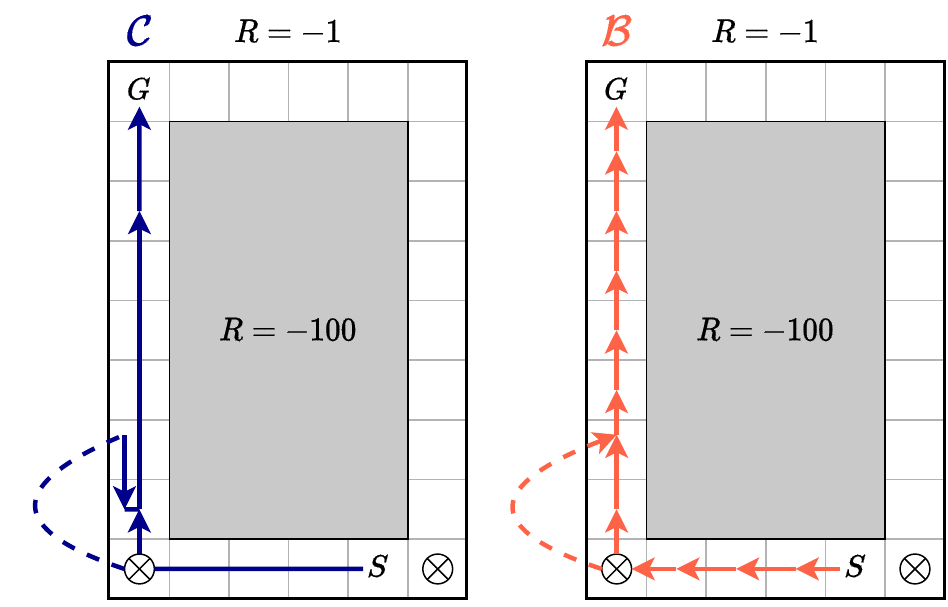}
            \caption{Execution results of SME and ME.}
            \label{fig:Gridworld-2}
        \end{subfigure}
        \begin{subfigure}[b]{0.38\textwidth}
            \includegraphics[width=\textwidth]{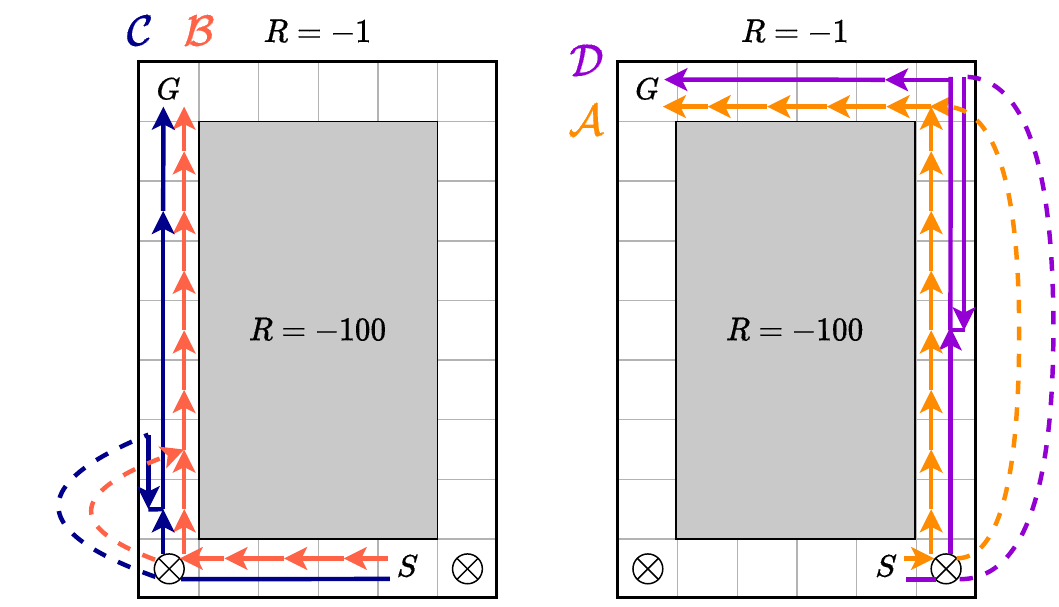}
            \caption{Policy differences between $\hat\pi_{hi}^-$ and $\hat\pi_{hi}^*$.}
            \label{fig:Gridworld-3}
        \end{subfigure}
    \caption{Trajectories of the agent in the \emph{GridWorld} under different policies and execution modes. The agent needs to reach $G$ at $(0,8)$ from $S$ at $(4,0)$ in the minimum number of frames, receiving a reward of $-1$ per frame. After reaching the position marked $\otimes$, there is a 50\% chance of an additional state transition along the dashed line.
    \textcolor{black}{The trajectories of the SMDP-optimal policy $\hat\pi_{hi}^-$ under SME and ME are shown in blue and red, with their expected returns denoted by $J_{\mathcal{C}}$ and $J_{\mathcal{B}}$ respectively. The trajectories of the MDP-optimal policy $\hat\pi_{hi}^*$ under SME and ME are shown in purple and orange, with their expected returns denoted by $J_{\mathcal{D}}$ and $J_{\mathcal{A}}$ respectively.}
    The detailed descriptions of the policies and execution modes, as well as the related numerical computations, are provided in the~\cref{sec:A}}
    \label{fig:Gridworld}
\end{figure*}
\subsection{Execution suboptimality}
\label{sec:3.1}
\cref{fig:Gridworld-2} shows the trajectories of an agent with a two-level policy set $\pi$ under SME and ME.
Given the termination functions and node policies of all tasks within the agent, \textbf{\textit{using different execution modes results in different trajectories}}. That is, for any $\pi$, there exist the expected return under SME $J^{\text{SME}} (\pi | \tau)$, and under ME $J^{\text{ME}} (\pi | \tau)$.
To reach $(0,8)$ as quickly as possible, \texttt{up} is the optimal action at state $(0,2)$. For the policy set in~\cref{fig:Gridworld-2}, it is clear that ME achieves a higher expected return than SME: $J^{\text{ME}}(\pi | G) = -11 > -12 = J^{\text{SME}}(\pi | G)$.
During the execution of subtask to reach $(0,1)$, the subtask with the highest expected return changed (due to a random event), but SME could not promptly stop the unfinished suboptimal subtask, resulting in suboptimality. We refer to this suboptimality during the deployment phase as \textbf{execution suboptimality}.
\subsection{Policy suboptimality}
\cref{fig:Gridworld-3} shows the trajectories of the optimal policy under the two different execution modes, SME and ME:
$J^{\text{ME}}(\pi^* | G) = -8 > -11 = J^{\text{ME}}(\pi^- | G), J^{\text{SME}}(\pi^- | G) = -12 > -14 = J^{\text{SME}}(\pi^* | G)$.
Given the termination functions of all subtasks within an agent, \textbf{\textit{different execution modes can lead to different optimal node policies.}}This paper refers to the difference in policy set caused by the execution mode during the training phase as \textbf{policy suboptimality}.
We provide the following definitions of policy optimality under these two specific execution modes:
\begin{definition}[\textbf{MDP-optimal}]\label{def:MDP_opt}
For a task $\tau$, given the node policies of all its descendant tasks, the node policy $\hat\pi_{\tau}$ that maximizes $V^\text{ME}_\pi(\cdot|\tau)$ is called the \emph{MDP optimal node policy} of task $\tau$, denoted as $\hat\pi^*_{\tau}$. That is,
$
\hat\pi^*_{\tau} = \argmax_{\hat\pi_\tau}V^\text{ME}_\pi(\cdot|\tau)
$.
When all node policies in $\pi$ are MDP optimal node policies, we call $\pi$ an \emph{MDP-optimal policy set}, denoted as $\pi^*$.
\end{definition}
\begin{definition}[\textbf{SMDP-optimal}]\label{def:SMDP_opt}
For a task $\tau$, given the node policies and termination functions of all its descendant tasks, the policy $\hat\pi_{\tau}$ that maximizes $V^\text{SME}_\pi(\cdot|\tau)$ is called the \emph{SMDP optimal node policy} of task $\tau$, denoted as $\hat\pi^{-}_{\tau}$. That is,
$
\hat\pi^-_{\tau} = \argmax_{\hat\pi_\tau}V^\text{SME}_\pi(\cdot|\tau)
$.
When all node policies in $\pi$ are SMDP optimal node policies, we call $\pi$ an \emph{SMDP-optimal policy set}, denoted as $\pi^{-}$.
\end{definition}
\section{Additional theory overview}
\label{sec:4}
\subsection{Limitations of SME}
\label{sec:4.1}
We now formally present the narrow SMDP suboptimality theory under the two-level HRL setting, which provides theoretical guarantees for the limitations of SME.
\begin{table}[t]
\centering
\caption{Optimality comparison table.}
\begin{tabular}{lcc}
\toprule
Expected return & SME & ME \\
\midrule
SMDP-optimal policies & $J_{\mathcal C}$ & $J_{\mathcal B}$ \\
MDP-optimal policies & $J_{\mathcal D}$ & $J_{\mathcal A}$ \\
\bottomrule
\end{tabular}
\label{tab:SubOpt}
\end{table}
\begin{theorem}[\textbf{Narrow SMDP Suboptimality Theorem}]
\label{def:4.1}
In a two-level HRL problem where the environment satisfies the Markov property, the four expected returns generated by the MDP-optimal policy set $\pi^*$ and the SMDP-optimal policy set $\pi^-$ under ME and SME are defined in~\cref{tab:SubOpt}. Then the following ordering holds:
\begin{equation}
J_{\mathcal A} \ge J_{\mathcal B} \ge J_{\mathcal C} \ge J_{\mathcal D}. \nonumber
\end{equation}
\end{theorem}
This paper also proposes a \textbf{Generalized SMDP Suboptimality Theorem} under more-than-two-level and arbitrary execution modes, which constitutes the core theoretical contribution of this work. We provide detailed \textcolor{black}{formulations, } descriptions and proofs in the~\cref{sec:B}, where the narrow SMDP suboptimality theory is also proven as a special case.
\begin{proof}[Proof Sketch \textcolor{black}{for \cref{def:4.1}}.]
\textbf{Conceptual Clarification:} This paper notes that each task $\tau$ has two termination functions with different roles. One controls the optimization objective of $\hat\pi_\tau$, namely the \textbf{termination function} $\beta_\tau$ introduced in~\cref{sec:2.2}. The other controls the duration of $\tau$'s transitions when it is called as a subtask during interaction, defined as the \textbf{interrupt function} $\hat\beta_\tau$. In this context, we observe that when using ME, $\hat\beta_\tau \equiv 1$, while when using SME, $\hat\beta_\tau \equiv \beta_\tau$.
In our proof, we further introduce the \textbf{task tree} and the \textbf{execution tree} to decouple the binding of $\hat\beta_\tau$ and $\beta_\tau$ in existing HRL, which respectively describe the complex task optimization objectives and task execution modes across multiple hierarchical levels.
In simple terms, for each $\tau$, the task tree corresponds to the task decomposition, which determines the termination conditions of $\tau$ as well as the boundary cases when returning to different upper-level tasks. The execution tree, on the other hand, acts as the policy scheduler, specifying how $\tau$ calls its lower-level descendant tasks. By defining these two types of trees, we can precisely formalize the complex situations in more-than-two-level hierarchical settings where the termination or interruption of any upper-level task of $\tau$ will cause $\tau$ itself to be terminated or interrupted.
\textbf{Unified Formalization:} The value computation under different execution modes cannot be directly described by the Bellman equation, since action selection depends on both the current frame's state and the task. At each frame during interaction, the hierarchical level of a task is jointly influenced by $\hat\pi$ and $\hat\beta$. For instance, in a two-level HRL, under ME, decisions are made from $\hat\pi_{hi}$ downward at every frame, whereas under SME, downward decisions from $\hat\pi_{hi}$ occur only when a subtask ends, and the remaining frames are decided by $\hat\pi_{lo}$.
To address this, we describe the generalized interaction process of HRL using four stages: \textbf{Preparation stage:} Determine which task makes downward decisions based on $\hat\beta$. \textbf{Decision stage:} The task determined in the previous stage makes downward decisions according to $\hat\pi$ until an action $a$ is selected. \textbf{Interaction stage:} Each task receives rewards according to the agent's transitions. \textbf{Termination stage:} Tasks in execution determine termination based on $\beta$, including boundary conditions for termination.
Based on four stages, we propose a \textbf{Unified Value Function for HRL(UVFH)}. Using UVFH, we establish the \textbf{Generalized Hierarchical Bellman Equation (GHBE)} and the \textbf{Generalized Hierarchical Bellman Optimality Equation (GHBOE)}, which describe the dynamic value transitions across the four stages and serve as the foundation for subsequent theoretical analysis. Here "generalized" refers to arbitrary execution modes and task hierarchy configurations.
\textbf{Theoretical Proofs:} We prove that the iterative operators of the two equations above are Contraction Mappings, thereby providing theoretical guarantees for the convergence of their solutions. On this basis, we further prove the \textbf{Generalized Hierarchical Policy Improvement Theorem}, which ensures that under the ME and SME execution settings, the optimal node policies induced by the GHBOE correspond to the MDP optimal node policies and the SMDP optimal node policies, respectively. Consequently, we establish that
$J_{\mathcal A} \geq J_{\mathcal B}$ and $J_{\mathcal C} \geq J_{\mathcal D}.
$
Next, we prove the \textbf{Hierarchical Execution Improvement Theorem} for execution modes controlled by $\hat\beta$. Using this theorem, together with certain subtask conditions (satisfied by two-level HRL), we prove the \textbf{Optimal Execution Theorem}. This theorem guarantees that, under some conditions, the expected return of any execution mode is no greater than that under ME, i.e.,
$J_{\mathcal B} \geq J_{\mathcal C}.$
In particular, we also provide counterexamples for multi-level HRL settings where the conditions are not satisfied for descendant task, demonstrating that the optimality of ME is not evident.
\end{proof}
\textcolor{black}{
In the process of proving~\cref{def:4.1}, we derived the following key theorems.
The relevant concepts and their precise formal definitions can be found in~\cref{sec:B}.
}
\textcolor{black}{
\begin{theorem}[\textbf{Hierarchical Policy Improvement Theorem}]
\label{def:4.1.1}
Given execution trees $e$ and $e'$, suppose that the only difference between them lies in the node policy functions $\hat{\pi}_{n(h)}$ and $\hat{\pi}'_{n(h)}$ at state $s$.
If
$
\mathbb{E}_{ \tau_{n(h+1)} \sim \hat{\pi}_{n(h)}(s) }
V_{e|n(h')} (s, \tau_{n(h+1)})
\leq
\mathbb{E}_{ \tau_{n(h+1)} \sim \hat{\pi}'_{n(h)}(s) }
V_{e|n(h')} (s, \tau_{n(h+1)}),
$
where $h \geq h'$,
then for all $s$, it holds that
$
V_{e|n(h')} (s, \tau_{n(h)})
\leq
V_{e'|n(h')} (s, \tau_{n(h)})
$
.
\end{theorem}
}
\textcolor{black}{
\begin{theorem}[\textbf{Hierarchical Execution Improvement Theorem}]
\label{def:4.1.2}
Given execution trees $e$ and $e'$, suppose that only the decision interruption functions $\hat\beta_{n(h)}$ and $\hat\beta'_{n(h)}$ differ at $s'$. Let decision transition mechanism $\hat\gamma$ and $\hat\gamma'$ be computed from $\hat\beta$ and $\hat\beta'$, respectively. If for all $n(l)$ it holds that
$
\sum^l_{k = h'} \hat\gamma_{n(k)} (s') V_{e|n(h')} (s', \tau_{n(k)}| \tau_{n(h')})
\leq
\sum^l_{k = h'} \hat\gamma'_{n(k)} (s') V_{e|n(h')} (s', \tau_{n(k)}| \tau_{n(h')}),
$
then for all $s$ we have
$
V_{e|n(h)} (s, \tau_{n(h)}) \leq V_{e'|n(h)} (s, \tau_{n(h)}),
$
where $l \geq h \geq h'$, and $l$ denotes the level of the task directly calling an action in the previous frame.
\end{theorem}
}
\textcolor{black}{
\begin{theorem}[\textbf{Optimal Execution Theorem}]
\label{def:4.1.3}
Under~\cref{def:B.16}, if the node policy $\hat\pi_{n(h)}$ is the optimal policy for the execution induced by any interruption subtree $\mathbf{B}|n(h)$, then
$
V^{\mathbf{B}|n(h)}_{\Pi|n(h)} \leq V^{\text{ME}}_{\Pi|n(h)},
$
i.e., ME is the optimum.
\end{theorem}
}
Through the SMDP suboptimality theory, we reveal the expected return suboptimality of SME during both the training and deployment phases. \textcolor{black}{
To avoid policy suboptimality, the target policy should use ME to approaches $J_{\mathcal{A}}$ instead of being limited by $J_{\mathcal{B}}$. To avoid execution suboptimality, the deployment policy shoud use ME to approaches $J_{\mathcal{B}}$(or $J_{\mathcal{A}}$) instead of being limited by $J_{\mathcal{C}}$(or $J_{\mathcal{D}}$).
}
\subsection{Limitations of ME}
Unlike SME, under ME, $\hat\beta$ and $\beta$ are set differently, which may result in frequent interruptions of a task by higher-level tasks during its execution. During task optimization, deviations between the target policy and the behavior policy may occur, as indicated by the red marks in~\cref{fig:ME_offline}. This makes it difficult for the boundary conditions controlled by $\beta$ at task termination to be properly back-propagated along the interaction trajectory during dynamic programming.
\begin{figure*}[t]
\centering
\includegraphics[width=0.95\textwidth]{
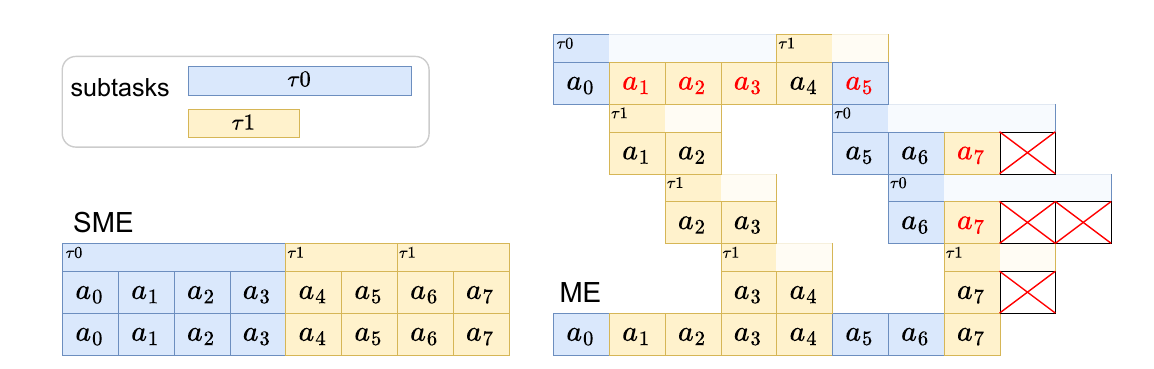
}
\caption{Subtask off-policy issue under ME. Following the example in Figure~1, the bottom row of the action sequence represents the agent's interaction trajectory with the environment,
where blue and yellow colors are used to distinguish whether an action is decided by the node policy of subtask $\tau 0$ or $\tau 1$ respectively.
Within each subtask box, the action sequence corresponds to the truncated trajectory used for training that subtask.
Whenever the truncated trajectory contains data that is not decided by the node policy of the corresponding subtask,
or when trajectory data is missing, these actions are highlighted in red, indicating the occurrence of an off-policy issue.
}
\label{fig:ME_offline}
\end{figure*}
Therefore, if the agent directly interacts with the environment using ME during training, the node policies of all subtasks will face severe off-policy issues, making efficient training difficult. \textcolor{black}{To ensure stable sampling and learnability of lower-level subtasks during training, the behavior policy should still use SME.}
\subsection{The exploration--exploitation tradeoff induced by execution modes}
\textcolor{black}{Given that we have already adopted a hierarchical architecture, what extra suboptimality is introduced purely by the SME, and how can we remove it by using a more suitable execution mode?}
Based on the theoretical analysis above, the problem can be reduced to an exploration-exploitation tradeoff caused by the execution mode in HRL.
\textcolor{black}{
Overall, the common HRL design in which both target and behavior policies use SME is exactly the fundamental design flaw.}
\textbf{During the training phase:} Directly using ME to learn the MDP-optimal policy set results in low subtask sampling efficiency, leading to an exploration problem; directly using SME to learn the SMDP-optimal policy set results in policy suboptimality, leading to an exploitation problem.
\textbf{During the deployment phase:} The optimal execution mode, ME, should be used to ensure exploitation.
Therefore, a reasonable HRL algorithm should balance exploration and exploitation, interacting with the environment via SME to generate data while learning the MDP-optimal policy set in some way.
\section{Framework implementations}
Based on~\cref{sec:4}, this paper designs three theoretically guaranteed frameworks for improving HRL algorithms
by separating the task termination function from the interruption functions and independently specifying the execution modes of the \textbf{behavior policy}, \textbf{target policy}, and \textbf{deployment policy}:
the \emph{Execution Suboptimality Improvement Framework (ESIF)}, the \emph{Two-Stage Policy Suboptimality Improvement Framework (PSIF-2S)}, and the \emph{One-Stage Policy Suboptimality Improvement Framework (PSIF-1S)}. Existing HRL algorithms can be improved through these frameworks to balance the exploration-exploitation tradeoff and mitigate the SMDP suboptimality of the original algorithms to varying degrees.
\subsection{Frameworks}
\textbf{ESIF:} This framework does not modify the training process of the algorithm and directly uses the original algorithm to learn $\pi^-$ (i.e., by sampling data through SME to solve the SMDP-optimal policy set). During the deployment phase, using ME can achieve the expected return $J_{\mathcal B}$.
This framework is analogous to using $\epsilon$-greedy SARSA~\citep{1994On} to train the agent in RL, followed by greedy execution during deployment.
\textbf{PSIF-2S:} Each task node in the task hierarchy has two training stages.
\textbf{Stage 1:} A task node calls its children using SME and trains its node policy for the corresponding SME objective.
\textbf{Stage 2:} Once every child has completed its own Stage 2 and its policy is fixed, the parent calls its children using ME and trains its MDP node policy. The parent policy is fixed only after this second-stage optimization has approximately converged.
The lowest trainable task nodes, whose children are primitive actions, start in Stage 2 because their SME and ME objectives coincide; all other task nodes start in Stage 1. During training, each unfinished node solves the GHBOE for its current node-local execution mode. Stage transitions propagate from the bottom upwards, and nodes at different depths may be in different stages simultaneously. Training ends when every task node has completed Stage 2. ME deployment then corresponds to $J_{\mathcal A}$ in the idealized comparison. Here stage numbers describe a node's own training and its calls to children, consistently with the general framework in Section~\ref{new:sec:frameworks}.
\textbf{PSIF-1S:} Each task node is called using SME during data collection. When a particular node policy is optimized, its target calls descendants within its subtree using ME, while execution outside that subtree remains SME. GHBE supplies its single-frame bootstrap and GHBOE its MDP node-policy objective $\hat\pi_\tau^*$. The same construction applies at every depth. During deployment, the whole policy tree uses ME, corresponding to $J_{\mathcal A}$ in the idealized comparison.
This is analogous to using Q-Learning~\citep{watkins1992q} to train the agent in RL, followed by greedy execution during deployment.
We provide the pseudocode and improvement instances for the three frameworks in~\cref{sec:C}.
\subsection{Algorithm improvement instances}
The OC algorithm is a representative algorithm of the Options Framework~\citep{stolle2002learning,riemer2018learning,brunskill2014pac,harutyunyan2019termination}, suitable for verifying theoretical correctness.
\textbf{ESIF framework:} During deployment, options can be selected greedily at each frame according to the higher-level value function.
\textbf{PSIF-2S framework:}
During the root's Stage 1, options are called using SME to collect trajectory data; the lowest task nodes learn their option policies, for which primitive-action SME and ME coincide.
The value updates of both high-level and low-level node policies follow the original algorithm updated by:
\textcolor{black}{
\[
\begin{aligned}
V(s_t,\omega_t),\,V(s_t,\omega_t,a_t)\gets{}&R(s_t,a_t)+\gamma\big(\beta_{\omega_t}(s_{t+1}) \max_{\omega} V(s_{t+1},\omega) \\
&+ (1-\beta_{\omega_t}(s_{t+1})) V (s_{t+1},\omega_t)\big)
\end{aligned}
\]
}
where $\omega$ denotes an option subtask,
and the left arrow indicates that the left-hand side is trained under its input distribution sampling,
using the right-hand side as the supervision target via Mean Squared Error.
$V(s,\omega)$ is used to induce the greedy high-level node policy, to provide gradients for the termination function of the low-level policy, and can also serve as a baseline for low-level policy gradient updates. The original Option-Critic paper also used an intra-option policy-gradient baseline in its Atari experiments~\citep{AAAI17OC}.
During the root's Stage 2, it calls the completed options using ME. The option policies and termination functions are fixed, and the upper-level value update follows the GHBOE under ME:
\[
V(s_t,\omega_t) \leftarrow
R(s_t,a_t) + \gamma \max_{\omega} V(s_{t+1},\omega)
\]
\textbf{PSIF-1S framework:} The agent interacts with SME, but the higher-level values are updated off-policy according to the GHBOE under ME:
\[
V^{hi}(s_t,\omega_t) \leftarrow R(s_t,a_t)+\gamma \max_{\omega} V^{hi}(s_{t+1},\omega)
\]
Primitive-action SME and ME coincide within an option, but the option-exit boundary must still be specified. The PSIF-1S instance studied here uses the upper ME value on option exit, yielding the lower updates: \textcolor{black}{
\[
\begin{aligned}
V^{lo}(s_t,\omega_t),\,V^{lo}(s_t,\omega_t,a_t)\gets{}&R(s_t,a_t)+\gamma \big(
\big(1-\beta_{\omega_t}(s_{t+1})\big)V^{lo}(s_{t+1},\omega_t)\\
&+\beta_{\omega_t}(s_{t+1})\max_{\omega} V^{hi}(s_{t+1},\omega)\big).
\end{aligned}
\]
}
$V^{hi}(s,\omega)$ induces the higher-level greedy node policy; $V^{lo}(s,\omega)$ supplies the lower-level policy and termination gradients. The full updates and the separate termination baseline are given in Equations~\eqref{new:eq:oc_offtrace}--\eqref{new:eq:oc_policy}, with the Fourrooms termination-update variant described in Appendix~\ref{new:app:oc_updates}.
\section{Additional experimental analysis}
This section analyzes ESIF in OC and HAC and compares the two PSIF training procedures in OC. The learning curves, percentage improvements, and uncertainty summaries use the experimental settings described in Appendix~\ref{sec:D}. Section~\ref{sec:stochastic-experiments} studies the effects of controlled transition kernels.
\begin{figure*}[tbp]
	\centering
	\begin{subfigure}[b]{\textwidth}
			\includegraphics[width=0.49\textwidth]{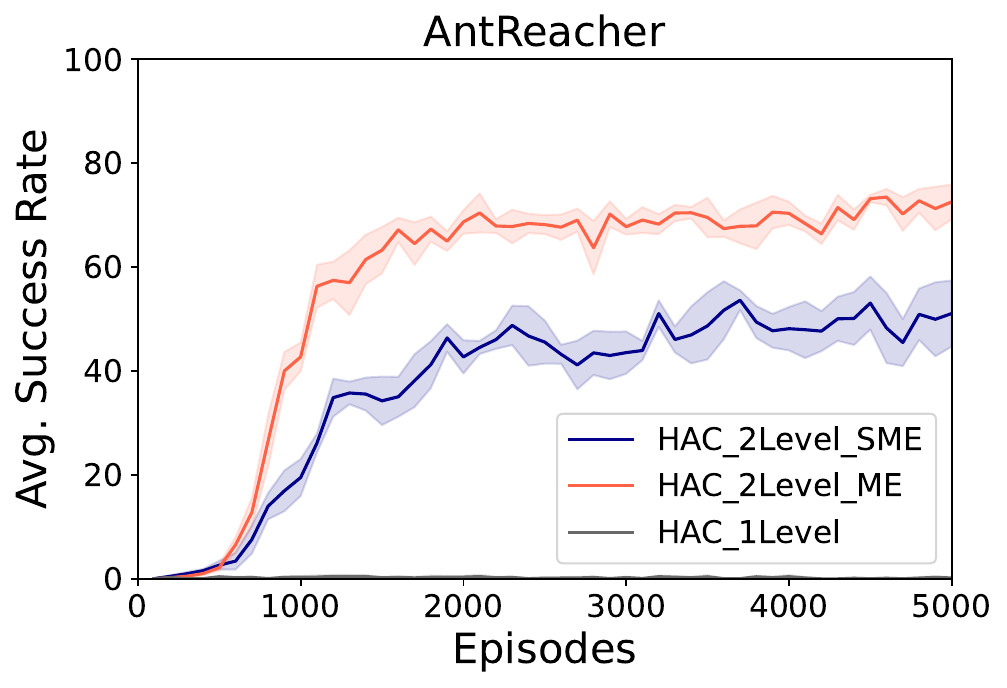}
			\includegraphics[width=0.49\textwidth]{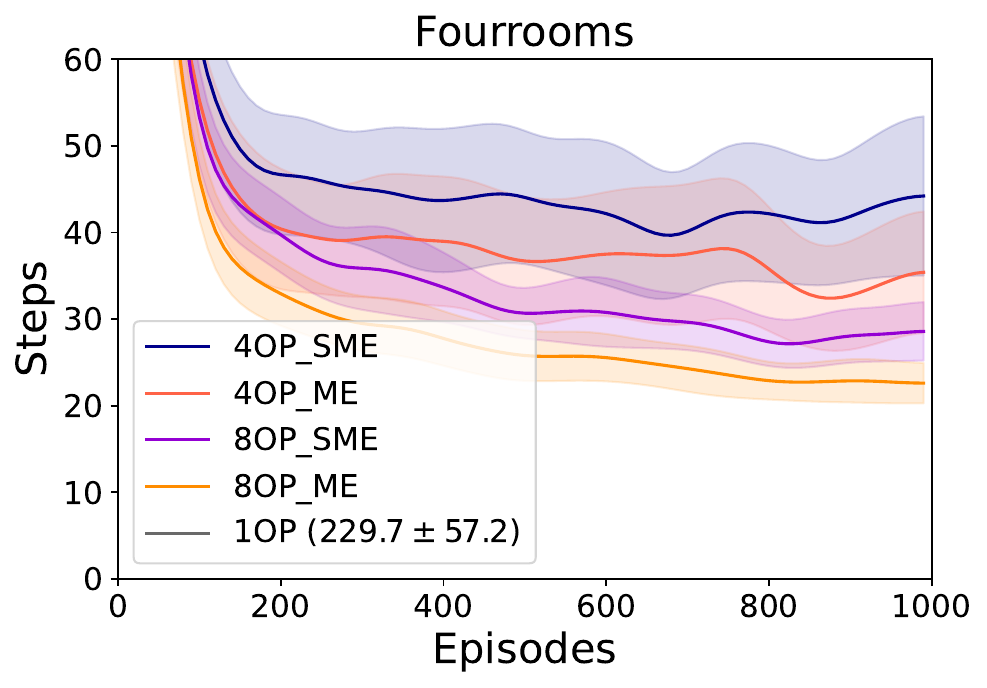}
			\caption{ESIF for HAC and ESIF for OC.}
			\label{fig:Experiments-1}
	\end{subfigure}
	\par\medskip
	\begin{subfigure}[b]{\textwidth}
			\includegraphics[width=0.49\textwidth]{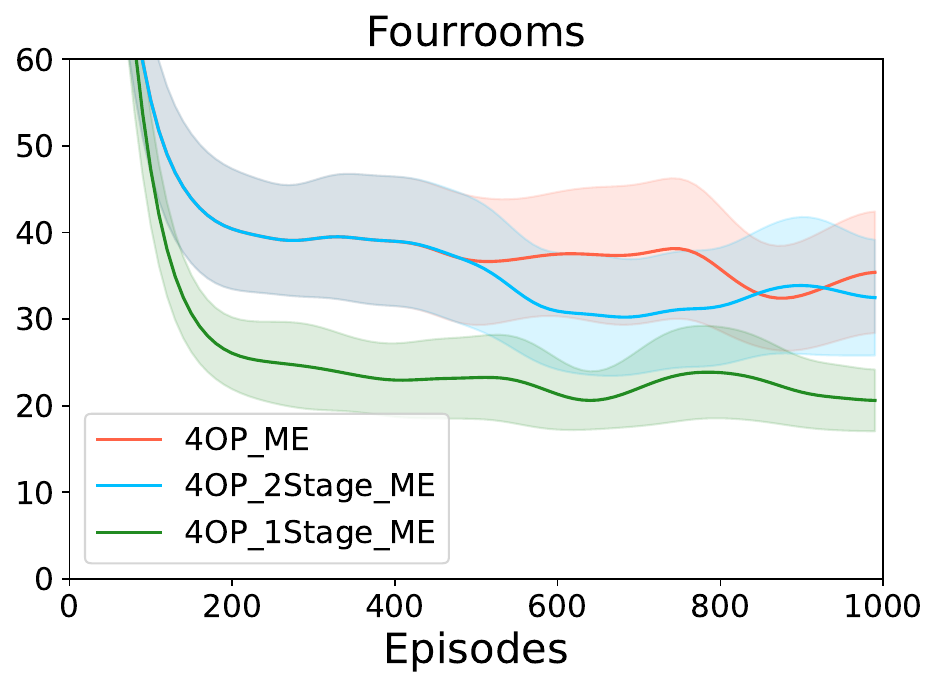}
			\includegraphics[width=0.49\textwidth]{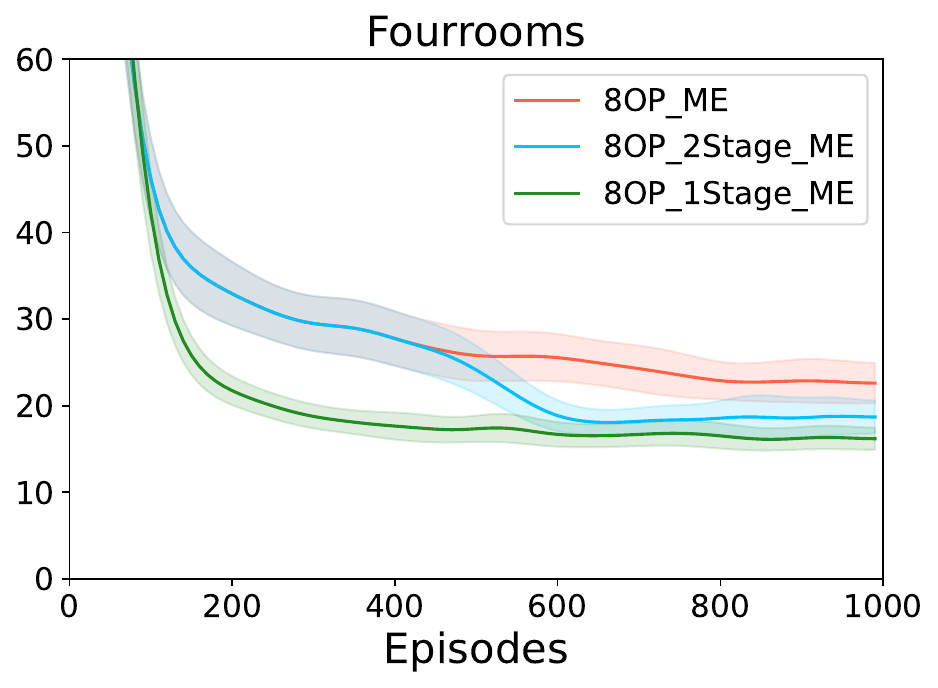}
			\caption{PSIF-2Stage and PSIF-1Stage for OC.}
			\label{fig:Experiments-2}
	\end{subfigure}
	\caption{
    Results of the improvement frameworks. (a) Verification of execution suboptimality.
    (b) Verification of policy suboptimality. The shaded regions represent one standard deviation over results obtained with multiple random seeds. Higher success rates and fewer steps are better.
	}
	\label{fig:Experiments}
\end{figure*}
In the grid-world environment \textit{Fourrooms}, we employ the option-based Option Critic (OC) algorithm \citep{AAAI17OC} as a representative of option frameworks with solid theoretical guarantees, to empirically validate the suboptimality theory of SMDPs.
In the robotic navigation environment \textit{AntReacher}, we adopt the HAC algorithm \citep{ICLR19HAC} as a representative of mainstream goal-based frameworks, verifying the execution suboptimality problem of the feudal framework and achieving performance improvements.
Details of the environments, additional experiments, and visualizations are provided in~\cref{sec:D}.
\textbf{Improvement of ESIF and Verification of Execution Suboptimality.}
In the right side of \cref{fig:Experiments-1}, the execution suboptimality of the option-based baseline algorithm OC was validated on the \textit{Fourrooms} problem, while on the left side of \cref{fig:Experiments-1}, the execution suboptimality of the goal-based baseline algorithm HAC was validated on the stochastic version of \textit{AntReacher}, where this suboptimality led to a significant performance gap of over 20\%.
This experiment verifies both $J_{\mathcal B} \geq J_{\mathcal C}$ and the effectiveness of ESIF.
\textbf{Ablation Study.}
To further illustrate the difficulty of these tasks, we also conducted single-level experimental validation using corresponding flat RL algorithms , which failed to be effectively trained on these long-horizon tasks.
Therefore, it is infeasible to directly obtain the MDP-optimal policy $\pi^*$ using standard MDP-based RL algorithms.
\footnote{\textcolor{black}{For the single-layer RL baseline ablation experiments that were truncated due to zero success rate and excessively long episode lengths, we display them only as numerical values in the legend so as not to affect the clarity of the main curves.}}
\textbf{Improvement of PSIF and Verification of Policy Suboptimality.}
\cref{fig:Experiments-2} shows the comparison of the OC algorithm under different settings of 4 and 8 options, where the OC baseline, OC-PSIF-2S, and OC-PSIF-1S are deployed using ME.
In OC-PSIF-2S, option learning is fixed after 500 episodes and the root enters its second stage.
This experiment validates the policy suboptimality of the OC baseline in the \textit{Fourrooms} problem: the MDP-optimal policies $\pi^*$ obtained via the two improvement frameworks achieve higher expected returns $J_{\mathcal A}$ under ME deployment compared with the SMDP-optimal policies $\pi^-$ and their expected returns $J_{\mathcal B}$ under ME.
For 8 options, OC-PSIF-2S completes the task in over 30\% fewer steps than the OC baseline, and for 4 options, the reduction exceeds 50\%, indicating that the OC baseline suffers from stronger SMDP suboptimality when fewer options.
\textbf{Analysis of Differences Between the Two PSIF Variants.}
We observe that PSIF-1S achieves better final performance than PSIF-2S upon deployment. This is because in PSIF-2S, the low-level node policies are fixed during the second-stage training. However, when $\hat\pi_{hi}$ updates from $\hat\pi^-_{hi}$ to $\hat\pi^*_{hi}$ upon entering the second stage, the state distribution of the agent $s$ changes, and the frequency of subtask $g$ called also changes. This causes a more severe $(s,g)$ distribution shift for the low-level policy $\hat\pi_{lo}(s,g)$, degrading the performance of the learned low-level policies and limiting the achievable final performance of the high-level tasks.
This issue does not occur in the PSIF-1S framework.
We define the set of already-trained subtasks among the subtasks of $\tau$ as the \emph{subtask support} of $\tau$. Since employing different policies for $\tau$ results in different distributions over its subtasks during training, the obtained subtask support also differs accordingly.
Comparing the performance curves of PSIF-1S and PSIF-2S in \cref{fig:Experiments-2}, the expected return $J_{\mathcal A}$ obtained with $\hat\pi^*_{hi}$ subtask support is higher than that obtained with $\hat\pi^-_{hi}$ subtask support.
The advantages of PSIF-1S over PSIF-2S are:
1) It does not require manual determination of whether a child node policy has reached $\pi^*_\tau$ and can be fixed so that its parent enters the second stage.
2) PSIF-1S learns the subtask support of $\pi^*_\tau$ without being constrained by the subtask support of $\pi^-_\tau$.
\textbf{Limitations.}
ESIF can only address execution suboptimality and is ineffective for policy suboptimality.
In PSIF-2S, during a parent's second-stage training, its fixed child policies can suffer from state-task distribution shift.
PSIF-1S, on the other hand, faces an off-policy problem during training due to the mismatch between the target execution mode (using ME) and the behavior execution mode (using SME).
Our framework decouples execution modes and training objectives in HRL; efficiently computing these objectives for different algorithms remains an open problem.
For instance, in the goal-based PSIF-2S method, the challenge lies in subtask dependencies: during training, subtask frequencies are concentrated on those commonly called by the SME-optimal node policy, leading to insufficient learning of subtasks required by the ME-optimal node policy.
PSIF-1S needs to address the off-policy problem of SME data.
Extending SMDP suboptimality improvement frameworks to policy-gradient-based HRL algorithms is a promising direction for future research.
\section{Extended related work}
\label{related}
\textbf{Optimality in SMDP.}
\citet{NIPS20EfficiencyHRL} provide optimality guarantees for planning in HRL, but the number of subtasks required to guarantee optimality grows exponentially and becomes impractical.~\citet{ICLR19RLHRL} studies the impact of subtask space representation on suboptimality, which represents another type of suboptimality in HRL.~\citet{AI99MDPsSMDPs} also reveals the existence of suboptimality in SMDP under the option framework, but only identifies \emph{execution suboptimality} while overlooking \emph{policy suboptimality}.
In contrast, this paper further proves the \emph{Optimal Execution Theorem}, which formally establishes the value-optimal execution mode and provides theoretical guarantees for general hierarchical policy improvement.
We also provide a generalized multi-level proof and counterexamples for SMDP suboptimality, reducing the problem to the exploration--exploitation tradeoff in HRL. Recent work~\citet{drappo2023option} analyzed regret bounds in finite-horizon SMDPs, showing how option duration shapes performance. Meanwhile, HiT-MDP~\citep{li2023hit} reformulated option-induced processes as equivalent MDPs to stabilize optimization. These studies reveal the tradeoffs of temporal abstraction
but not take into account that SME and ME should differ in the behavior policy, target policy, and deployment policy.
\textbf{Exploration in SMDP.}
Potential algorithms in HRL that can be adapted to our improvement frameworks
cover a wide range of research directions, including autonomously creating option policies from data~\citep{harutyunyan2019termination,riemer2020role,NIPS21MOC,zhang2021hierarchical,lin2024hierarchical,nayyar2025autonomous}, skill discovery~\citep{menache2002q,bagaria2020option,ICLR19DIAYN,tang2018subgoal}, multi-task or multi-agent learning~\citep{igl2020multitask,liu2021coach,yang2019hierarchical,tang2018hierarchical}, and goal-space learning~\citep{andre2002state,ghosh2018learning,zhang2020generating}.
These algorithms consistently demonstrate that HRL can significantly improve efficiency in long-horizon decision-making problems, and \citet{NIPS20EfficiencyHRL,robert2023sample} also provide theoretical guarantees on the sample efficiency of goal-based HRL.
However, most of these algorithms are based on the SMDP framework and therefore inherently suffer from the SMDP suboptimality issues.
\textbf{Value formalization in HRL.}
\citet{ICML15UVFA} proposed the concept of unified value function, and \citet{AAAI17OC} share similar ideas, but these approaches cannot represent multi-level policy values under arbitrary execution modes.

\section{Notations}
\label{notations}
Table~\ref{tab:legacy-notations} lists the symbols used in the detailed formalization and proofs.
\begin{table}[h]
\centering
\caption{Notations used in this paper.}
\label{tab:legacy-notations}
\begin{tabular}{ll}
\hline
\textbf{Symbol} & \textbf{Description} \\
\hline
$\mathcal{S}$ & State space \\
$\mathcal{A}$ & Action space \\
$T : \mathcal{S} \times \mathcal{A} \to \Delta(\mathcal{S})$ & Transition kernel of MDP \\
$R_\tau : \mathcal{S} \times \mathcal{A} \to \mathbb{R}$ & Per-frame reward function of task $\tau$ \\
$\hat{R}_\tau : \mathcal{S} \times \mathcal{T}'_\tau \to \mathbb{R}$ & Exit reward function of task $\tau$ \\
$\alpha_\tau \in \Delta(\mathcal{S})$ & Initial state distribution of task $\tau$ \\
$\beta_\tau : \mathcal{S} \to [0,1]$ & Termination function of task $\tau$ (in execution tree)\\
$\mathcal{T}'_\tau$ & Subtask set of task $\tau$ \\
$\bar{T}_\tau : \mathcal{S} \times \mathcal{T}'_\tau \to \Delta(\mathcal{S})$ & Subtask transition function \\
$\bar{R}_\tau : \mathcal{S} \times \mathcal{T}'_\tau \to \mathbb{R}$ & Subtask reward function \\
$t$ & Frames number \\
$\hat t$ & Subtasks been called number \\
$\pi_\tau : S \to \Delta(\mathcal{A})$ & Policy of task $\tau$ (calls actions across full level) \\
$\hat{\pi}_\tau : S \to \Delta(\mathcal{T}'_\tau)$ & Node policy of task $\tau$ (calls subtasks) \\
$\pi$ & Policy set \\
$V_\pi(s|\tau)$ & Value function of policy set $\pi$ under task $\tau$ \\
$J(\pi|\tau)$ & Expected return of policy set $\pi$ under task $\tau$ \\
$\hat{\beta}_\tau : \mathcal{S} \to [0,1]$ & Interrupt function of task $\tau$ (in execution tree) \\
$\pi^*$ & MDP-optimal policy set (all $\hat{\pi}_\tau$ optimal under ME) \\
$\pi^-$ & SMDP-optimal policy set (all $\hat{\pi}_\tau$ optimal under SME) \\
$J_\mathcal{A}, J_\mathcal{B}, J_\mathcal{C}, J_\mathcal{D}$ & Expected returns under ME/SME with $\pi^*$ or $\pi^-$ \\
$\gamma, \hat{\gamma}$ & Discount factors for task termination / decision interruption \\
$N(i)$ & Set of node indices at depth $i$ in a tree \\
$n(i)$ & Node index at depth $i$ in a tree \\
$\cdot | n(i)$ & Subtree rooted at $\cdot_{n(i)}$ \\
$\tau_{n(i)}$ & Task node at index $n(i)$ \\
$e_{n(i)}$ & Execution node at index $n(i)$ \\
$\Pi$ & Policy tree in the execution tree \\
$\mathbf{B}$ & Interruption tree in the execution tree \\
$\hat s : (s, \tau_{n(i)})$ & Extended state \\
$V_e(\hat s|\tau) : V^{\mathbf{B}}_\Pi(\hat s|\tau)$ & Value function of execution tree $e$ under task $\tau$ (UVFH) \\
$\mathcal{F}^{e,\tau}_{n(h)}$ & The GHBE operator for executing $\tau_{n(h)}$ using $e|n(h)$ \\
\hline
\end{tabular}
\end{table}
\clearpage
\section{Naive case study: GridWorld}
\label{sec:A}
Figure~\ref{fig:GridWorld} gives the corridor geometry used in both numerical calculations below; Figure~\ref{fig:Gridworld} summarizes the resulting policy and execution comparisons.
\begin{figure}[t]
\centering
\includegraphics[width=0.6\textwidth]{
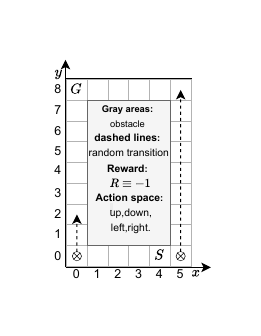
}
\caption{GridWorld.}
\label{fig:GridWorld}
\end{figure}
This is a $6\times 8$ GridWorld, where the central $4\times 6$ region is set as an obstacle. Under this configuration, the agent can only reach the goal through the left or right corridors. The action space of the agent is $\{\text{up},\text{down},\text{left},\text{right}\}$, which allows the agent to move one cell in the corresponding direction. When the agent reaches the $\otimes$ region on the map, it has a $50\%$ probability of undergoing an additional state transition along the dashed line. During interaction with the environment, the agent starts from the starting point $S$ and receives a reward of $-1$ for each transition until it reaches the goal $G$.
Within the HRL framework, the agent in this example completes sub-tasks by calling actions, and accomplishes the overall task of reaching $G$ by calling these sub-tasks.
In this example, the sub-task space $G'$ is defined as:
\[
G' = \{ g \in S \mid \text{sub-goal\footnote{A sub-goal is a special type of sub-task, in which the agent attempts to reach the location of the specified sub-goal.
} $g$ is a state whose $L_1$ distance from the current state is exactly $5$} \},
\]
where the $L_1$ distance is defined as
\[
L_1(s_1,s_2)=|x_1-x_2|+|y_1-y_2|.
\]
For instance, at the starting point $(4,0)$, the only two selectable sub-goals are $(0,1)$ and $(5,4)$. In fact, under this setting, there are always exactly two selectable sub-goals from any state, which greatly facilitates our clear analysis of the problem.
In each sub-task, the agent will learn a policy that reaches the sub-goal as quickly as possible: starting from the initiation of the sub-task, the agent receives a reward of $-1$ for each transition until reaching the sub-goal $g$. By maximizing the cumulative reward, the optimal policy for each sub-task can be obtained. Since the shortest paths to each sub-goal in this example are straightforward, we omit the detailed explanation here.
\subsection{Execution suboptimality}
\label{A_1}
\begin{figure}[t]
\centering
\includegraphics[width=0.95\textwidth]{
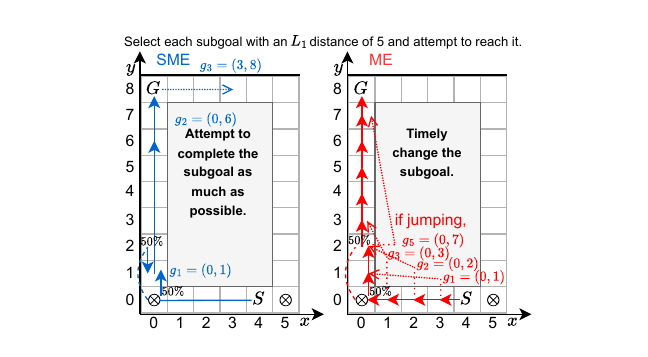
}
\caption{Execution Suboptimality.}
\label{fig:Execution_Suboptimality}
\end{figure}
The figure illustrates the trajectories of an agent under SME and ME execution modes with a two-level policy set $\pi$.
The high-level policy is denoted as $\hat\pi_{hi}(g \mid s, G)$, which selects an appropriate subgoal $g$ from the subtask space based on the current state $s$ in order to reach $G$ as quickly as possible.
The low-level policy is denoted as $\hat\pi_{lo}(a \mid s, g)$, which selects an appropriate action $a \in \{ \text{up}, \text{down}, \text{left}, \text{right} \}$ based on the current state $s$ to reach $g$ as quickly as possible.
The mapping of $\hat\pi_{hi}$ and its value estimation under ME are as~\cref{tab:SMDP_ME_value} (only listing the left-side path):
\begin{table}[t]
\centering
\caption{The mapping of $\hat\pi_{hi}$ and its value estimation under ME.}
\begin{tabular}{|l|c|c|c|c|c|c|}
\hline
$s$ & (4,0) & (3,0) & (2,0) & (1,0) & (0,0) & (0,1) \\
$g$ & (0,1) & (0,2) & (0,3) & (0,4) & (0,5) & (0,6) \\
$V(s)$ & -11 & -10 & -9 & -8 & -8 & -7 \\
\hline
$s$ & (0,2) & (0,3) & (0,4) & (0,5) & (0,6) & (0,7) \\
$g$ & (0,7) & (0,8) & (1,8) & (2,8) & (3.8) & (4,8) \\
$V(s)$ & -6 & -5 & -4 & -3 & -2 & -1 \\
\hline
\end{tabular}
\label{tab:SMDP_ME_value}
\end{table}
The trajectory of the agent using SME is marked by the blue arrows on the left side of~\cref{fig:Execution_Suboptimality}.
$\hat\pi_{hi}$ will sequentially select three subgoals $(0,1),(0,6),(3,8)$ and call $\hat\pi_{lo}$ to complete these subgoals, eventually reaching $G$.
The trajectory of the agent using ME is marked by the orange arrows on the right side of~\cref{fig:Execution_Suboptimality}.
At each frame, it selects a new optimal subgoal $g_t$ through the high-level policy $\hat\pi_{hi}$, such as $(0,1),(0,2),\dots$, and executes the optimal action $a_t$ for $g_t$ through $\hat\pi_{lo}$ for one frame.
If the agent does not encounter the stochastic transition event at $\otimes$, then both SME and ME agents will require 12 frames to reach $G$, yielding a return of $-12$.
If the agent encounters a random transition event at \(\otimes\) and is additionally teleported to \((0,2)\), the two execution modes will diverge: the agent using SME will continue executing the unfinished subtask \((0,1)\) and choose the optimal action \(\text{down}\) under this subtask; in contrast, the agent using ME will, at \((0,2)\), reselect a new optimal subtask \((0,7)\) via \(\hat{\pi}_{hi}\) and choose the optimal action \(\text{up}\) under this subtask. In this case, the agent with SME still requires 12 frames to reach \(G\), while the agent with ME only requires 10 frames to reach \(G\).
Thus,
\[
J^{ME}(\pi|G) = -12 \times 0.5 - 10 \times 0.5 = -11 > -12 = -12 \times 0.5 - 12 \times 0.5 = J^{SME}(\pi|G).
\]
In the table,
\[
V(s=(1,0)) = -1 + 0.5 \times V(s'=(0,0)) + 0.5 \times V(s'=(0,2))
= -1 - 0.5 \times 8 - 0.5 \times 6 = -8.
\]
\subsection{Policy suboptimality}
\begin{figure}[t]
\centering
\includegraphics[width=0.95\textwidth]{
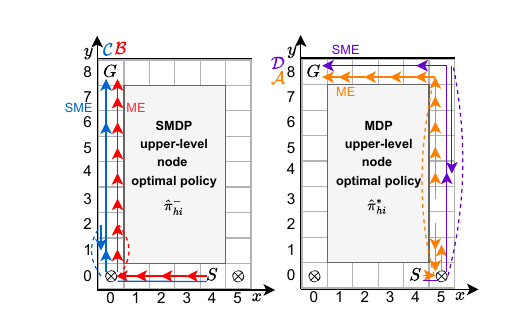
}
\caption{Policy Suboptimality.}
\label{fig:Policy_Suboptimality}
\end{figure}
Now consider an agent with a two-level policy set $\pi^*$,
the mapping of $\hat\pi^*_{hi}$ and the value estimation under ME are as~\cref{tab:MDP_ME_value_right} (right-hand path) and ~\cref{tab:MDP_ME_value_left} (left-hand path):
\begin{table}[t]
\centering
\caption{The mapping of $\hat\pi^*_{hi}$ and its value estimation under ME (right-hand path).}
\begin{tabular}{|l|c|c|c|c|c|c|c|}
\hline
$s$ & (4,0) & (5,0) & (5,1) & (5,2) & (5,3) & (5,4) & (5,5) \\
$g$ & (5,4) & (5,5) & (1,0) & (2,0) & (5,8) & (4,8) & (3,8) \\
$V(s)$ & -8 & -9 & -8 & -9 & -10 & -9 & -8 \\
\hline
$s$ & (5,6) & (5,7) & (5,8) & (4,8) & (3,8) & (2,8) & (1,8) \\
$g$ & (2,8) & (1,8) & (0,8) & (0,7) & (0,6) & (0,5) & (0,4) \\
$V(s)$ & -7 & -6 & -5 & -4 & -3 & -2 & -1 \\
\hline
\end{tabular}
\label{tab:MDP_ME_value_right}
\end{table}
\begin{table}[t]
\centering
\caption{The mapping of $\hat\pi^*_{hi}$ and its value estimation under ME (left-hand path).}
\begin{tabular}{|l|c|c|c|c|c|c|}
\hline
$s$ & (4,0) & (3,0) & (2,0) & (1,0) & (0,0) & (0,1) \\
$g$ & (5,4) & (5,3) & (0,3) & (0,4) & (0,5) & (0,6) \\
$V(s)$ & -8 & -9 & -9 & -8 & -8 & -7 \\
\hline
$s$ & (0,2) & (0,3) & (0,4) & (0,5) & (0,6) & (0,7) \\
$g$ & (0,7) & (0,8) & (1,8) & (2,8) & (3.8) & (4,8) \\
$V(s)$ & -6 & -5 & -4 & -3 & -2 & -1 \\
\hline
\end{tabular}
\label{tab:MDP_ME_value_left}
\end{table}
The agent's trajectory under SME is indicated by the purple arrows on the right of~\cref{fig:Policy_Suboptimality}. It sequentially selects three sub-tasks $(5,4),(4,8),(0,7)$ and executes them using $\hat\pi^*_{lo}$, eventually reaching $G$. The trajectory under ME is indicated by the yellow arrows on the right of~\cref{fig:Policy_Suboptimality}; at each frame, the agent selects a new sub-task $g_t$: $(5,4),(5,5)$ or $(0,8),\dots$ via the high-level policy $\hat\pi^*_{hi}$, and executes a single-frame optimal action $a_t$ under $\hat\pi^*_{lo}$.
The expected returns of the two trajectories are denoted as $J_{\mathcal D}=J^{SME}(\pi^*|G)$ and $J_{\mathcal A}=J^{ME}(\pi^*|G)$. The original policy set $\pi$ from\cref{A_1} is denoted as $\pi^-$, with its expected returns under SME and ME denoted as $J_{\mathcal C}=J^{SME}(\pi^-|G)$ and $J_{\mathcal B}=J^{ME}(\pi^-|G)$, with trajectories shown on the left side of~\cref{fig:Policy_Suboptimality}.
The calculation of $J^{SME}(\pi^*|G)$ is relatively straightforward:
\[
J_{\mathcal D}=J^{SME}(\pi^*|G)=-14\times 0.5 - 14\times 0.5 = -14.
\]
The expected return under ME is more complex, as the agent may repeatedly leverage stochastic transitions to reach the goal faster:
\[
J_{\mathcal A}=J^{ME}(\pi^*|G)=-1-\text{Time}[(5,0)\sim(5,8)]-5 = -6 - \sum_{i=0}^\infty (0.5)^{i+1} (2 \times i) = -6-2=-8.
\]
In this example, under SME, the agent's $\hat\pi_{hi}$ can only make decisions at the five positions $(4,0),(0,1),(0,6),(5,4),(4,8)$. Among the $2^5=32$ possible node policies, only $\hat\pi^-_{hi}$ and $\hat\pi^*_{hi}$ considered in this paper have finite expected returns. It is thus easy to verify that $\hat\pi^-_{hi}$ is the node policy that maximizes the expected return from $S$ to $G$ under SME\footnote{We do not further elaborate on the SMDP optimality of $\hat\pi^-_{hi}$ at other states, as the proof follows in a similar manner.}, i.e., $\hat\pi^-_{hi}$ is the SMDP optimal node policy as defined in~\cref{def:SMDP_opt}.
It can be verified that the value under ME using $\pi^*$ satisfies the Bellman optimality equation, which proves that $\pi^*_{hi}$ is the optimal policy in the single-level RL case. Therefore, $\hat\pi^*_{hi}$ is also the node policy that achieves the highest expected return under ME, i.e., $\hat\pi^*_{hi}$ is the MDP optimal node policy as defined in~\cref{def:MDP_opt}.
When the agent calls $\pi^-$ under SME, we have $J_{\mathcal C}=-12$, and for any policy $\pi$ with a different high-level node strategy, $J^{SME}(\pi|G) \leq J_{\mathcal C}$, e.g., $J_{\mathcal D}=J^{SME}(\pi^*|G)=-14<J_{\mathcal C}$.
When the agent calls $\pi^*$ under ME, we have $J_{\mathcal A}=-8$, and for any policy $\pi$ with a different high-level node strategy, $J^{ME}(\pi|G) \leq J_{\mathcal A}$, e.g., $J_{\mathcal B}=J^{ME}(\pi^-|G)=-11<J_{\mathcal A}$.
\section{Formalization and theoretical proof}
\label{sec:B}
\subsection{Task Tree and Execution Tree}
\label{sec:B.1}
This work observes that each task $\tau$ has two distinct termination functions. One controls the optimization objective of $\hat\pi_\tau$, i.e., the \textbf{termination function} $\beta_\tau$ defined in~\cref{sec:2.2}. The other controls the duration for which $\tau$ executes when it is called as a subtask during interaction, and is defined as the \textbf{interrupt function} $\hat\beta_\tau$.
In HRL, both for the execution of the behavioral policy and the target policy, the number of frames a subtask is actually executed is simply tied to the subtask's own duration, i.e., $\hat\beta_\tau \equiv \beta_\tau$. From the analysis in~\cref{sec:A}, we recognize that this is unreasonable.
To address this, we propose the \textbf{task tree} and \textbf{execution tree}, which independently describe, respectively, the complex task optimization objectives and task calling mechanisms in multi-level hierarchies. The most crucial aspect is the precise formalization and proper design of hierarchical boundary conditions when a task terminates.
We standardize the notation for node indices used in the tree structures in this paper as\cref{def:B.1}:
\begin{definition}[\textbf{Node indices in the tree}]
\label{def:B.1} Let $N$ be the set of all node indices in the tree, and $N(i) \subset N, i \geq 0$ denote the set of indices of all nodes at depth $i$ in the tree. We use the variable $n(i) \in N(i)$ as the index of a node at depth $i$, $n(i-1)$ as the index of its parent node, and $n(i+1)$ as the index of one of its child nodes. Given a path $L$ on the tree, $n(i,L) \in N(i) \cap L$ denotes the index of the node at depth $i$ on path $L$. The set of leaf node indices is denoted by $N_{\text{leaf}}$. In this paper, if the variable $x$ represents a tree, we use $x_{n(i)}$ to denote the node corresponding to index $n(i)$, and $x|n(i)$ to denote the subtree rooted at $x_{n(i)}$.
\end{definition}
\subsubsection{Task Tree}
In this paper, we use the concept of a task tree to formally describe the hierarchical nested calling relationships among multiple tasks in HRL, and to specify the optimization objective for each task.
\begin{definition}
\label{def:B.2}
\textbf{Task tree}
We extend the task $\tau$ in RL into a \textit{task tree} $\tau$.
A task tree $\tau$ is a tree structure composed of task nodes $\tau_{n(i)}, n(i) \in N$.
It describes the calling relationships among tasks and the nested relationships among the optimization objectives of the node policies of each task.
\end{definition}
\begin{definition}[\textbf{Task Node}]
\label{def:B.3}
\footnote{There are two special types of nodes in the task tree:
\begin{itemize}
  \item \textbf{External task} $\tau_{out} = \tau_{n(-1)} : (\mathcal{T}'_{out}, /, /, /)$,
  where $\mathcal{T}'_{out} = \{ \tau_{n(0)} \}$.
  This node does not require a policy optimization objective and is used solely to call the root task and wait for its termination.
  \item \textbf{Action} $\tau_{leaf} : (/, \beta_{leaf}, /, /)$,
  where $\beta_{leaf} \equiv 1$.
  This node also does not require a policy optimization objective and is used to produce a single-frame state transition.
\end{itemize}
}
A task node $\tau_{n(i)}$ is defined as a quadruple $(\mathcal{T}', \beta, R, \hat R)_{n(i)}$ consisting of:
\begin{itemize}
  \item \textbf{$\mathcal{T}'$: The set or space of child task nodes.}
  The parent-child relationships between task nodes form the tree structure of the task tree.
  \item \textbf{$\beta$: The task termination function.}
  $\beta_{n(i)}(\cdot): \mathcal{S} \to [0,1]$ represents the probability that $\tau_{n(i)}$ generates a task termination event when the agent reaches state $s'$.
  This is an extension of the termination function in ~\cref{sec:2.1} of the main text.
  As will be shown in~\cref{def:B.4}, the termination of a task at each frame is jointly determined by its own and all its ancestor task nodes' termination events, which defines the temporal scope of reward accumulation.
  \item \textbf{$R$: The per-frame reward function.}
  $R_{n(i)}(\cdot,\tau_{leaf}): \mathcal{S} \to \mathbb{R}$ provides a per-frame reward whenever the agent takes an action $a$ at state $s$ during the interaction process of $\tau_{n(i)}$.
  \item \textbf{$\hat{R}$: The exit reward function.}
  $\hat{R}_{n(i)}(\cdot,\tau_{n(j)}): \mathcal{S} \to \mathbb{R},\ -1\leq j < i$, provides a boundary reward for the entire process when $\tau_{n(i)}$ exits to $\tau_{n(j)}$ due to its own or an ancestor node's termination event.
\end{itemize}
\end{definition}
\begin{definition}[\textbf{Task termination event}]
\label{def:B.4}
For a task node $\tau_{n(i)}$, if a task termination event is triggered by itself or any of its ancestor nodes $\tau_{n(j)}$, $0 \leq j \leq i$,
then the termination event of the highest-level ancestor $\tau_{n(j_{\min})}$ takes effect.
In this case, the task node $\tau_{n(i)}$ is terminated and exits to its ancestor node $\tau_{n(j_{\min-1})}$.
This mechanism ensures that when an ancestor task is terminated, all of its descendant tasks are also terminated,
thereby maintaining the nested structure among tasks.
\end{definition}
The task tree functions in the following two phases:
1) \textbf{Interaction phase per frame:} After the agent executes an action $a$, each task node $\tau_{n(i)}$ will individually obtain a per-frame reward for this task through its per-frame reward function $R_{n(i)}(s,a)$.
2) \textbf{Termination phase per frame:} All task nodes in the task tree will independently trigger task termination events according to their termination functions.
Based on the triggered termination events, Definition B.4 determines whether each task is terminated.
If a task $\tau_{n(i)}$ is terminated, the accumulation of per-frame rewards ends, and an exit reward $R_{n(i)}(s',\tau_{n(j_{\min}-1)})$ is provided by the exit reward function.
Otherwise, the task $\tau_{n(i)}$ continues.
The task tree defines the optimization objective of each task node's policy, which is to maximize the expected return
$\mathbb{E}_{s_0=s}\big[\sum_{t=0}^{T-1} R_{n(r)}(s_t,a_t)+\hat R_{n(r)}(s_T,\tau_{n(j_{\min}-1)})\big]$
starting from any state $s$, where $T$ is the number of frames when the task is terminated.
Unlike standard RL, for tasks in HRL we must also consider the hierarchical structure at the task's termination.
Specifically, when different termination events occur, they should correspond to different task boundary conditions, as in the option framework.
\paragraph{Example: Option framework.}
The root task has a countable set of subtasks (options). The termination function of the root task is given by the environment. The per-frame reward of the root task is defined as \(R_{n(0)}(s,\tau_{leaf}) = R(s,a)\), which is the environmental reward, and its exit reward is \(R_{n(0)}(s,\tau_{n(-1)}) \doteq 0\).
For each option, its subtasks are the actions. The termination function of an option is learned via the termination gradient theorem. The per-frame reward of the option is \(R_{n(1)}(s,\tau_{leaf}) = R(s,a)\), and its exit rewards are \(R_{n(1)}(s,\tau_{n(-1)}) \doteq 0\) and \(R_{n(1)}(s,\tau_{n(0)}) \doteq V(s)\), where \(V(s)\) denotes the state value function of the root task. By setting the exit reward to the value function of the upper-level task, the lower-level policy is optimized to improve the upper-level task.
\paragraph{Example: Goal-based framework.}
In the Goal-based framework, the root task has a subtask space (sub-goal space). The termination function of the root task is given by the environment. The per-frame reward of the root task is \(R_{n(0)}(s,\tau_{leaf}) = R(s,a)\), and its exit reward is \(R_{n(0)}(s,\tau_{n(-1)}) \doteq 0\).
For each goal, its subtasks are the actions. The termination function of a goal is fixed as \(\bar\beta(s') \doteq \mathbb{I}[d(s',g)<d_0]\), where \(d\) is a distance metric and \(d_0\) is a hyperparameter related to neighborhood determination. The reward tree of a goal is provided by itself: its per-frame reward \(R_{n(1)}(s,\tau_{leaf})\) is defined by the goal, and its exit rewards are \(R_{n(1)}(s,\tau_{n(-1)}) \doteq 0\) and \(R_{n(1)}(s,\tau_{n(0)}) \doteq 0\). By setting the exit reward to zero, the lower-level tasks become feudal, i.e., independent from the upper-level value.
\subsubsection{Execution Tree}
This paper introduces the concept of an execution tree to describe how an agent in HRL executes its hierarchical policies to interact with the environment.
We observe a fundamental difference between HRL and standard RL:
during the execution process of HRL,
in addition to each task node having its own node policy function, an extra \textit{execution mechanism} is required to properly organize and call these node policies at every frame.
This is the key point discussed in~\cref{sec:3.1} of the main text.
\textbf{The differences in execution mechanisms can significantly affect the agent's expected return, its exploration capability, and the data efficiency of each node policy during the learning process.}
Therefore, the choice of execution mechanisms for the agent's behavior policy during training, the target policy during training, and the policy during deployment are all crucial research questions that have been largely overlooked.
\begin{definition}[\textbf{Execution Tree}]
\label{def:B.5}
An execution tree $e$ is a tree structure composed of execution nodes $e_{n(i)}, n(i) \in N$, used to control the agent's level-wise decision-making to call actions and complete various tasks in the task tree. Except for the two special nodes in the task tree, each execution node $e_{n(i)}$ corresponds one-to-one with a task node $\tau_{n(i)}$. The execution subtree $e|n(i)$ is used to execute the corresponding task $\tau_{n(i)}$.
\end{definition}
\begin{definition}[\textbf{Execution Node}]
\label{def:B.6}
An execution node $e_{n(i)}$ is composed of a tuple $(\hat\pi, \hat\beta)_{n(i)}$:
\begin{itemize}
    \item \textbf{$\hat\pi$: Node policy function.} $\hat\pi_{n(i)}(\cdot): \mathcal{S} \to \Delta \mathcal{T}'_{n(i)}$ (or equivalently $\hat\pi_{n(i)}(\cdot|s): \mathcal{T}'_{n(i)} \to [0,1]$), which allows task $\tau_{n(i)}$ to select a subtask when at state $s$.
    \item \textbf{$\hat\beta$: Decision interruption function.} $\hat\beta_{n(i)}(\cdot): \mathcal{S} \to [0,1]$ denotes the probability that $e_{n(i)}$ generates a decision interruption event after the agent reaches state $s'$.
\end{itemize}
\end{definition}
\begin{definition}[\textbf{Decision Interruption Event}]
\label{def:B.7}
For an execution node $e_{n(i)}$, if a decision interruption event occurs at $e_{n(i)}$ itself or any of its ancestors $e_{n(j)}, 0 \leq j \leq i$, then the interruption from the highest-level ancestor $e_{n(j_{\min})}$ takes effect. The decision authority of execution node $e_{n(i)}$ is then transferred to the ancestor node $e_{n(j_{\min}-1)}$.
\end{definition}
\begin{definition}[\textbf{Policy Tree}]
\label{def:B.8}
The tree structure in the execution tree $e$ formed by the node policy functions is referred to as the policy tree $\Pi$.
\end{definition}
\begin{definition}[\textbf{Interruption Tree}]
\label{def:B.9}
The tree structure in the execution tree $e$ formed by the decision interruption functions is referred to as the interruption tree $\mathbf{B}$.
\end{definition}
When the agent interacts with the environment using the execution tree, only the execution nodes that currently hold \textbf{decision authority} can call their node policies to make decisions. Taking the interaction of executing task $\tau_{n(h)}$ using the execution subtree $e|n(h)$ as an example, the execution tree operates in the following two phases:
1) \textbf{Preparation phase of each frame:} At frame 0, $e_{n(h)}$ directly obtains decision authority. In subsequent frames, the decision authority initially resides with the execution node $e_{n(l)}$ that made the decision in the previous frame. All execution nodes within the subtree $e|n(h)$, \textbf{\textit{except the root $e_{n(h)}$}}, then generate decision interruption events according to their interruption functions. The decision authority is transferred based on Definition B.7. If $e_{n(l)}$ is interrupted, the authority is passed to its ancestor $e_{n(j_{min-1})}$; otherwise, it remains with $e_{n(l)}$.
2) \textbf{Decision phase of each frame:} The execution node holding the decision authority selects a subtask according to its node policy function and continuously passes the decision authority to the execution node corresponding to the selected subtask. If there is no corresponding execution node (e.g., for actions), the authority is not transferred. This process continues until an action is selected and executed.
At this point, we have decoupled the original termination function $\beta$ under the definitions of the task tree and execution tree: $\beta$ now controls the intrinsic optimization objective of the task, while $\hat\beta$ controls the calling of node policies during execution.
In the task tree, the optimal subtask set for facilitating task learning can be found by optimizing $\beta$ and $R$. In existing algorithms, some $\beta$ functions are fixed; for example, in the Goal-based framework, each task (goal) has a fixed $\beta$. Others are learnable; for instance, in the OC method, each task (option) has a $\beta$ optimized via the termination gradient theorem.
In the execution tree, the optimization of $\hat\beta$ must balance the exploration and exploitation of task $\tau_{n(h)}$ during both behavior policy and target policy calling. On one hand, it should consider interacting with the environment under a certain execution mode to achieve high sampling efficiency (exploration). On the other hand, it should consider planning an optimal policy tree under the given execution mode using existing data to ensure maximal expected return (exploitation)\footnote{Analysis shows that there is room for improvement in the design of OC here. The termination gradient theorem should be used for ``exploitation'' to learn $\hat\beta$ in order to maximize expected return, instead of modifying the subtasks' $\beta$.
When $\beta$ approaches $0$, subtasks rarely terminate, causing the subtask learning difficulty to approach that of the original task. Conversely, when $\beta$ approaches $1$, subtasks terminate almost every frame, preventing the higher-level task from effectively leveraging the subtask for efficient exploration.
}.
Only by decoupling the intrinsic optimization objective of the task, the execution mode of the behavior policy, and the execution mode of the target policy can the HRL field further address and study the exploration-exploitation tradeoff during the learning process.
For the sake of formal convenience in the proofs of subsequent theorems, we extend and decouple the discount factor $\gamma$ in the same manner as the termination function $\beta$. Unlike $\beta$ and $\hat\beta$, which are bound to each node, $\gamma$ and $\hat\gamma$ are defined to directly describe task transitions or decision transitions during the interaction process.
\begin{definition}[\textbf{Task Transition Mechanism $\gamma$}]
\label{def:B.10}
Let $\bar\gamma^{n(h)}_{n(i)}$ denote the probability that the currently executing task $\tau_{n(h)}$ exits to its ancestor $\tau_{n(i)}, -1 \leq i \leq h$ during each frame's termination phase based on the task termination events.
\end{definition}
$\gamma^{n(h)}_{n(i)}$ can be computed from $\beta$: when exiting to $\tau_{n(i)}, -1 \leq i < h$, $\gamma^{n(h)}_{n(i)}$ is the probability that along the path from $\tau_{n(h)}$ to the root task $\tau_{n(0)}$, $\tau_{n(i+1)}$ triggers a termination event while all its ancestors do not trigger any termination events, i.e.,
\[
\gamma^{n(h)}_{n(i)} = \beta_{n(i+1)} \prod_{k=0}^{i} (1 - \beta_{n(k)}).
\]
When $\tau_{n(h)}$ does not exit,
\[
\gamma^{n(h)}_{n(h)} = 1 - \sum_{k=0}^{h-1} \gamma^{n(h)}_{n(k)}.
\]
In this paper, when task $\tau_{n(h)}$ is clear from context, the superscript of $\gamma^{n(h)}_{n(i)}$ is omitted for simplicity, and we write it as $\gamma_{n(i)}$.
\begin{definition}[\textbf{Decision transition mechanism $\hat\gamma$}]
\label{def:B.11}
Let $\hat\gamma^{n(h)}_{n(i)}$ denote the probability that, when executing the task $\tau_{n(h)}$ using the execution subtree $e|n(h)$, the decision authority at the preparation stage of each frame transfers from the execution node $e_{n(l)}$ that made the final decision in the previous frame to $e_{n(i)}$.
\end{definition}
$\hat\gamma^{n(h)}_{n(i)}$ can be computed from $\hat\beta$: when the decision authority transfers to $e_{n(i)}, h \leq i < l$, $\hat\gamma^{n(h)}_{n(i)}$ is the probability that, along the path from $e_{n(l)}$ to $e_{n(h)}$, $e_{n(i+1)}$ triggers a decision interruption while none of its ancestors do, i.e.,
\[
\hat\gamma^{n(h)}_{n(i)} = \hat\beta_{n(i+1)} \prod_{k=h+1}^{i} (1-\hat\beta_{n(k)}).
\]
When the authority does not transfer,
\[
\hat\gamma^{n(h)}_{n(l)} = 1 - \sum_{k=h}^{l-1} \hat\gamma^{n(h)}_{n(k)}.
\]
When the task $\tau_{n(h)}$ is clear from context, we omit the superscript and write $\hat\gamma^{n(h)}_{n(i)}$ simply as $\hat\gamma_{n(i)}$.
\subsubsection{Two special execution modes: SME and ME}
Existing hierarchical reinforcement learning algorithms often execute their policy trees using the SME mode, without decoupling the task termination functions. The transfer of decision authority is controlled by task termination events, and the decision interruption events are bound to the task termination events. In the computation of the decision transition mechanism $\hat\gamma$, it holds that $\hat\beta \dot= \beta$.
If all node policies within a policy subtree $\Pi|n(h)$ are called under SME, it is referred to as SME called the policy subtree $\Pi|n(h)$. The interaction process when an agent uses SME to execute task $\tau_{n(h)}$ with the policy subtree $\Pi|n(h)$ is illustrated in~\cref{fig:SME_process}:
\begin{figure}[t]
\centering
\includegraphics[width=0.95\textwidth]{
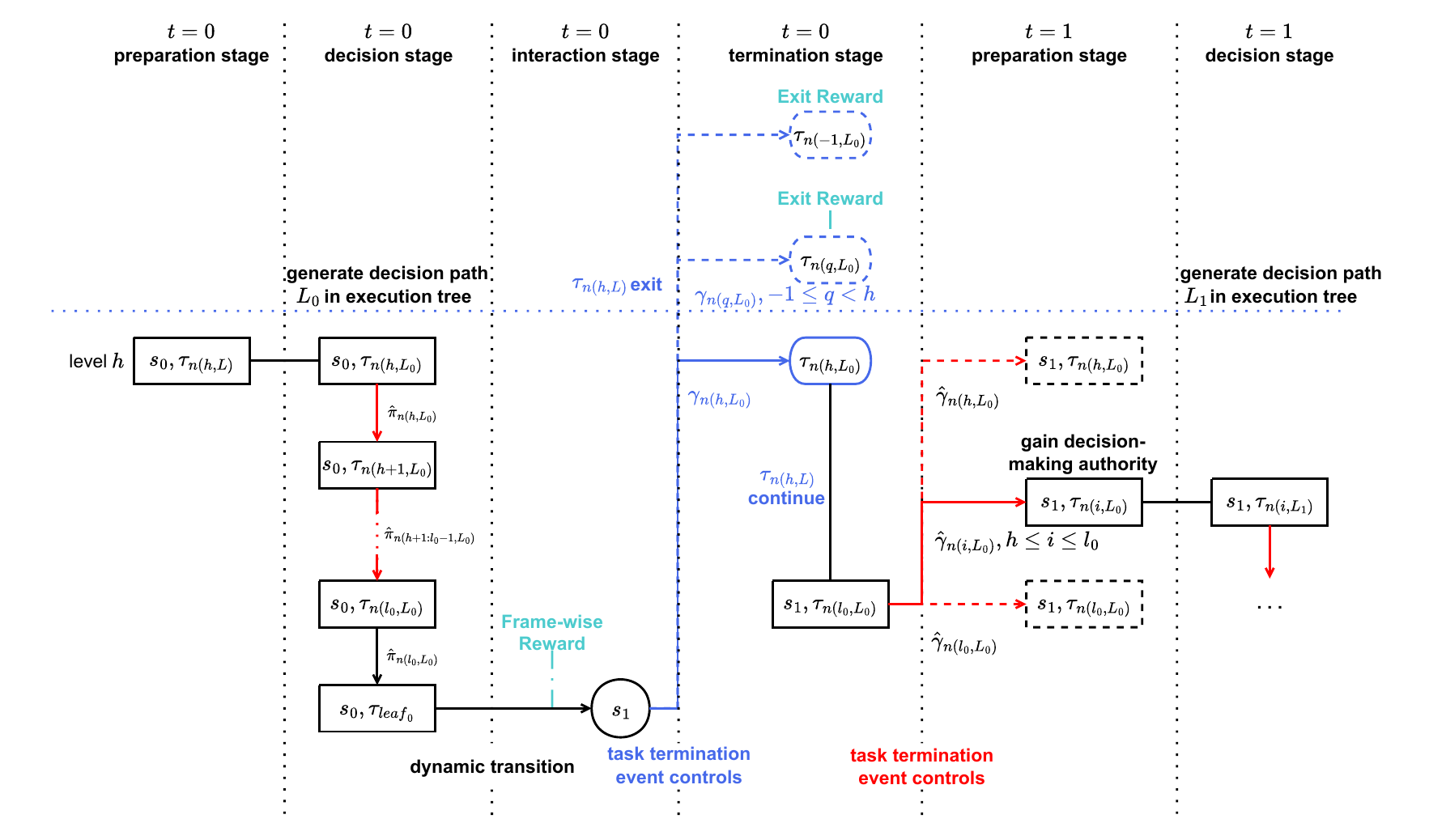
}
\caption{SME interaction process.}
\label{fig:SME_process}
\end{figure}
Since under SME the transfer of decision authority is controlled by task termination events, during the SME call of a policy subtree $\Pi|_{n(h')}$ with $h' < h$ to execute task $\tau_{n(h)}$, if a task termination event causes the decision authority to exit the execution subtree $e|_{n(h)}$, then the task $\tau_{n(h)}$ will also be terminated. Therefore, for the entire execution of task $\tau_{n(h)}$, it is equivalent to calling the policy subtree $\Pi|_{n(h)}$ under SME. Consequently, when the agent interacts with the environment using the SME call of the policy subtree $\Pi|_{n(h')}$ as the behavior policy, all descendant tasks $\tau_{n(h)}$ can obtain the complete trajectory data corresponding to executing $\tau_{n(h)}$ under an SME call of $\Pi|_{n(h)}$.
After decoupling the termination functions, the agent's decision interruption events and task termination events become independent. When the agent executes task $\tau_{n(h)}$ by calling the policy subtree $\Pi|_{n(h)}$ under ME, we have $\hat\beta \equiv 1$ for the descendant execution nodes.
If all node policies within the policy subtree $\Pi|_{n(h)}$ are called under ME, the subtree is referred to as a ME called policy subtree $\Pi|_{n(h)}$. The interaction process when the agent executes task $\tau_{n(h)}$ using $\Pi|_{n(h)}$ under ME is illustrated in~\cref{fig:ME_process}:
\begin{figure}[t]
\centering
\includegraphics[width=0.95\textwidth]{
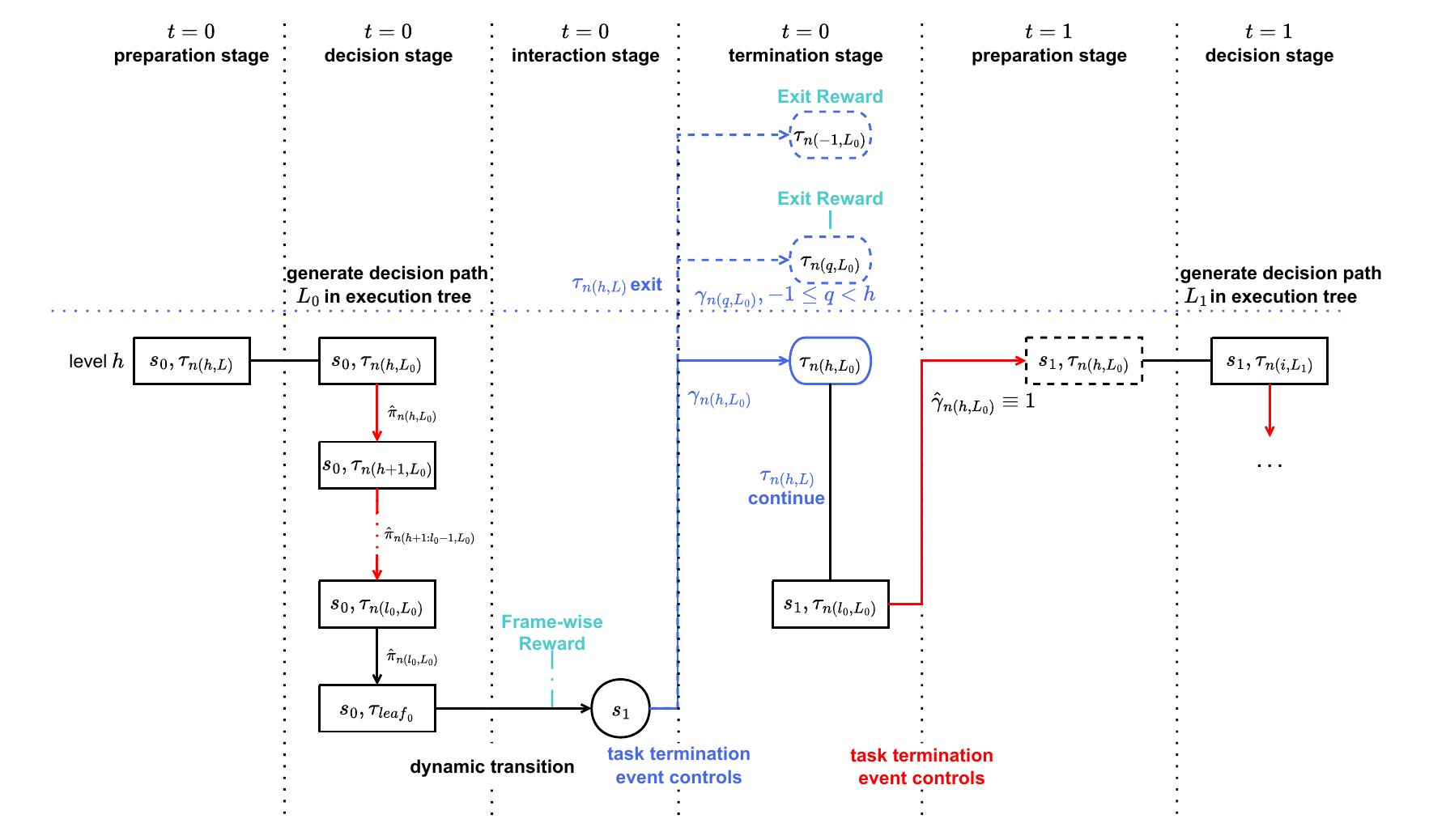
}
\caption{ME interaction process.}
\label{fig:ME_process}
\end{figure}
Unlike in SME, decision interrupt events are independent of task termination events in ME. During the execution of task $\tau_{n(h)}$ with the ME called policy subtree $\Pi|_{n(h')}, h'<h$, if the decision authority is transferred out of the execution subtree $e|_{n(h)}$ in the preparation phase due to a decision interrupt event while the task $\tau_{n(h)}$, controlled by task termination events, is not terminated, and in the next frame the decision authority is not returned to $e_{n(h)}$, as~\cref{fig:ME_process_offline}, this will cause an off-policy discrepancy between the target policy (decisions made downward from $e_{n(h)}$) and the behavior policy (decisions made downward from the execution node $e'_{n(h)}$ that holds the decision authority).
\begin{figure}[t]
\centering
\includegraphics[width=0.95\textwidth]{
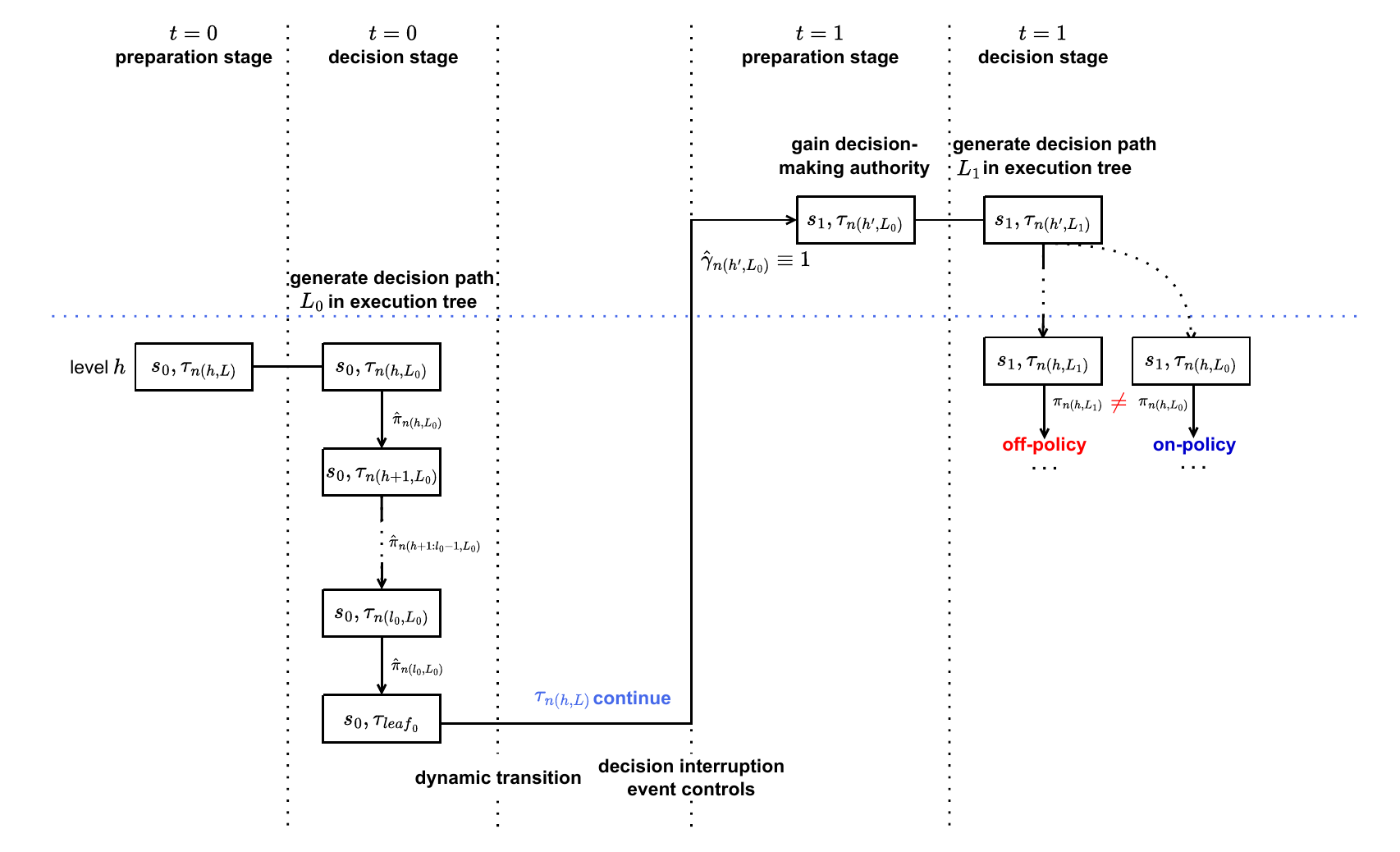
}
\caption{Off-policy issue in subtask node policy learning under ME.}
\label{fig:ME_process_offline}
\end{figure}
\subsection{Hierarchical Bellman Equation}
\label{sec:B.2}
\subsubsection{Unified Value Function in HRL}
We need to establish a consistent representation such that the value of any policy can be expressed under any execution mechanism. To this end, we construct a \emph{unified value function} applicable to HRL under arbitrary execution modes, which describes the expected return across different settings.
\begin{definition}[\textbf{Unified Value Function for HRL, UVFH}]
\label{def:B.12}
The unified value function in HRL is denoted as $V_e(\hat s|\tau)$, where $\hat s$ is the extended state of the agent at the decision stage, and $\tau$ is the task node whose value we aim to evaluate.
\begin{itemize}
    \item
    \textbf{Extended state} $\hat s \dot= (s, \tau_{n(i)})$: represents the agent being at state $s$ during the decision stage and the execution node $e_{n(i)}$ having the decision authority (or having selected an action).
    \item
    \textbf{task} $\tau \dot= \tau_{n(h)}$: specifies the scope of per-frame reward accumulation, the per-frame reward values, and the exit reward values.
    \item
    \textbf{execution tree} $e$: is the execution tree used by the agent.
    \item
    \textbf{unified value function} is defined as
    \[
    V^\mathbf{B}_\Pi(\hat s|\tau) = V_e(\hat s|\tau) = \mathbb{E}_{s_0=s, T \sim \tau, e}\Bigg[\sum_{t=0}^{T-1} R_{n(h)}(s_t, a_t) + \hat R_{n(h)}(s_T, \tau_{n(j_{min}-1)})\Bigg],
    \]
    which represents the cumulative sum of all rewards, including per-frame and exit rewards, when the agent uses execution tree $e$, starts from $\hat s$, and executes until task $\tau$ terminates.
\end{itemize}
\end{definition}
In particular, given a strategy tree $\Pi$, when the execution modes are SME or ME, the UVFH for these two special execution trees are denoted as $V^{SME}_{\Pi}$ and $V^{ME}_{\Pi}$, respectively.
The state value and state-action value in RL are unified into the value of the extended state $\hat s$ at the decision stage. This transforms the non-Markovian process in HRL, given $\tau$ and $e$, into an extended Markov process that is easier to analyze. Meanwhile, both the Goal-based framework and the Option framework in HRL are unified under the task node $\tau$.
\subsubsection{Generalized Hierarchical Bellman Equation}
For each task node in the task tree, there exists a set of generalized Bellman equations for value estimation.
When the agent executes a task $\tau_{n(h)}$, $h \geq h'$, using any given execution subtree $e|n(h')$ of $e$, we can iteratively compute the HRL unified value through the generalized hierarchical Bellman equation given $\tau_{n(h)}$ and $e|n(h')$, i.e., value evaluation:
\begin{definition}[\textbf{Generalized Hierarchical Bellman Equation, GHBE}]
\label{def:B.13}
\textbf{Decision stage:}
\[
V_{e|n(h')} (s, \tau_{n(i)}| \tau_{n(h)}) = \mathbb{E}_{\tau_{n(i+1)} \sim \hat\pi_{n(i)}} \left[ V_{e|n(h')} (s, \tau_{n(i+1)}| \tau_{n(h)}) \right], \quad \forall n(i)
\]
\textbf{Interaction stage:}
\[
V_{e|n(h')} (s, \tau_{leaf}| \tau_{n(h)}) = R_{n(h)}(s,a) + \mathbb{E}_{s' \sim T(s,a)} \left[ \bar V_{e|n(h')} (s', \tau_{leaf}| \tau_{n(h)}) \right]
\]
\textbf{Termination stage:}
\[
\bar V_{e|n(h')} (s', \tau_{leaf}| \tau_{n(h)}) =
\hat V_{e|n(h')} (s', \tau_{leaf}| \tau_{n(h)}) \gamma_{n(h)}(s') +
\sum^{h-1}_{k=-1} \gamma_{n(k)}(s') \hat R_{n(h)} (s', \tau_{n(k)})
\]
\textbf{Preparation stage:}
\[
\hat V_{e|n(h')} (s', \tau_{leaf}| \tau_{n(h)}) =
\sum^l_{k=h'} \hat\gamma_{n(k)} (s') V_{e|n(h')} (s', \tau_{n(k)}| \tau_{n(h)})
\]
\end{definition}
$\bar V$ and $\hat V$ are intermediate variables. In the termination stage equation, the first term on the right-hand side is the bootstrapped value when $\tau_{n(h)}$ has not yet terminated, and the second term is the exit reward provided by $\hat R_{n(h)}$ when $\tau_{n(h)}$ terminates and transfers to another task. In the preparation stage equation, the right-hand side represents the bootstrapped value when decision authority jumps occur due to decision interruption events.
In particular, since the execution subtree $e|n(h)$ will be used to execute task $\tau_{n(h)}$, we need to optimize the node policy $\hat\pi_{n(h)}$ given all decision interruption functions in $e|n(h)$ and all node policies except $\hat\pi_{n(h)}$, such that the expected return in $\tau_{n(h)}$
\[
J(e|\tau_{n(h)}) =
\mathbb{E}_{s \sim \alpha_{n(h)}}
\left[
V_{e|n(h)}(s, \tau_{n(h)} | \tau_{n(h)})
\right]
\]
is maximized.
If we want to optimize the node policy $\hat\pi_{n(h)}$ in $e|n(h')$, given all decision interruption functions and all node policies except $\hat\pi_{n(h)}$ for $h \geq h'$, such that the expected return in $\tau_{n(h')}$
\[
J(e|\tau_{n(h')}) =
\mathbb{E}_{s \sim \alpha_{n(h')}}
\left[
V_{e|n(h')}(s, \tau_{n(h')} | \tau_{n(h')})
\right]
\]
is maximized, this can be achieved via the Generalized Hierarchical Bellman Optimality Equation (GHBOE) for the optimization of $\hat\pi_{n(h)}$ given $\tau_{n(h')}$ and $e|n(h')$. We refer to this as the GHBOE for $\hat\pi_{n(h)}$. That is, in HRL, the optimization objective of a node policy can be set as maximizing the expected return of an ancestor task node $\tau_{n(h')}$, including its own task node. We denote the results of GHBOE iterations by adding superscripts to $V$, $\bar V$, and $\hat V$ as $V^*$, $\bar V^*$, $\hat V^*$ to indicate the computed values from the GHBOE instead of the value evaluation under a specific execution tree.
\begin{definition}[\textbf{GHBOE for $\hat\pi_{n(h)}$}]
\label{def:B.14}
\textbf{decision stage:}
\[
V^*_{e|n(h')} (s, \tau_{n(h)}| \tau_{n(h')})
= \max_{\tau_{n(h+1)}} V^*_{e|n(h')} (s, \tau_{n(h+1)}| \tau_{n(h')})
\]
\[V^*_{e|n(h')} (s, \tau_{n(i)}| \tau_{n(h')})
= \mathbb{E}_{\tau_{n(i+1)} \sim \hat\pi_{n(i)}} \left[ V^*_{e|n(h')} (s, \tau_{n(i+1)}| \tau_{n(h')}) \right], \quad \forall n(i) \neq n(h)\]
\textbf{interaction stage:}
\[V^*_{e|n(h')} (s, \tau_{leaf}| \tau_{n(h')})
= R_{n(h')}(s,a) + \mathbb{E}_{s' \sim T(s,a)} \left[ \bar V^*_{e|n(h')} (s', \tau_{leaf}| \tau_{n(h')}) \right]\]
\textbf{termination stage:}
\[\bar V^*_{e|n(h')} (s', \tau_{leaf}| \tau_{n(h')})
= \hat V^*_{e|n(h')} (s', \tau_{leaf}| \tau_{n(h')}) \gamma_{n(h')}(s')
+ \sum^{h-1}_{k = -1} \gamma_{n(k)}(s') \hat R_{n(h')} (s', \tau_{n(k)})\]
\textbf{preparation stage:}
\[\hat V^*_{e|n(h')} (s', \tau_{leaf}| \tau_{n(h')})
= \sum^l_{k = h'} \hat\gamma_{n(k)} (s') V^*_{e|n(h')} (s', \tau_{n(k)}| \tau_{n(h')})\]
\end{definition}
Next, we introduce the special cases of GHBE and GHBOE under the two particular execution modes, SME and ME. The difference in execution modes only affects the preparation stage of the equations.
For SME, we have:
\[
\hat V^{SME}_{\Pi|n(h')} (s', \tau_{leaf}| \tau_{n(h')}) =
\sum^l_{k = h'} \gamma_{n(k)} (s') V^{SME}_{\Pi|n(h')} (s', \tau_{n(k)}| \tau_{n(h')}),
\]
and the corresponding equation sets are denoted as HBE-SMDP and HBOE-SMDP.
For ME, we have:
\[
\hat V^{ME}_{\Pi|n(h')} (s', \tau_{leaf}| \tau_{n(h')}) =
V^{ME}_{\Pi|n(h')} (s', \tau_{n(h')}| \tau_{n(h')}),
\]
and the corresponding equation sets are denoted as HBE-MDP and HBOE-MDP.
As an extension of~\cref{sec:4.1}, given a task tree $\tau$ and a policy subtree $\Pi|n(h)$ with all node policies except the root node $\hat\pi_{n(h)}$ fixed, we define $\Pi^-$ as the policy tree derived from the solution of $\hat\pi_{n(h)}$'s HBOE-SMDP. The expected return under SME in $\tau_{n(h)}$ is denoted as $J_{\mathcal C}$, and under ME as $J_{\mathcal B}$. Similarly, given a policy subtree $\Pi|n(h)$ with all node policies except the root node $\hat\pi_{n(h)}$ fixed, we define $\Pi^*$ as the policy tree derived from the solution of $\tau_{n(h)}$'s HBOE-MDP. The expected return under SME in $\tau_{n(h)}$ is denoted as $J_{\mathcal D}$, and under ME as $J_{\mathcal A}$. We will prove the convergence of these solutions and the optimality of the corresponding policies.
\subsection{Generalized SMDP Suboptimality Theory}
\begin{theorem}[\textbf{Generalized SMDP Suboptimality Theory}]
\label{def:B.15}
In HRL problems where the environment satisfies the Markov property, given a task tree $\tau$ and a policy tree $\pi$ with all node policies except $\hat\pi_{n(h)}$ fixed, the policy $\hat\pi^-_{n(h)}$ derived from the HBOE-SMDP of $\hat\pi_{n(h)}$ is the SMDP optimal node policy for $\tau_{n(h)}$, while the policy $\hat\pi^*_{n(h)}$ derived from the HBOE-MDP of $\hat\pi_{n(h)}$ is the MDP optimal node policy for $\tau_{n(h)}$. As shown in Table 5.1, the four expected returns under SME and ME satisfy
$J_{\mathcal A} \geq J_{\mathcal B}$ and $J_{\mathcal C} \geq J_{\mathcal D}$. Under~\cref{def:B.16}, we also have $J_{\mathcal B} \geq J_{\mathcal C}$.
\end{theorem}
\begin{claim}[\textbf{Task Subtree Compatibility Condition}\footnote{\cref{def:4.1} and existing mainstream HRL algorithms are both two-level, and thus clearly satisfy~\cref{def:B.16}.}]
\label{def:B.16}
A task subtree $\tau|n(h)$ rooted at task $\tau_{n(h)}$ either has only two levels, or all of its descendant task nodes $\tau_{n(h')}, h' > h$ have their child task sets $\mathcal{T}'_{\tau_{n(h')}}$ satisfy $\mathcal{T}'_{\tau_{n(h')}} \subseteq \mathcal{T}'_{\tau_{n(h)}}$.
\end{claim}
Based on the formal tools provided in~\cref{sec:B.1} and~\cref{sec:B.2}, we now formally begin the proof of~\cref{def:B.15}.
\begin{lemma}
\label{def:B.17}
When $\forall s, n(h) \in N(h)$ satisfy $1-\gamma_{n(h)}(s) \geq \epsilon >0$, the GHBE for executing $\tau_{n(h)}$ using the execution subtree $e|n(h)$ has a unique fixed point. If $\max_{\tau,s}|\hat R_{n(h)}(s,\tau)|<M$ and $\max_{a,s}|R_{n(h)}(s,a)|<M$, then this fixed point $V^*$ satisfies $\max_{s,\tau}|V^*(s,\tau)| \leq \frac{2M}{\epsilon}$.
\end{lemma}
\textbf{Proof:}
The GHBE operator for executing $\tau_{n(h)}$ using the execution subtree $e|n(h)$ can be written as:
\begin{equation}
\begin{split}
\mathcal{F}^{e,\tau}_{n(h)} V(s,\tau_{n(i)}) =&
\mathbb{E}_{a\sim \pi_{n(i)}(s)}
\Big[
R_{n(h)}(s,a) \\ &+
\mathbb{E}_{s' \sim T(s,a)}\Big[
\gamma_{n(h)}
\sum^l_{k=h} \hat\gamma_{n(k)}(s')V(s',\tau_{n(k)})
\\ &+
\sum^{h-1}_{k=-1} \gamma_{n(k)}(s')\hat R_{n(h)}(s',\tau_{n(k)})
\Big]
\Big],\quad \forall n(i) \in N|n(h)
\end{split}\nonumber
\end{equation}
\noindent First, we prove that $\mathcal{F}^{e,\tau}_{n(h)}$ is a contraction mapping. Let $V:\mathcal{S} \times \mathcal{T} \to \mathbb{R}$ be a bounded function defined on the state space $\mathcal{S}$ and task space $\mathcal{T}$ (the union of all tasks in the task tree, which is complete). Denote the set of such functions by $B(\mathcal{S},\mathcal{T})$. Define the metric $d$ on $B$ as:
\[
\forall V_1, V_2 \in B(\mathcal{S},\mathcal{T}), \quad d(V_1,V_2) = \max_{s \in \mathcal{S}, \tau \in \mathcal{T}} |V_1(s,\tau) -V_2(s,\tau)|.
\]
\noindent Then we have $V_1(s,\tau) - V_2(s,\tau) \leq d(V_1,V_2)$, i.e., $V_1(s,\tau) \leq V_2(s,\tau) + d(V_1,V_2)$.
\noindent When the conditions of Theorem 5.2 hold, for all $e$ and for all $s \in \mathcal{S}, \tau_{n(i)} \in \mathcal{T}$,
\[
\begin{aligned}
\mathcal{F}^{e,\tau}_{n(h)} V_1(s,\tau_{n(i)}) &=
\mathbb{E}_{a\sim \pi_{n(i)}(s)}
\Big[
R_{n(h)}(s,a)\\&~~~~+
\mathbb{E}_{s' \sim T(s,a)}
\Big[
\gamma_{n(h)} \sum^l_{k=h} \hat\gamma_{n(k)}(s') V_1(s',\tau_{n(k)}) \\&~~~~+
\sum^{h-1}_{k=-1} \gamma_{n(k)}(s')\hat R_{n(h)}(s',\tau_{n(k)})
\Big]
\Big] \\
&\leq
\mathcal{F}^{e,\tau}_{n(h)} V_2(s,\tau_{n(i)}) + (1-\epsilon)d(V_1,V_2)
\end{aligned}
\]
\noindent From this, we have for all $s \in \mathcal{S}, \tau_{n(i)} \in \mathcal{T}$,
\[
\mathcal{F}^{e,\tau}_{n(h)} V_1(s,\tau_{n(i)}) - \mathcal{F}^{e,\tau}_{n(h)} V_2(s,\tau_{n(i)}) \leq (1-\epsilon)d(V_1,V_2).
\]
\noindent Similarly, for all $s \in \mathcal{S}, \tau_{n(i)} \in \mathcal{T}$,
\[
\mathcal{F}^{e,\tau}_{n(h)} V_2(s,\tau_{n(i)}) - \mathcal{F}^{e,\tau}_{n(h)} V_1(s,\tau_{n(i)}) \leq (1-\epsilon)d(V_1,V_2).
\]
\noindent Hence, for all $s \in \mathcal{S}, \tau_{n(i)} \in \mathcal{T}$,
\[
\big|\mathcal{F}^{e,\tau}_{n(h)} V_1(s,\tau_{n(i)}) - \mathcal{F}^{e,\tau}_{n(h)} V_2(s,\tau_{n(i)})\big| \leq (1-\epsilon)d(V_1,V_2).
\]
\noindent Therefore, for all $V_1, V_2 \in B$,
\[
d(\mathcal{F}^{e,\tau}_{n(h)} V_1, \mathcal{F}^{e,\tau}_{n(h)} V_2) \leq (1-\epsilon) d(V_1, V_2).
\]
\noindent Since $\epsilon > 0$, the operator $\mathcal{F}^{e,\tau}_{n(h)}$ is a contraction mapping for any given $e$. By the Banach fixed-point theorem, $\mathcal{F}^{e,\tau}_{n(h)}$ has a unique fixed point, denoted by $V^*$.
\noindent Next, we show the bound of $V^*$. For all $s, \tau_{n(i)}$,
\[
\begin{aligned}
\mathcal{F}^{e,\tau}_{n(h)} V^*(s,\tau_{n(i)}) &=
\mathbb{E}_{a \sim \pi_{n(i)}(s)}
\Big[
R_{n(h)}(s,a) \\&~~~~+
\mathbb{E}_{s' \sim T(s,a)}
\big[
\gamma_{n(h)} \sum^l_{k=h} \hat\gamma_{n(k)}(s') V^*(s',\tau_{n(k)}) \\&~~~~+
\sum^{h-1}_{k=-1} \gamma_{n(k)}(s') \hat R_{n(h)}(s',\tau_{n(k)})
\big]
\Big] \\
&\leq M + \gamma_{n(h)} \max_{s,\tau} |V^*(s,\tau)| + M \\
&\leq 2M + (1-\epsilon) \gamma_{n(h)} \max_{s,\tau} |V^*(s,\tau)|.
\end{aligned}
\]
\noindent Thus, we have
\[
\max_{s,\tau} |V^*(s,\tau)| \leq 2M + (1-\epsilon) \gamma_{n(h)} \max_{s,\tau} |V^*(s,\tau)|,
\]
which implies
\[
\max_{s,\tau} |V^*(s,\tau)| \leq \frac{2M}{\epsilon}.
\]
\noindent~\cref{def:B.17} is thus proved.\qed
\begin{corollary}
\label{def:B.18}
Under the conditions of~\cref{def:B.17}, for any $e$ and any initial value $V_0$, performing successive iterations using the operator $\mathcal{F}^{e,\tau}_{n(h)}$ will result in convergence:
$
\lim_{n \to \infty} V_n = V,
$
where $V$ is the unique fixed point of $\mathcal{F}^{e,\tau}_{n(h)}$, i.e., the unified HRL value for the given execution tree $e$,
$
V_{e|n(h)}(\cdot,\cdot|\tau_{n(h)}) = V.
$
\end{corollary}
\textbf{Proof:}
From the proof of~\cref{def:B.17} we have shown that \(\mathcal{F}^{e,\tau}_{n(h)}\) is a contraction mapping; hence the successive (Gauss) iterations converge to the unique fixed point, and this fixed point satisfies the GHBE for executing \(\tau_{n(h)}\) with execution subtree \(e|n(h)\). By~\cref{def:B.13}, this fixed point is exactly \(V_{e|n(h)}(\cdot,\cdot \mid \tau_{n(h)})\).\qed
\subsubsection{Step 1: Hierarchical Policy Improvement Theorem}
In this section, we first establish the \emph{Hierarchical Policy Improvement Theorem} for multi-level HRL. We then prove that when a node policy is iteratively improved until convergence, the resulting solution satisfies the GHBOE of that node policy under a given interruption tree $\mathbf{B}$. Throughout this section, the discussion is based on a fixed task $\tau_{n(h')}$ with its associated rewards and boundary conditions. Hence, for notational simplicity, we omit the conditioning part in the HRL unified value function.
\begin{theorem}[\textbf{Hierarchical Policy Improvement Theorem}]
\label{def:B.19}
Given two execution trees $e$ and $e'$, suppose that the only difference between them lies in the node policy functions $\hat{\pi}_{n(h)}$ and $\hat{\pi}'_{n(h)}$ at state $s$.
If
$
\mathbb{E}_{ \tau_{n(h+1)} \sim \hat{\pi}_{n(h)}(s) }
V_{e|n(h')} (s, \tau_{n(h+1)})
\leq
\mathbb{E}_{ \tau_{n(h+1)} \sim \hat{\pi}'_{n(h)}(s) }
V_{e|n(h')} (s, \tau_{n(h+1)}),
$
then for all $s$, it holds that
$
V_{e|n(h')} (s, \tau_{n(h)})
\leq
V_{e'|n(h')} (s, \tau_{n(h)}),
$
where $h \geq h'$.
\end{theorem}
\textbf{Proof.}
We first show that
\[
\mathcal{F}^{e,\tau}_{n(h')} V_{e|n(h')} \;\;\leq\;\;
\mathcal{F}^{e',\tau}_{n(h')} V_{e|n(h')}.
\]
If $e_{n(\hat h)}, \hat h \geq h'$ does not belong to $e_{n(h)}$ itself nor its ancestors, then the execution subtrees $e|n(\hat h)$ and $e'|n(\hat h)$ are identical. Hence,
\[
\mathcal{F}^{e,\tau}_{n(h')} V_{e|n(h')}(s,\tau_{n(\hat h)}) \;=\;
\mathcal{F}^{e',\tau}_{n(h')} V_{e|n(h')}(s,\tau_{n(\hat h)}).
\]
Otherwise, we have
\[
\begin{aligned}
& \mathcal{F}^{e',\tau}_{n(h')} V_{e|n(h')}(s,\tau_{n(\hat h)}) -
\mathcal{F}^{e,\tau}_{n(h')} V_{e|n(h')}(s,\tau_{n(\hat h)}) \\
&= \mathbb{P}[n(\hat h) \rightarrow n(h)] \Big(
\mathbb{E}_{ \tau_{n(h+1)} \sim \hat{\pi}'_{n(h)}(s) }
V_{e|n(h')} (s, \tau_{n(h+1)})
\\&~~~~-
\mathbb{E}_{ \tau_{n(h+1)} \sim \hat{\pi}_{n(h)}(s) }
V_{e|n(h')} (s, \tau_{n(h+1)}) \Big) \\ &\geq 0,
\end{aligned}
\]
where $\mathbb{P}[n(\hat h) \rightarrow n(h)]$ denotes the probability of eventually selecting $\tau_{n(h)}$ when starting from $\hat{\pi}_{n(\hat h)}$ and making decisions down the subtree.
The first equality follows directly from the decision stage equation of the GHBOE.
The second inequality is guaranteed by the condition of~\cref{def:B.19}.
Therefore, we conclude that
\[
\mathcal{F}^{e,\tau}_{n(h')} V_{e|n(h')} \;\;\leq\;\;
\mathcal{F}^{e',\tau}_{n(h')} V_{e|n(h')}.
\]
Since
\[
V_{e|n(h')} = \mathcal{F}^{e,\tau}_{n(h')} V_{e|n(h')}
\;\;\leq\;\;
\mathcal{F}^{e',\tau}_{n(h')} V_{e|n(h')}
\;\;\leq\;\;
(\mathcal{F}^{e',\tau}_{n(h')})^k V_{e|n(h')}
\]
for all $k$, by taking the limit we obtain
\[
V_{e|n(h')} \;\;\leq\;\; \lim_{k \to \infty} (\mathcal{F}^{e',\tau}_{n(h')})^k V_{e|n(h')}
= V_{e'|n(h')}.
\]
This completes the proof of~\cref{def:B.19}.
\qed
\begin{corollary}[\textbf{SME Hierarchical Policy Improvement Theorem}]
\label{def:B.20}
Given two policy trees $\Pi$ and $\Pi'$, where the node policy functions $\hat{\pi}_{n(h)}$ and $\hat{\pi}'_{n(h)}$ differ only at state $s$, if
$
\mathbb{E}_{ \tau_{n(h+1)} \sim \hat{\pi}_{n(h)}(s) }
V^{SME}_{\Pi|n(h')} (s, \tau_{n(h+1)})
\leq
\mathbb{E}_{ \tau_{n(h+1)} \sim \hat{\pi}'_{n(h)}(s) }
V^{SME}_{\Pi|n(h')} (s, \tau_{n(h+1)}),
$
then for all states $s$, it holds that
$
V^{SME}_{\Pi|n(h')} (s, \tau_{n(h)})
\leq
V^{SME}_{\Pi'|n(h')} (s, \tau_{n(h)}).
$
\end{corollary}
\begin{corollary}[\textbf{ME Hierarchical Policy Improvement Theorem}]
\label{def:B.21}
Given two policy trees $\Pi$ and $\Pi'$, where the node policy functions $\pi_{n(h)}$ and $\pi'_{n(h)}$ differ only at state $s$, if
$
\mathbb{E}_{ \tau_{n(h+1)} \sim \pi_{n(h)}(s) }
V^{ME}_{\Pi|n(h')} (s, \tau_{n(h+1)})
\leq
\mathbb{E}_{ \tau_{n(h+1)} \sim \pi'_{n(h)}(s) }
V^{ME}_{\Pi|n(h')} (s, \tau_{n(h+1)}),
$
then for all states $s$, it holds that
$
V^{ME}_{\Pi|n(h')} (s, \tau_{n(h)})
\leq
V^{ME}_{\Pi'|n(h')} (s, \tau_{n(h)}).
$
\end{corollary}
Next, we prove the existence and uniqueness of the GHBOE solution for \(\hat\pi_{n(h)}\) (\textbf{hierarchical policy iteration}).
\paragraph{Existence of the GHBOE solution for \(\hat\pi_{n(h)}\).}
Given an arbitrary initial execution tree \(e^0 = e\), perform iterative updates that change only the node policy \(\hat\pi_{n(h)}\) each round. For the \(k\)-th iteration let the execution tree be \(e^k = \{\Pi^k,\mathbf{B}\}\), and define
\[
\hat\pi^k_{n(h)}(\tau_{n(h+1)}\mid s)
= \arg\max_{\tau_{n(h+1)}} V_{e^{k-1}|n(h')}(s,\tau_{n(h+1)}),
\]
where \(V_{e^{k-1}|n(h')}\) is the solution of the GHBE given \(\tau_{n(h')}\) and \(e^{k-1}|n(h')\).
By the Hierarchical Policy Improvement Theorem (~\cref{def:B.19}), the sequence \(V_{e^{k}|n(h')}\) is monotonically nondecreasing in \(k\). By~\cref{def:B.17}, \(\{V_{e^{k}|n(h')}\}_k\) is uniformly bounded for all \(k\). Hence the iterative process converges, and it converges to a unique limit. Denote
\[
V_{e^*|n(h')} \dot= \lim_{k\to\infty} V_{e^k|n(h')}.
\]
Let the execution tree derived from this limit be \(e^*\), where
\[
\hat\pi^*_{n(h)}(\tau_{n(h+1)}\mid s)
= \arg\max_{\tau_{n(h+1)}} V_{e^*|n(h')}(s,\tau_{n(h+1)}),
\]
and all other components of \(e^*\) coincide with those of \(e\). Since the iteration converges, the value evaluation of \(e^*\) is \(V_{e^*|n(h')}\), and further GHBOE updates on \(e^*\) leave it unchanged. It is straightforward to verify that \(V_{e^*|n(h')}\) satisfies the GHBOE for \(\hat\pi_{n(h)}\). Therefore a solution to the GHBOE for \(\hat\pi_{n(h)}\) exists. \qed
\paragraph{Uniqueness of the GHBOE solution for \(\hat\pi_{n(h)}\).}
Suppose \(V^*_{e|n(h')}\) is any solution in the solution set of the GHBOE for \(\hat\pi_{n(h)}\). Construct an execution tree \(\bar e^*\) from \(e\) and \(V^*_{e|n(h')}\) by setting
\[
\hat{\bar\pi}^*_{n(h)}(\tau_{n(h+1)}\mid s)
= \arg\max_{\tau_{n(h+1)}} V^*_{e|n(h')}(s,\tau_{n(h+1)}),
\]
while keeping the rest of \(\bar e^*\) identical to \(e\). If we take \(e^0=\bar e^*\) and run the same GHBOE iteration, then by construction the iterates satisfy \(V_{e^k}\equiv V^*_{e|n(h')}\) for all \(k\). But from the previous existence argument the iteration from any initial \(e^0\) converges to the unique limit \(V_{e^*|n(h')}\). Therefore \(V^*_{e|n(h')} = V_{e^*|n(h')}\). As this holds for any solution \(V^*_{e|n(h')}\), the GHBOE solution is unique. Consequently, the GHBOE for \(\hat\pi_{n(h)}\) has a unique solution \(V_{e^*|n(h')}\). \qed
As a special case of the GHBOE solution \(V_{e^*|n(h')}\), the solutions of HBOE-SMDP and HBOE-MDP also exist and are unique; we denote them by \(V^{SME}_{\Pi^-|n(h')}\) and \(V^{ME}_{\Pi^*|n(h')}\), respectively.
Therefore, given a policy tree \(\Pi\) with all node policies except \(\hat\pi_{n(h)}\) fixed, let \(\pi^0=\pi\). The solutions of the SMDP hierarchical policy iteration and the MDP hierarchical policy iteration satisfy
\[
V^{SME}_{\Pi^-|n(h')} \;\geq\; V^{SME}_{\Pi|n(h')},
\qquad
V^{ME}_{\Pi^*|n(h')} \;\geq\; V^{ME}_{\Pi|n(h')}.
\]
The corresponding improved node policies are defined by
\[
\begin{cases}
\displaystyle \hat\pi^-_{n(h)}(\tau_{n(h+1)}\mid s)
= \arg\max_{\tau_{n(h+1)}} V^{SME}_{\Pi^-|n(h')}(s,\tau_{n(h+1)}),\\[6pt]
\hat\pi^-_{n(\hat h)}(\tau_{n(h+1)}\mid s) = \hat\pi_{n(\hat h)}(\tau_{n(h+1)}\mid s),
\quad \forall\, n(\hat h)\neq n(h),
\end{cases}
\]
and
\[
\begin{cases}
\displaystyle \hat\pi^*_{n(h)}(\tau_{n(h+1)}\mid s)
= \arg\max_{\tau_{n(h+1)}} V^{ME}_{\Pi^*|n(h')}(s,\tau_{n(h+1)}),\\[6pt]
\hat\pi^*_{n(\hat h)}(\tau_{n(h+1)}\mid s) = \hat\pi_{n(\hat h)}(\tau_{n(h+1)}\mid s),
\quad \forall\, n(\hat h)\neq n(h).
\end{cases}
\]
Hence \(\hat\pi^-_{n(h)}\) and \(\hat\pi^*_{n(h)}\) are the SMDP and MDP optimal node policies for \(\tau_{n(h)}\), respectively.
Taking \(h' = h\), we further obtain the chain of inequalities on expected returns:
\[
\begin{aligned}
J_{\mathcal C}(\tau_{n(h)})
&= J^{SME}(\Pi^- \mid \tau_{n(h)})
= \mathbb{E}_{s\sim\alpha_{n(h)}}\big[ V^{SME}_{\Pi^-|n(h)}(s,\tau_{n(h)}\mid\tau_{n(h)})\big] \\[4pt]
&\ge
\mathbb{E}_{s\sim\alpha_{n(h)}}\big[ V^{SME}_{\Pi^*|n(h)}(s,\tau_{n(h)}\mid\tau_{n(h)})\big]
= J^{SME}(\Pi^* \mid \tau_{n(h)}) = J_{\mathcal D}(\tau_{n(h)}),
\\[6pt]
J_{\mathcal A}(\tau_{n(h)})
&= J^{ME}(\Pi^* \mid \tau_{n(h)})
= \mathbb{E}_{s\sim\alpha_{n(h)}}\big[ V^{ME}_{\Pi^*|n(h)}(s,\tau_{n(h)}\mid\tau_{n(h)})\big] \\[4pt]
&\ge
\mathbb{E}_{s\sim\alpha_{n(h)}}\big[ V^{ME}_{\Pi^-|n(h)}(s,\tau_{n(h)}\mid\tau_{n(h)})\big]
= J^{ME}(\Pi^- \mid \tau_{n(h)}) = J_{\mathcal B}(\tau_{n(h)}).
\end{aligned}
\]
Thus we have established the inequalities in~\cref{def:B.15}:
\[
J_{\mathcal A} \ge J_{\mathcal B}, \qquad J_{\mathcal C} \ge J_{\mathcal D}.
\]
\qed
\subsubsection{Step 2: Hierarchical Execution Improvement Theorem and Optimal Execution Theorem}
This section proves the Hierarchical Execution Improvement Theorem in multi-level HRL. Based on this theorem, we further establish the Optimal Execution Theorem under the Task Subtree Compatibility Condition (\cref{def:B.16}). All discussions in this section are based on a fixed task $\tau_{n(h')}$ with its corresponding reward and boundary conditions; therefore, for notational simplicity, we omit the conditioning on $\tau_{n(h')}$ in the HRL unified value function.
\begin{theorem}[\textbf{Hierarchical Execution Improvement Theorem}]
\label{def:B.22}
Given execution trees $e$ and $e'$, suppose that only the decision interruption functions $\hat\beta_{n(h)}$ and $\hat\beta'_{n(h)}$ differ at $s'$. Let $\hat\gamma$ and $\hat\gamma'$ be computed from $\hat\beta$ and $\hat\beta'$, respectively. If for all $n(l)$ it holds that
$
\sum^l_{k = h'} \hat\gamma_{n(k)} (s') V_{e|n(h')} (s', \tau_{n(k)}| \tau_{n(h')})
\leq
\sum^l_{k = h'} \hat\gamma'_{n(k)} (s') V_{e|n(h')} (s', \tau_{n(k)}| \tau_{n(h')}),
$
then for all $s$ we have
$
V_{e|n(h)} (s, \tau_{n(h)}) \leq V_{e'|n(h)} (s, \tau_{n(h)}),
$
where $l \geq h \geq h'$, and $l$ denotes the level of the task directly calling an action in the previous frame.
\end{theorem}
\textbf{Proof.}
First, we show that
\[
\mathcal{F}^{e,\tau}_{n(h')} V_{e|n(h')} \leq \mathcal{F}^{e',\tau}_{n(h')} V_{e|n(h')}.
\]
Since for all $\tau_{leaf}, n(l)$,
\[
\begin{aligned}
&\mathcal{F}^{e',\tau}_{n(h')} V_{e|n(h')}(s,\tau_{leaf}) - \mathcal{F}^{e,\tau}_{n(h')} V_{e|n(h')}(s,\tau_{leaf}) \\
&= T(s'|s,a) \gamma_{n(h')}(s') \left[ \sum^l_{k = h'} \hat\gamma'_{n(k)} (s') V_{e|n(h')} (s', \tau_{n(k)}) - \sum^l_{k = h'} \hat\gamma_{n(k)} (s') V_{e|n(h')} (s', \tau_{n(k)}) \right] \\
&\geq 0,
\end{aligned}
\]
by the same reasoning as in the proof of~\cref{def:B.19}, it follows that for all $\tau_{n(\hat h)}, \hat h \geq h'$,
\[
\mathcal{F}^{e',\tau}_{n(h')} V_{e|n(h')}(s,\tau_{n(\hat h)}) - \mathcal{F}^{e,\tau}_{n(h')} V_{e|n(h')}(s,\tau_{n(\hat h)}) \geq 0.
\]
Hence,
\[
\mathcal{F}^{e,\tau}_{n(h)} V_{e|n(h)} \leq \mathcal{F}^{e',\tau}_{n(h)} V_{e|n(h)}.
\]
Similarly, applying the contraction property and taking the limit gives
\[
V_{e|n(h')} \leq \lim_{k \to \infty} (\mathcal{F}^{e',\tau}_{n(h')})^k V_{e|n(h')} = V_{e'|n(h')}.
\]
This completes the proof of~\cref{def:B.22}. \qed
\begin{theorem}[\textbf{Optimal Execution Theorem}]
\label{def:B.23}
Under~\cref{def:B.16}, if the node policy $\hat\pi_{n(h)}$ is the optimal policy for the execution induced by any interruption subtree $\mathbf{B}|n(h)$, then
$
V^{\mathbf{B}|n(h)}_{\Pi|n(h)} \leq V^{ME}_{\Pi|n(h)},
$
i.e., the ME is the optimal execution method.
\end{theorem}
\textbf{Proof.}
for all $s, \hat h > h$ We have
\[
\begin{aligned}
\quad
V^{\mathbf{B}|n(h)}_{\Pi|n(h)}(s,\tau_{n(h)})
&= \max_{\tau_{{n}(h+1)} \in \mathcal{T}'_{\tau_{n(h)}}} V^{\mathbf{B}|n(h)}_{\Pi|n(h)}(s,\tau_{{n}(h+1)}) \\
&\geq \max_{\tau_{{n}(h+1)} \in \mathcal{T}'_{\tau_{n(\hat h)}}} V^{\mathbf{B}|n(h)}_{\Pi|n(h)}(s,\tau_{{n}(h+1)}) \\
&= \max_{\tau_{{n}(\hat h+1)} \in \mathcal{T}'_{\tau_{n(\hat h)}}} V^{\mathbf{B}|n(h)}_{\Pi|n(h)}(s,\tau_{{n}(\hat h+1)}) \\
&\geq \mathbb{E}_{\tau_{{n}(\hat h+1)} \sim \hat\pi_{n(\hat h)}}V^{\mathbf{B}|n(h)}_{\Pi|n(h)}(s,\tau_{{n}(\hat h+1)}) \\
&= V^{\mathbf{B}|n(h)}_{\Pi|n(h)}(s,\tau_{n(\hat h)}),
\end{aligned}
\]
where the first equality follows from the optimality of $\hat\pi_{n(h)}$, and the first inequality is guaranteed by~\cref{def:B.16}.
Thus, for all $n(l)$ we have
\[
\sum^l_{k = h} \hat\gamma_{n(k)}(s') V^{\mathbf{B}|n(h)}_{\Pi|n(h)}(s', \tau_{n(k)}|\tau_{n(h)}) \leq V^{\mathbf{B}|n(h)}_{\Pi|n(h)}(s', \tau_{n(h)}|\tau_{n(h)}).
\]
By~\cref{def:B.22}, it follows that for all $s$
\[
\quad V^{\mathbf{B}|n(h)}_{\Pi|n(h)}(s, \tau_{{n}(h)}) \leq V^{ME}_{\Pi|n(h)}(s, \tau_{{n}(h)}).
\]
Hence,~\cref{def:B.23} is proved. \qed
We also have
\[
\begin{aligned}
J_{\mathcal{B}}(\tau_{n(h)})
&= J^{ME}(\Pi^-|\tau_{n(h)}) \\
&= \mathbb{E}_{s \sim \alpha_{n(h)}} \left[ V^{ME}_{\Pi^-|n(h)} (s, \tau_{n(h)}| \tau_{n(h)}) \right] \\
&\geq \mathbb{E}_{s \sim \alpha_{n(h)}} \left[ V^{SME}_{\Pi^-|n(h)} (s, \tau_{n(h)}| \tau_{n(h)}) \right] \\
&= J^{SME}(\Pi^-|\tau_{n(h)}) \\
&= J_{\mathcal{C}}(\tau_{n(h)}),
\end{aligned}
\]
which proves that $J_{\mathcal{B}} \geq J_{\mathcal{C}}$ in~\cref{def:B.15}.  \qed
Hence,~\cref{def:B.15} is fully proved. \qed
~\cref{def:B.16} is crucial. In multi-level HRL, if~\cref{def:B.16} is not satisfied, there exist \textbf{counterexamples} where
$
J_{\mathcal{B}} < J_{\mathcal{C}}
$.
\begin{figure}[t]
\centering
\includegraphics[width=0.95\textwidth]{
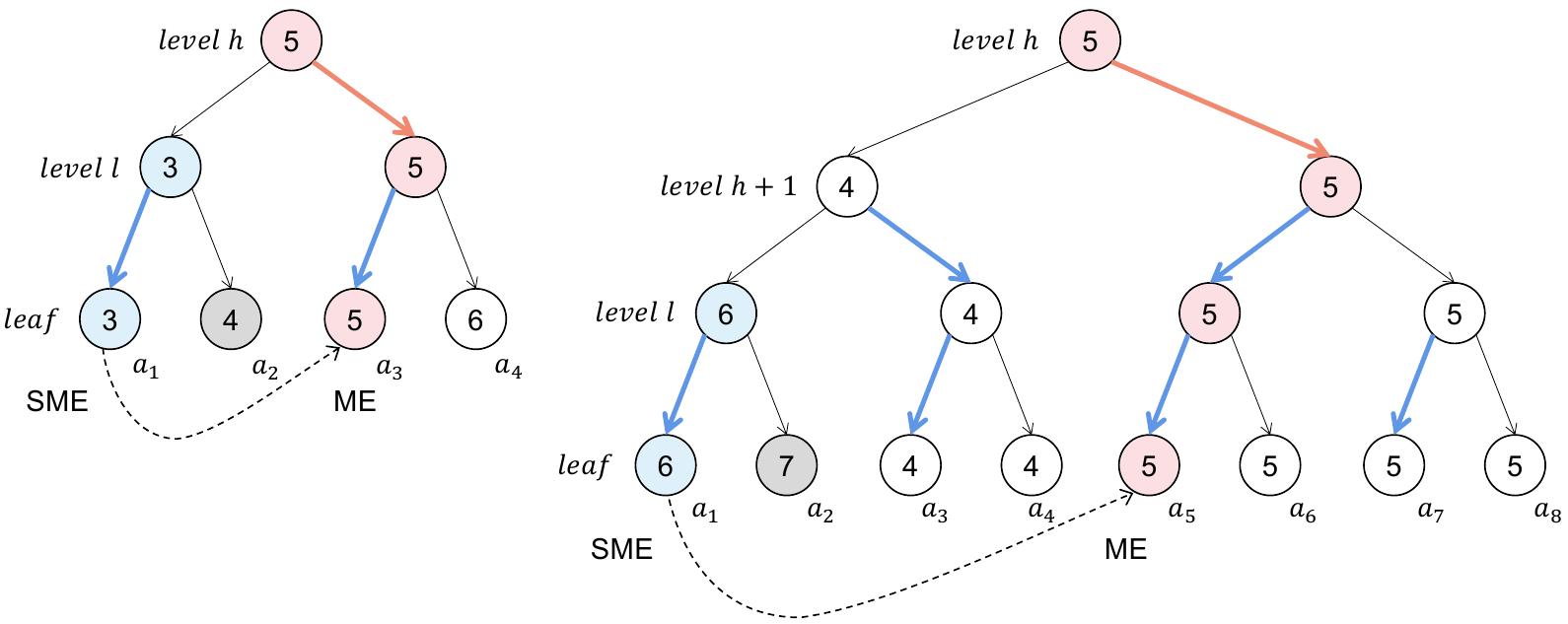
}
\caption{Counterexample: The red and blue colored arrows represent the SMDP-optimal policies for each task node, while the numbers inside the circles indicate the state-task values at the current frame when SME executes $\tau_{n(h)}$. In this example, no task termination events occur. Gray nodes denote the actions taken by the agent in the previous frame; the blue nodes form the decision path of the current frame if SME is used, and the red nodes form the decision path of the current frame if ME is used.
}
\label{fig:Counterexample}
\end{figure}
We can observe that in the left figure, the values of the red nodes are higher than those of the blue nodes, indicating that using ME to improve the execution method can yield a value increase. However, in the right figure, the values of the red nodes are lower than those of the blue nodes, so applying ME for execution improvement fails. This phenomenon arises because subtasks have their own independent reward accumulation mechanisms, and the decisions made by the subtask's optimal node policy can be suboptimal for the current task $\tau_{n(h)}$.
We note that the node with value 6 at level-$l$ is not selected by the optimal policy of its parent node; it is precisely such special parent tasks that conflict with $\tau_{n(h)}$ that produce this counterexample. By adjusting the decision path, this situation can be avoided, and~\cref{def:B.16} serves as a sufficient condition.
\section{Framework improvement and algorithmic design principles}
\label{sec:C}
\begin{table}[h]
\centering
\renewcommand{\arraystretch}{1.4}
\setlength{\tabcolsep}{4pt}
\begin{tabular}{p{1.2cm}|p{1.6cm}|p{2cm}|p{2.5cm}|p{1.5cm}|p{0.9cm}|p{2cm}}
\toprule
\textbf{Frame-\newline work} &
\textbf{Execution Mode (Deployment)} &
\textbf{Execution Mode of Behavior Policy (Training)} &
\textbf{Execution Mode of Target Policy (Training)} &
\textbf{Training Requirement} &
\textbf{$J$} &
\textbf{Limitation} \\
\midrule
HRL Baseline & SME & All node policies called by SME & All node policies called by SME & On-policy & $J_{\mathcal{C}}$ & Execution and policy suboptimality exist \\
\hline
ESIF & ME & All node policies called by SME & All node policies called by SME & On-policy & $J_{\mathcal{B}}$ & Policy suboptimality exists \\
\hline
PSIF-2S & ME & Each node calls children by SME in its Stage 1, by ME in its Stage 2 & Same node-local execution modes as behavior; completed nodes retain ME & On-policy & $J_{\mathcal{A}}$ & Subtask support problem exists \\
\hline
PSIF-1S & ME & All node policies called by SME & For each optimi-\newline zed node policy, its descendants\newline called by ME,\newline others by SME & Off-policy & $J_{\mathcal{A}}$ & Off-policy issue exists \\
\bottomrule
\end{tabular}
\caption{Comparison of HRL Improvement Frameworks}
\label{tab:Comparison}
\end{table}
The comparison of the three improved frameworks is presented in~\cref{tab:Comparison}.
\subsection{Improved framework pseudocode}
\subsubsection{ESIF pseudocode}
\begin{algorithm}[H]
\caption{ESIF Deployment Pseudocode}
\begin{algorithmic}[1]
\STATE \textbf{Deployment:} Deploy subtrees of $\Pi^-$ rooted at $\hat\pi^-_{n(h)}$ using ME
\STATE Initialize $s_0$, current task $\bar\tau = \tau_{n(h)}$, decision authority task $\tau_0 = \tau_{n(h)}$
\FOR{frame $t = 0,1,...$ until $\bar\tau = \tau_{out}$}
    \STATE $\tau = \tau_t$
    \WHILE{$\tau \notin \{\tau_{leaf}\}$}
        \STATE $\tau = \pi_\tau(s_t)$
        \IF{$\tau \in \{\tau_{leaf}\}$}
            \STATE $a_t = \tau$
        \ENDIF
    \ENDWHILE
    \STATE Execute action $a_t$ and observe next state $s_{t+1}$
    \STATE Sample task exit $\bar\tau = \tau_{n(j)}$ according to $\gamma$
    \STATE \textbf{Return decision authority to root:} $\tau_t = \tau_{n(h)}$
\ENDFOR
\end{algorithmic}
\end{algorithm}
\clearpage
\begin{algorithm}[!h]
\caption{ESIF Training Pseudocode}
\begin{algorithmic}[1]
\REQUIRE HRL algorithm $L$, MDP environment $E$
\STATE \textbf{Training:} Interact with $E$ using the original SME execution mode of $L$ and learn $\Pi^-$
\STATE Initialize policy tree $\Pi$
\COMMENT{Initialize each node policy $\hat\pi_{n(h)}, n(h) \in N$ and all state-task values $V(s,\tau_{n(i)}|\tau_{n(h')})$ for all $s \in \mathcal{S}, n(i) \in N$. $\tau_{n(h')}$ is the task to be optimized by $\hat\pi_{n(h)}$}
\FOR{episode $m = 0,1,...$}
    \STATE Initialize state $s_0$, current task $\bar\tau = \tau_{n(0)}$, decision authority task $\tau_0 = \tau_{n(0)}$, decision path $L = \{\tau_{n(0)}\}$
    \FOR{frame $t = 0,1,...$ until $\tau_t = \tau_{out}$}
        \STATE \COMMENT{Decision phase}
        \STATE $\tau = \tau_t$
        \STATE Retain in $L$ only the root-to-$\tau$ prefix of the current decision path
        \WHILE{$\tau \notin \{\tau_{leaf}\}$}
            \STATE $\tau = \hat\pi_\tau(s_t)$
            \IF{$\tau \in \{\tau_{leaf}\}$}
                \STATE $a_t = \tau$
            \ELSE
                \STATE Add the selected task $\tau$ to $L$
            \ENDIF
        \ENDWHILE
        \STATE Execute action $a_t$ and observe next state $s_{t+1}$
        \STATE \COMMENT{Interaction phase}
        \FOR{task $\tau_{n(i)} \in L$}
            \STATE $r_{t+1}(\tau_{n(i)}) = R_{n(i)}(s_t, a_t)$
        \ENDFOR
        \STATE Sample task exit $\bar\tau = \tau_{n(j)}$ according to $\gamma$ \COMMENT{Termination phase}
        \STATE Return decision authority via SME: $\tau_{t+1} = \bar\tau$ \COMMENT{Preparation phase}
        \FOR{task $\tau_{n(i)} \in L$}
            \STATE \COMMENT{Data collection}
            \STATE buffer($n(i)$) $\gets \{s_t, L, a_t, s_{t+1}, r_{t+1} | \tau_{n(i)}\}$
        \ENDFOR
    \ENDFOR
\ENDFOR
\STATE \textbf{Learning:}
\FOR{node $n(h) \in N$}
    \FOR{transition $\{s_t, L, a_t, s_{t+1}, r_{t+1} | \tau_{n(h')}\} \in \text{buffer}(n(h'))$}
        \STATE Resample task exit along $L$ according to current $\gamma$ to $\tau_{n(j)}$
        \STATE Update state-task value via HBE-SMDP:
        \[
        V(s_t, a_t | \tau_{n(h')}) \gets
        r_{t+1} + \mathbb{I}_{-1 \le j < h'} \hat R_{n(h')}(s_t, \tau_{n(j)})
        + \mathbb{I}_{h' \le j} V(s_t, \tau_{n(j)} | \tau_{n(h')})
        \]
        \FOR{task $\tau_{n(i)} \in L$}
            \STATE $V(s_t, \tau_{n(i)} | \tau_{n(h')}) \gets \mathbb{E}_{\tau_{n(i+1)} \sim \hat\pi_{n(i)}(s_t)} [V(s_t, \tau_{n(i+1)} | \tau_{n(h')})]$
        \ENDFOR
    \ENDFOR
    \STATE \textbf{Policy update:}
    \STATE $\hat\pi_{n(h)} \gets \arg\max_{\tau_{n(h+1)} \in \mathcal{T}'_{n(h)}} V(\cdot, \tau_{n(h+1)} | \tau_{n(h')})$
\ENDFOR
\STATE \textbf{Output:} SMDP-optimal policy tree $\Pi^-$ (each node policy induced by state-task values)
\end{algorithmic}
\end{algorithm}
\subsubsection{PSIF-2S pseudocode}
\clearpage
\begin{longalgorithm}
\caption{PSIF-2S Pseudocode}
\begin{algorithmic}[1] \small
\REQUIRE HRL algorithm $L$, MDP environment $E$
\STATE \textbf{Training:} Learn $\Pi^*$ through node-local Stage 1, Stage 2, and completion
\STATE \textbf{Interaction:} Stage-1 nodes call children by SME; Stage-2 and completed nodes call children by ME
\STATE \textbf{Initialize policy tree} $\Pi$
\COMMENT{Initialize each node policy $\hat\pi_{n(h)}, n(h) \in N$ and all state-task values $V(s,\tau_{n(i)}|\tau_{n(h')})$ for all $s \in \mathcal{S}, n(i) \in N$. $\tau_{n(h')}$ is the task to be optimized by $\hat\pi_{n(h)}$}
\STATE Initialize lowest task nodes in Stage 2, other task nodes in Stage 1; primitive actions are fixed
\FOR{training rounds until every task node is completed}
\FOR{each collection episode in this round}
    \STATE Initialize state $s_0$, current task $\bar\tau = \tau_{n(0)}$, decision authority task $\tau_0 = \tau_{n(0)}$, decision path $L = \{\tau_{n(0)}\}$
    \FOR{frame $t = 0,1,...$ until $\tau_t = \tau_{out}$}
        \STATE \COMMENT{Decision phase}
        \STATE $\tau = \tau_t$
        \STATE Retain in $L$ only the root-to-$\tau$ prefix of the current decision path
        \WHILE{$\tau \notin \{\tau_{leaf}\}$}
            \STATE $\tau = \hat\pi_\tau(s_t)$
            \IF{$\tau \in \{\tau_{leaf}\}$}
                \STATE $a_t = \tau$
            \ELSE
                \STATE Add the selected task $\tau$ to $L$
            \ENDIF
        \ENDWHILE
        \STATE Execute action $a_t$ and observe next state $s_{t+1}$ \COMMENT{Interaction phase}
        \FOR{task $\tau_{n(i)} \in L$}
            \STATE $r_{t+1}(\tau_{n(i)}) = R_{n(i)}(s_t, a_t)$
        \ENDFOR
        \STATE Sample task exit $\bar\tau = \tau_{n(j)}$ according to $\gamma$\COMMENT{Termination phase}
        \STATE Return decision authority: $\tau_{t+1} = \bar\tau$
        \STATE \COMMENT{Preparation phase: a caller in Stage 2 or completed reselects its children each frame}
        \WHILE{$\tau_{t+1}\notin\{\tau_{n(0)},\tau_{out}\}$ and its parent is Stage 2 or completed}
            \STATE $\tau_{t+1} = \tau_{t+1}.\text{father}$
        \ENDWHILE
        \FOR{task $\tau_{n(i)} \in L$}
            \STATE \COMMENT{Data collection}
            \STATE buffer($n(i)$) $\gets \{s_t, L, a_t, s_{t+1}, r_{t+1} | \tau_{n(i)}\}$
        \ENDFOR
    \ENDFOR
\ENDFOR
\STATE \textbf{Learning:}
\FOR{unfinished task node $n(h) \in N$}
    \FOR{current-mode transition $\{s_t, L, a_t, s_{t+1}, r_{t+1} | \tau_{n(h')}\} \in \text{buffer}(n(h'))$}
        \STATE Resample task exit along $L$ according to $\gamma$ to $\tau_{n(j)}$
        \STATE Set decision authority $\tau_{n(k)}=\tau_{n(j)}$, keeping task-exit index $j$ unchanged
        \WHILE{$\tau_{n(k)}\notin\{\tau_{n(0)},\tau_{out}\}$ and its parent is Stage 2 or completed}
            \STATE $\tau_{n(k)} = \tau_{n(k)}.\text{father}$
        \ENDWHILE \COMMENT{The target policy and the behavior policy use \textbf{the same execution mode.}}
        \STATE Update state-task value via GHBE:
        \[
        V(s_t, a_t | \tau_{n(h')}) \gets
        r_{t+1} + \mathbb{I}_{-1 \le j < h'} \hat R_{n(h')}(s_{t+1}, \tau_{n(j)})
        + \mathbb{I}_{h' \le j} V(s_{t+1}, \tau_{n(k)} | \tau_{n(h')})
        \]
        \FOR{task $\tau_{n(i)} \in L$}
            \STATE $V(s_t, \tau_{n(i)} | \tau_{n(h')}) \gets \mathbb{E}_{\tau_{n(i+1)} \sim \hat\pi_{n(i)}(s_t)} [V(s_t, \tau_{n(i+1)} | \tau_{n(h')})]$
        \ENDFOR
    \ENDFOR
    \STATE Improve $\hat\pi_{n(h)}$ from its current-mode values, as in ESIF
    \IF{$\tau_{n(h)}$ is Stage 2 and $\hat\pi_{n(h)}\approx\hat\pi^*_{n(h)}$}
        \STATE Mark $\tau_{n(h)}$ completed and fix its node policy
    \ENDIF
\ENDFOR
\FOR{Stage-1 task whose child task policies are all completed and fixed}
    \STATE Move this task to Stage 2 for the next collection round
\ENDFOR
\ENDFOR
\STATE \textbf{Output:} MDP-optimal policy tree $\Pi^*$ (each node policy induced by state-task values)
\STATE \textbf{Deployment:} same as ESIF
\end{algorithmic}
\end{longalgorithm}
Stage variables persist across episodes. A node begins Stage 2 after its children have completed it; its own policy is fixed only after its own Stage-2 optimization finishes. Thus behavior and target execution change from SME to ME from the bottom upwards. The task-exit index $j$ determines reward termination, whereas $k$ determines the next decision authority; changing interruption must not change the sampled task-exit boundary.
\subsubsection{PSIF-1S pseudocode}
\begin{algorithm}[H]
\caption{PSIF-1S pseudocode}
\begin{algorithmic}[1]
\REQUIRE HRL algorithm $L$, MDP environment $E$
\STATE \textbf{Training:} Interact with $E$ using the original SME execution mode of $L$ and learn $\Pi^*$
\STATE \textbf{Interaction:} Same as ESIF
\COMMENT{Behavior policy execution: SME}
\STATE \textbf{Learning:}
\COMMENT{Value update via GHBE, and target policy execution: where $\hat\pi_{n(h)}$'s descendant tasks are executed via ME, others via SME}
\FOR{node $n(h) \in N$}
    \FOR{transition $\{s_t, L, a_t, s_{t+1}, r_{t+1} | \tau_{n(h')}\} \in \text{buffer}(n(h'))$}
        \STATE Resample task exit along $L$ according to $\gamma$ to $\tau_{n(j)}$
        \STATE Set decision authority $\tau_{n(k)}=\tau_{n(j)}$, keeping task-exit index $j$ unchanged
        \IF{task $\tau_{n(k)}$ is a descendant of $\tau_{n(h)}$}
            \STATE $\tau_{n(k)} = \tau_{n(h)}$
        \ENDIF \COMMENT{Inside the optimized subtree, return authority to its root; outside, retain SME}
        \STATE Update state-task value via GHBE:
        \[
        V(s_t, a_t | \tau_{n(h')}) \gets
        r_{t+1} + \mathbb{I}_{-1 \le j < h'} \hat R_{n(h')}(s_{t+1}, \tau_{n(j)})
        + \mathbb{I}_{h' \le j} V(s_{t+1}, \tau_{n(k)} | \tau_{n(h')})
        \]
        \FOR{task $\tau_{n(i)} \in L$}
            \STATE $V(s_t, \tau_{n(i)} | \tau_{n(h')}) \gets \mathbb{E}_{\tau_{n(i+1)} \sim \hat\pi_{n(i)}(s_t)} [V(s_t, \tau_{n(i+1)} | \tau_{n(h')})]$
        \ENDFOR
    \ENDFOR
\ENDFOR
\STATE \textbf{Policy update:} Same as ESIF
\STATE \textbf{Output:} MDP-optimal policy tree $\Pi^*$ (each node policy induced by state-task values)
\STATE \textbf{Deployment:} Same as ESIF
\end{algorithmic}
\end{algorithm}
The behavior policy uses SME across the entire hierarchy. For each optimized node, the target uses ME within that node's subtree and SME outside it. The sampled task-exit boundary remains indexed by $j$; only the next decision authority changes to $k$.
This design helps mitigate the off-policy issue,
as node policies outside the subtree only influence the initial state distribution of the task,
which does not affect the optimality of the target policy optimization.
\subsection{Example of algorithmic improvement}
\subsubsection{Option framework algorithm}
~\cref{sec:B.1} presents the task tree setup under the options framework, where $\tau_{n(0)}$ is the root task, $\tau_{n(1)}$ is an option, $\tau_{leaf}$ is an action, and $R(s,a)$ is the per-step reward provided by the environment. Accordingly, the original algorithm corresponds to the HBOE-SMDP as follows:
For the root task $\hat\pi_{n(0)}$:
\[
V(s,\tau_{n(0)} |\tau_{n(0)}) = \mathbb{E}_{ \tau_{n(1)} \sim \hat\pi_{n(0)}(s) }[V(s,\tau_{n(1)} | \tau_{n(0)})],
\]
\[
\hat\pi_{n(0)}(s) = \argmax_{\tau_{n(1)}} V(s,\tau_{n(1)} |\tau_{n(0)}),
\]
\[
V(s,\tau_{n(1)} | \tau_{n(0)}) = \mathbb{E}_{ \tau_{leaf} \sim \hat\pi_{n(1)}(s) }[V(s,\tau_{leaf} | \tau_{n(0)})],
\]
\[
V(s,\tau_{leaf} | \tau_{{n}(0)}) = R(s,a)
+ \mathbb{E}_{ s' \sim T(s,a)} \Big[ \hat\gamma_{n(1)}(s') V(s',\tau_{n(1)} |  \tau_{n(0)})
+ \hat\gamma_{n(0)}(s') V(s',\tau_{n(0)} | \tau_{n(0)})  \Big].
\]
For the option $\hat\pi_{n(1)}$:
\[
V(s,\tau_{n(0)} |\tau_{n(1)}) = \mathbb{E}_{ \tau'_{n(1)} \sim \hat\pi_{n(0)}(s) }[V(s,\tau'_{n(1)} | \tau_{n(1)})],
\]
\[
V(s,\tau_{n(1)} | \tau_{n(1)}) = \mathbb{E}_{ \tau_{leaf} \sim \hat\pi_{n(1)}(s) }[V(s,\tau_{leaf} | \tau_{n(1)})],
\]
\[
\hat\pi_{n(1)}(s) = \argmax_{\tau_{leaf}} V(s,\tau_{leaf} |\tau_{n(1)}),
\]
\[
V(s,\tau_{leaf} | \tau_{{n}(1)}) = R(s,a)
+ \mathbb{E}_{ s' \sim T(s,a)} \Big[ \hat\gamma_{n(1)}(s') V(s',\tau_{n(1)} |  \tau_{n(1)})
+ \gamma_{n(0)}(s') V(s',\tau_{n(0)} | \tau_{n(1)})  \Big].
\]
It can be observed that during value estimation, because $\hat\gamma = \gamma$, the two sets of value operators between the root task and the option are identical. Therefore, a single value function can be shared between them.
In the PSIF-2S framework, once all options have completed their second stage, their node policies and termination functions are fixed. The root then enters its second stage and calls these options using ME. The GHBOE at this stage is as follows:
For the root task $\hat\pi_{n(0)}$:
\[
V(s,\tau_{n(0)} |\tau_{n(0)}) = \mathbb{E}_{ \tau_{n(1)} \sim \hat\pi_{n(0)}(s) }[V(s,\tau_{n(1)} | \tau_{n(0)})],
\]
\[
\hat\pi_{n(0)}(s) = \argmax_{\tau_{n(1)}} V(s,\tau_{n(1)} |\tau_{n(0)}),
\]
\[
V(s,\tau_{n(1)} | \tau_{n(0)}) = \mathbb{E}_{ \tau_{leaf} \sim \hat\pi_{n(1)}(s) }[V(s,\tau_{leaf} | \tau_{n(0)})],
\]
\[
V(s,\tau_{leaf} | \tau_{{n}(0)}) = R(s,a)
+ \mathbb{E}_{ s' \sim T(s,a)} \left[ \hat\gamma_{n(0)}(s') V(s',\tau_{n(0)} | \tau_{n(0)}) \right].
\]
Through iterative updates, $\hat\pi_{n(0)}$ will gradually converge from $\hat\pi^-_{n(0)}$ to $\hat\pi^*_{n(0)}$.
In the PSIF-1S framework, unlike the original options framework, the target execution processes evaluated by the root and the options differ. Therefore, separate upper and lower values are maintained. The option-exit reward uses the upper ME value, as in Equation~\eqref{new:eq:oc_offtrace}. The corresponding GHBOE is as follows:
For the root task $\hat\pi_{n(0)}$:
\[
V(s,\tau_{n(0)} |\tau_{n(0)}) = \mathbb{E}_{ \tau_{n(1)} \sim \hat\pi_{n(0)}(s) }[V(s,\tau_{n(1)} | \tau_{n(0)})],
\]
\[
\hat\pi_{n(0)}(s) = \argmax_{\tau_{n(1)}} V(s,\tau_{n(1)} |\tau_{n(0)}),
\]
\[
V(s,\tau_{n(1)} | \tau_{n(0)}) = \mathbb{E}_{ \tau_{leaf} \sim \hat\pi_{n(1)}(s) }[V(s,\tau_{leaf} | \tau_{n(0)})],
\]
\[
V(s,\tau_{leaf} | \tau_{{n}(0)}) = R(s,a)
+ \mathbb{E}_{ s' \sim T(s,a)} \left[ \hat\gamma_{n(0)}(s') V(s',\tau_{n(0)} | \tau_{n(0)}) \right].
\]
For the option $\hat\pi_{n(1)}$, the exit reward at the root boundary is
\[
\hat R_{n(1)}(s,\tau_{n(0)})=U(s)=\max_{\tau'_{n(1)}} V(s,\tau'_{n(1)} | \tau_{n(0)}).
\]
The option's values then satisfy
\[
V(s,\tau_{n(1)} | \tau_{n(1)}) = \mathbb{E}_{ \tau_{leaf} \sim \hat\pi_{n(1)}(s) }[V(s,\tau_{leaf} | \tau_{n(1)})],
\]
\[
\hat\pi_{n(1)}(s) = \argmax_{\tau_{leaf}} V(s,\tau_{leaf} |\tau_{n(1)}),
\]
\[
\begin{aligned}
V(s,\tau_{leaf} | \tau_{{n}(1)}) ={}& R(s,a)\\
&+ \mathbb{E}_{ s' \sim T(s,a)} \Big[ \hat\gamma_{n(1)}(s') V(s',\tau_{n(1)} |  \tau_{n(1)})
+ \gamma_{n(0)}(s')\hat R_{n(1)}(s',\tau_{n(0)}) \Big].
\end{aligned}
\]
This corresponds to the OC-PSIF-1S algorithm used here. $V(\cdot,\cdot|\tau_{n(0)})$ corresponds to $V^{hi}$, while $V(\cdot,\cdot|\tau_{n(1)})$ corresponds to $V^{lo}$. While the option continues, its own lower value supplies the bootstrap; when it terminates, the upper ME value supplies the exit reward. Data collection still follows SME. The single-frame target uses the observed next state and can be evaluated without sampling the upper target policy's full trajectory. It does not remove state--option distribution shift during learning.
$V(s,\tau_{n(1)} |\tau_{n(0)})$ induces the greedy root node policy, while $V(s,\tau_{n(1)} |\tau_{n(1)})$ supplies the option-policy baseline and the option termination update. The termination baseline is the maximum lower value, distinct from the upper-ME option-exit reward. Thus the framework updates are:
\[
\Delta \vartheta \sim -\alpha_{\vartheta} \frac{\partial \beta_{\omega}(s)}{\partial \vartheta}
\left( V^{lo}(s,\omega) - \max_{\omega'}V^{lo}(s,\omega') \right),
\]
\[
\Delta \theta \sim \alpha_{\theta} \frac{\partial \pi_{\omega}(a|s)}{\partial \theta}
\left( V(s,a |\tau_{n(1)}) - V(s,\tau_{n(1)} |\tau_{n(1)}) \right).
\]
The Fourrooms implementation uses an upper-value termination advantage instead; Appendix~\ref{new:app:oc_updates} describes this implementation variant.
\subsubsection{Goal-Based framework algorithm}
\cref{sec:B.1} presents the task tree setup for the target framework, where $\tau_{n(0)}$ belongs to the target set $G$, $\tau_{n(1)}$ belongs to the subgoal set $g$, and $\tau_{leaf}$ represents actions. The original algorithm corresponds to the HBOE-SMDP as follows:
For the root task $\hat\pi_{n(0)}$:
\[
V(s,\tau_{n(0)} |\tau_{n(0)}) = \mathbb{E}_{ \tau_{leaf} \sim \hat\pi_{n(0)}(s) } [V(s,\tau_{leaf} | \tau_{n(0)})],
\]
\[
\hat\pi_{n(0)}(s) = \argmax_{\tau_{n(1)}} V(s,\tau_{n(1)} |\tau_{n(0)}),
\]
\[
V(s,\tau_{n(1)} | \tau_{n(0)}) = \mathbb{E}_{ \tau_{leaf} \sim \hat\pi_{n(1)}(s) } [V(s,\tau_{leaf} | \tau_{n(0)})],
\]
\[
\begin{aligned}
V(s,\tau_{leaf} | \tau_{{n}(0)}) ={}& R_{n(0)}(s,a)\\
&+ \mathbb{E}_{ s' \sim T(s,a)} \Big[ \hat\gamma_{n(1)}(s') V(s',\tau_{n(1)} |  \tau_{n(0)})
+ \hat\gamma_{n(0)}(s') V(s',\tau_{n(0)} | \tau_{n(0)}) \Big].
\end{aligned}
\]
For the option $\hat\pi_{n(1)}$:
\[
V(s,\tau_{n(0)} |\tau_{n(1)}) = \mathbb{E}_{ \tau_{leaf} \sim \hat\pi_{n(0)}(s) } [V(s,\tau_{leaf} | \tau_{n(1)})]
,\footnote{this computation is irrelevant for policy learning and can be omitted in practice.},
\]
\[
V(s,\tau_{n(1)} | \tau_{n(1)}) = \mathbb{E}_{ \tau_{leaf} \sim \hat\pi_{n(1)}(s) } [V(s,\tau_{leaf} | \tau_{n(1)})],
\]
\[
\hat\pi_{n(1)}(s) = \argmax_{\tau_{n(1)}} V(s,\tau_{n(1)} |\tau_{n(1)}),
\]
\[
V(s,\tau_{leaf} | \tau_{{n}(1)}) = R_{n(1)}(s,a)
+ \mathbb{E}_{ s' \sim T(s,a)} \left[ \hat\gamma_{n(1)}(s') V(s',\tau_{n(1)} |  \tau_{n(1)}) \right].
\]
In some target-framework algorithms, the execution of a subgoal lasts for a fixed number of frames $c$, while the upper-level value estimation only accounts for the value of the subgoal when it has been fully executed. It does not estimate the value of the remaining subtask transitions during execution. Therefore, the ESIF framework cannot be directly applied. To address this, the subgoal $g$ can be extended to a subgoal--elapsed-time pair $(g, \hat t)$, with $0 \leq \hat t \leq c$.
\section{Experiment details and analysis}
\label{sec:D}
\subsection{Environment settings}
\paragraph{Randomized version of AntReacher.} A continuous task simulated in MuJoCo aims at controlling an ant robot to achieve a random goal on the map. When deploying a policy, an additional random transition event is engaged: For every $10$ timestep, the ant will have a random state transition on the whole map with the probability of $10\%$, only involving coordinates $x$ and $y$.
\par
\paragraph{Fourrooms.} A GridWorld environment of four rooms. The cells of the grid correspond to the states of the environment. The agent must start from a random starting cell and reach a fixed-end cell. In any state, the agent is feasible to perform one of four actions, $up, down, left, right$. And there is a random transition after the action is selected. With the probability of $\frac{2}{3}$, the agent moves one cell in the chosen direction. With the probability of $\frac{1}{3}$, the agent moves one cell in one of the other three directions, each with the chance of $\frac{1}{9}$. In either case, if the movement takes the agent onto a wall, the agent remains in the same cell.
\subsection{Algorithms used}
\paragraph{HAC.} This is a state-of-the-art goal-based HRL algorithm. For the goal-based HRL algorithms, the higher level of the hierarchy interacts with an SMDP induced by the environment dynamics and the behavior of the lower level. During training, the lower-level policy constantly changes with updates, making the higher-level interaction non-stationary. HAC proposed the hindsight action relabeling method to alleviate this problem for static environments by `pretending' that the lower-level policy has already been optimal. In addition, HAC also proposed hindsight goal transitions and subgoal testing to accelerate the learning of sparse-reward tasks.
\paragraph{OC.} This is a classic option-based HRL algorithm. This algorithm learns the entire hierarchical policy tree in an end-to-end manner. The higher-level policy and all the options (i.e., intra-option policies and termination functions) are trained using policy gradient derived using the root task rewards (i.e., the original environmental reward). A limitation of OC is that it requires predefining the number of options. The policy gradients are heavily dependent on the root task reward. Thus, OC is likely to perform poorly on the sparse-reward tasks.
\begin{figure}[htbp]
\centering
\includegraphics[width=\textwidth]{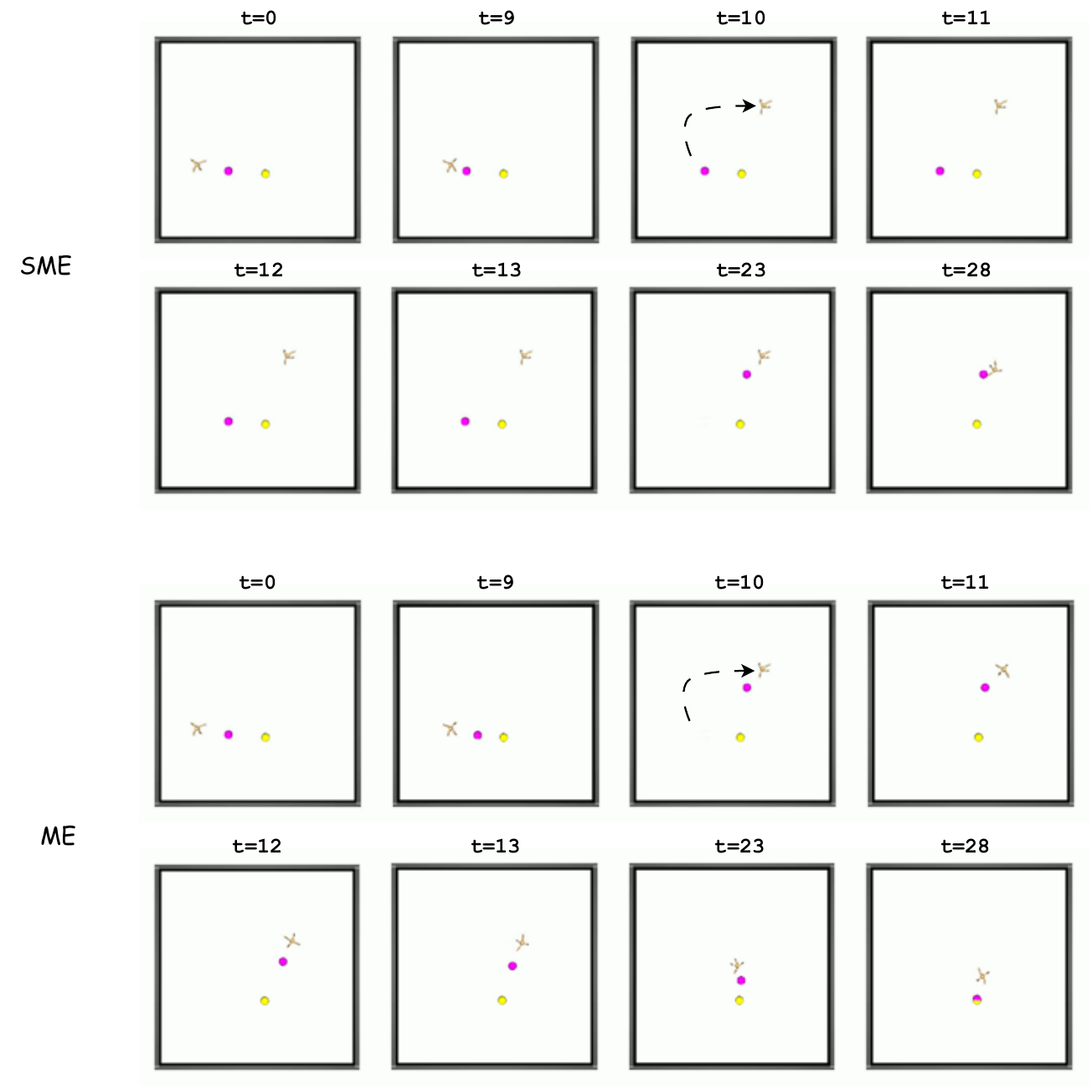}
\caption{The ant robot navigates from a random starting position to a random ending position in the AntReacher environment. The yellow sphere on the map refers to the ending position, the red sphere refers to the subgoal proposed by the agent's higher-level policy, and the lower-level policy directly controls the ant's joints for moving towards this subgoal.}
\label{fig:visual}
\end{figure}
\begin{figure}[ht]
\centering
\begin{subfigure}{1\linewidth}
\centering
\includegraphics[scale=0.4]{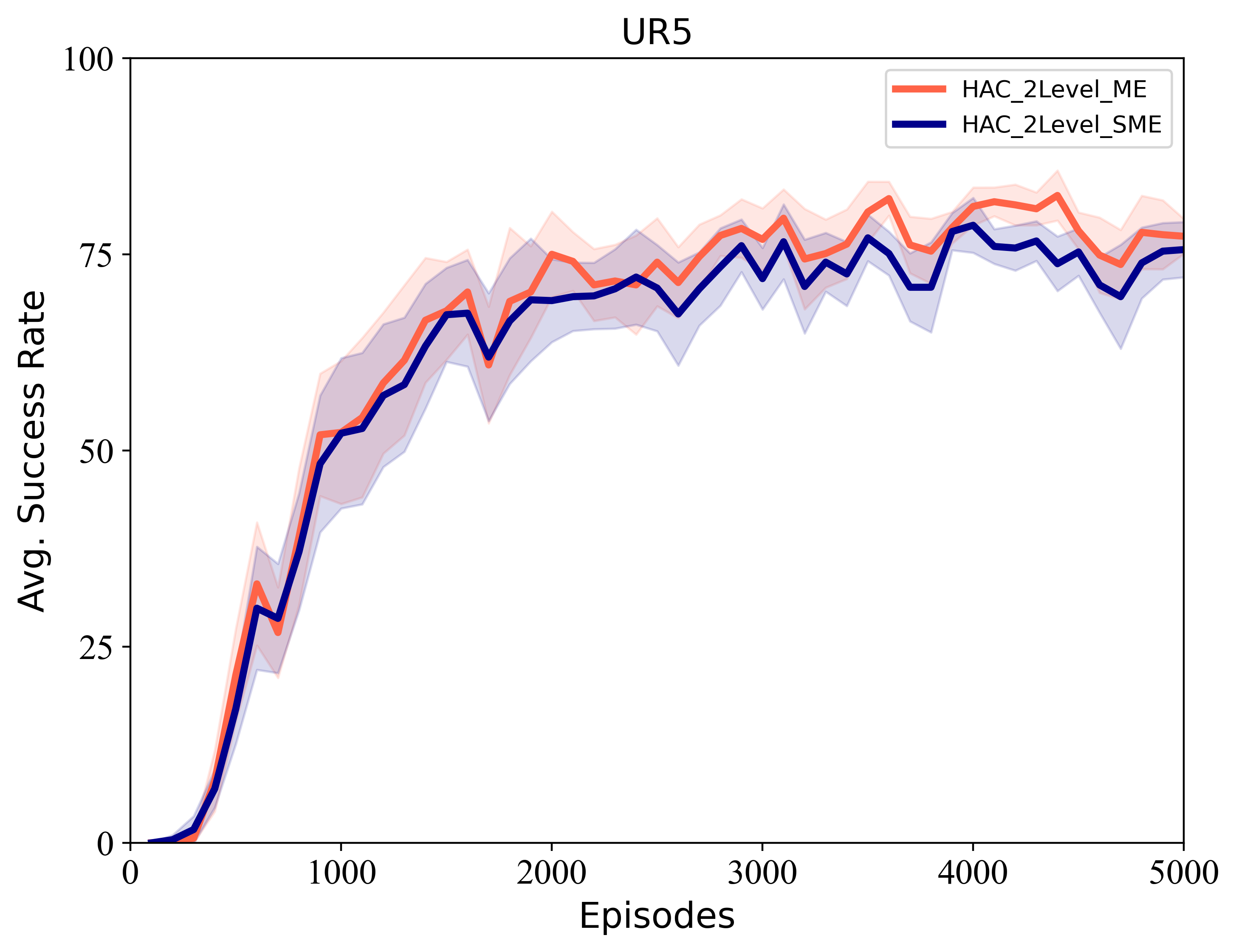}
\phantomsubcaption
\label{fig:UR5}
\end{subfigure}
\begin{subfigure}{1\linewidth}
\centering
\includegraphics[scale=0.4]{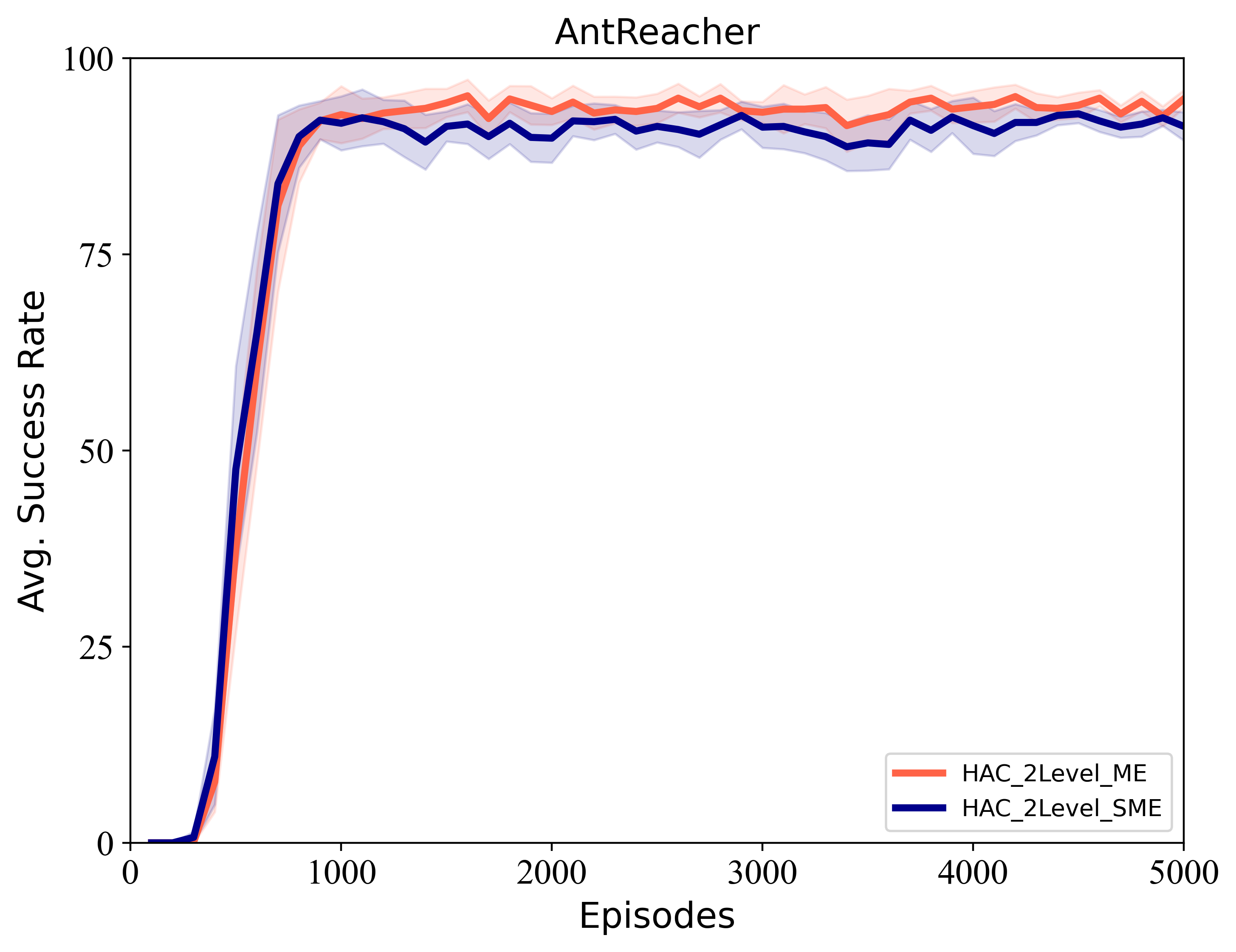}
\phantomsubcaption
\label{fig:Antreacher}
\end{subfigure}
\begin{subfigure}{1\linewidth}
\centering
\includegraphics[scale=0.4]{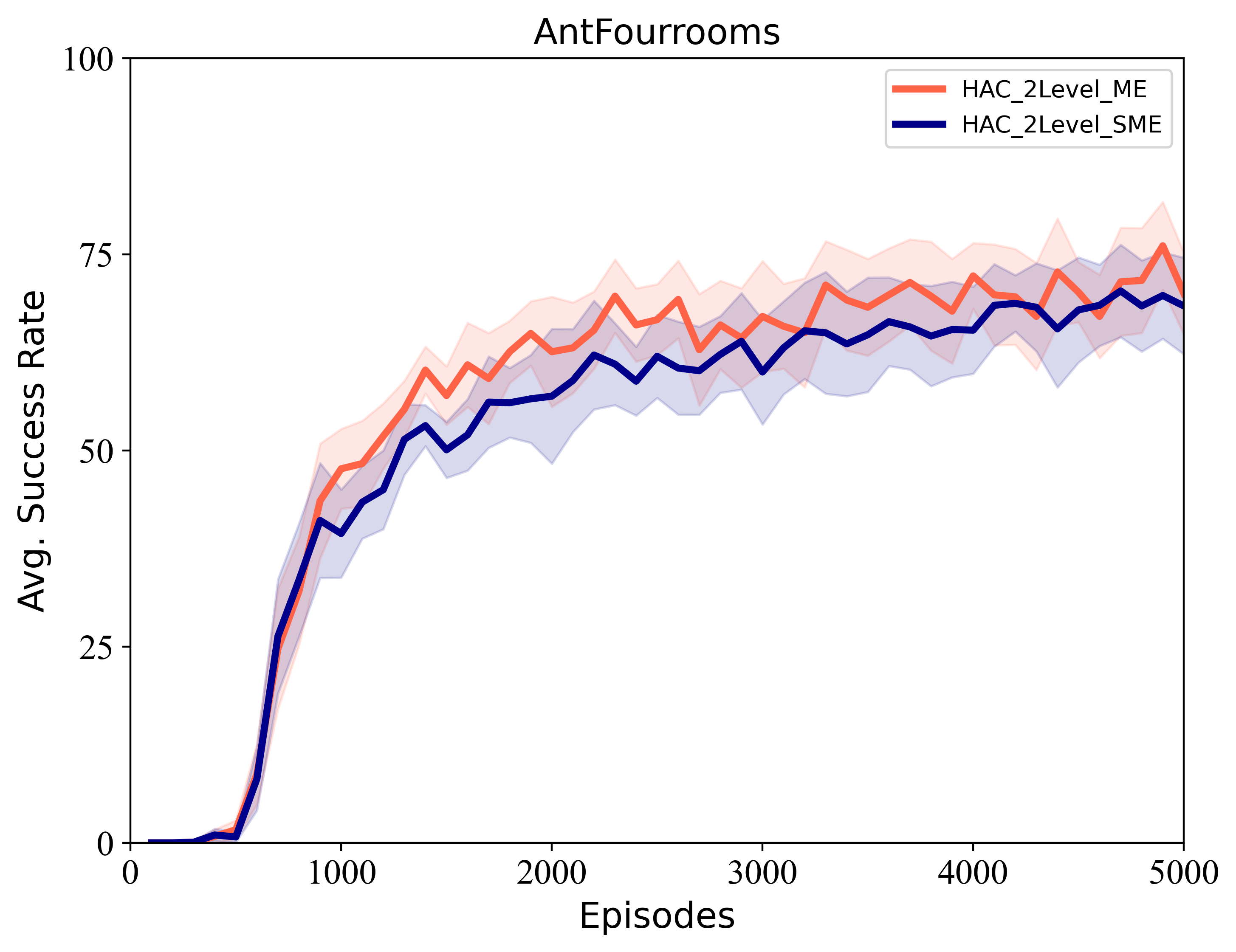}
\phantomsubcaption
\label{fig:Antfourrooms}
\end{subfigure}
\caption{Comparisons for HAC under SME and ME: (a) UR5 (top), (b) AntReacher (middle), and (c) AntFourrooms (bottom).}
\label{fig:HAC}
\end{figure}
Figures~\ref{fig:UR5}, \ref{fig:Antreacher}, and \ref{fig:Antfourrooms} show the same-policy execution comparison in UR5, AntReacher, and AntFourrooms, respectively.
To explore the performance difference of the same policy under SME and ME in the environment without random transition events, we deploy the HAC algorithm in the original environments of its paper: UR5, AntReacher, and AntFourrooms. And the results are shown in~\cref{fig:HAC}. We observe that the performance of ME is as least as good as SME's. The opposite conclusion may be obtained in other algorithms. Not all algorithms can improve performance after switching the execution mode from SME to ME at no cost due to their possible defects, especially when the higher-level policy is suboptimal and has not yet converged to the optimum. At this point, the conditions required by the Optimal Execution Theorem are not satisfied. Instead of directly adopting ME based on the theorem, one should design a new algorithm according to the Execution Improvement Theorem to learn the optimal interruption function for the deployment policy. Our framework provides a principled foundation for further potential improvements: the execution modes of the target policy, behavior policy, and deployment policy are mutually independent.
\par
\subsection{Hyperparameters}
For HAC, we follow the open-source implementation of the author, which is available at \url{https://github.com/andrew-j-levy/Hierarchical-Actor-Critc-HAC-}. All the experimental results are averaged using $10$ random seeds.
For OC, all the experimental results are averaged using $350$ random seeds. The key training hyperparameters are as follows:
\begin{itemize}[leftmargin=0.3cm]
    \item the discount factor $\gamma = 0.99$.
    \item the learning rate of termination function $\alpha_{term} = 0.25$.
    \item the learning rate of intra-option policy  $\alpha_{intra} = 0.25$.
    \item the learning rate of critic $\alpha_{critic} = 0.5$.
    \item the learning rate of higher-level policy $\alpha_{q} = 0.2$.
    \item the number of options $\omega = \{4,8\}$.
\end{itemize}
\subsection{Extra results and limitations}
We visualize the results of deploying the HAC algorithm in the AntReacher environment under SME and ME in~\cref{fig:visual} and found evident disadvantages of SME. When a random transition event occurs at $t=10$, the ant shows \emph{stagnant} until $t=23$ under SME. This result may be because the large-scale coordinate offset prevents the lower-level policy from making effective actions according to the current state and subgoal. However, under ME, the higher-level policy adjusts the subgoal in time so that the task of the lower-level policy is still moving towards the attainable subgoal.

\FloatBarrier
\end{document}